\pdfoutput=1
\documentclass[twoside,11pt]{article}

\usepackage[preprint]{jmlr2e}
\usepackage{amsmath}
\usepackage{amsfonts}
\usepackage{colonequals}
\usepackage{tikz}
\usepackage{float}
\usepackage{placeins}
\usepackage{tabularx}
\usepackage{booktabs}      
\usepackage[dvipsnames,table]{xcolor} 
\usepackage{algorithm}
\usepackage{algpseudocode}
\usepackage{array}
\hypersetup{
  hidelinks
}

\newcolumntype{I}{>{\slshape}X}
\newcolumntype{Y}{>{\raggedright\arraybackslash}X}

\newcommand{\bw}{\mathbf{w}}
\newcommand{\bff}{\mathbf{f}}
\newcommand{\F}{\mathbf{F}}

\usepackage{lastpage}
\jmlrheading{}{2026}{}{}{}{}{Eric Weine and Peter Carbonetto and Rafael Irizarry and Matthew Stephens}

\ShortHeadings{Poisson NMF with the Shifted Log Link}{Weine, Carbonetto, Irizarry, and Stephens}
\firstpageno{1}

\usetikzlibrary{calc}
\begin{document}

\title{A New Family of Poisson Non-negative Matrix Factorization Methods Using the Shifted Log Link}

\author{\name Eric Weine \email ericw456@mit.edu \\
       \addr Department of Electrical Engineering and Computer Science\\
       Massachusetts Institute of Technology\\
       Cambridge, MA, USA
       \AND
       \name Peter Carbonetto \email pcarbo@uchicago.edu \\
       \addr Department of Human Genetics\\
       University of Chicago\\
       Chicago, IL, USA
       \AND
       \name Rafael A. Irizarry \email rafael\_irizarry@dfci.harvard.edu \\
       \addr Department of Data Science\\
       Dana-Farber Cancer Institute\\
       Boston, MA, USA
      \AND
       \name Matthew Stephens \email mstephens@uchicago.edu \\
       \addr Department of Statistics and Department of Human Genetics\\
       University of Chicago\\
       Chicago, IL, USA
       }

\maketitle

\begin{abstract}
Poisson non-negative matrix factorization (NMF) is a widely used method to find interpretable ``parts-based'' decompositions of count data. While many variants of Poisson NMF exist, they all assume that the ``parts'' in the decomposition combine additively. This assumption may be natural in some settings, but not in others. Here, we introduce Poisson NMF with the shifted-log link function to relax this assumption. The shifted-log link function has a single tuning parameter, and as this parameter varies the model changes from assuming that parts combine additively to assuming that parts combine more multiplicatively. We provide an algorithm to fit this model by maximum likelihood, and also an approximation that substantially reduces computation time for large, sparse datasets. We illustrate these new methods on a variety of real and simulated data. 
Our examples show that while standard Poisson NMF often leads to sparse, cluster-like representations of data, the log1p link can produce more modular, layered results.
\end{abstract}


\section{Introduction}

Non-negative Matrix Factorization (NMF) \citep{lee1999learning} is a widely used method for dimensionality reduction of non-negative data matrices. Given a non-negative data matrix $\mathbf{Y}$, NMF methods attempt to find low-rank, non-negative matrices $\mathbf{L}$ and $\mathbf{F}$ such that 
\begin{equation}\label{eq:mf_goal1}
    \mathbf{Y} \approx \mathbf{L}\mathbf{F}^{\top} = \sum_{k=1}^K \mathbf{l}_{\cdot k} \mathbf{f}_{\cdot k}^\top.
\end{equation}
Here $\mathbf{l}_{\cdot k}$ (respectively $\mathbf{f}_{\cdot k}$) denotes the $k^{\textrm{th}}$ column of $\mathbf{L}$ (respectively $\mathbf{F}$).
Thus, NMF decomposes the data into a sum of $K$ components; if each of these $K$ components, {\it individually}, has a physical or scientific interpretation then the decomposition is said to provide a ``parts-based'' representation of the data \citep{lee1999learning}. The emphasis on {\it individually} is crucial here, because it distinguishes the idea of a parts-based representation from other low-dimensional representations, or embeddings. While other matrix factorization methods, like principal components analysis (PCA), also provide a decomposition of the form \eqref{eq:mf_goal1}, \cite{lee1999learning} argue that the decompositions provided by NMF tend to produce more individually-interpretable components. That is, NMF more often provides a parts-based representation. This feature has led to widespread adoption of NMF in practical applications \citep{pritchard2000inference,lda,luce2016using,dey2017visualizing, mackevicius2019unsupervised,cancersig}.  

Count data, which consist of integer counts $(y_{ij})$, 
represent a common type of non-negative data frequently analyzed by NMF. Examples include word counts in documents, transcript counts in RNA-seq, and mutation counts in tumors. Most NMF methods for count data assume a Poisson model, where the elements $y_{ij}$ are independent and Poisson distributed, and where the expected values of the count variables are modeled directly as $\mathbb{E}[y_{ij}] = (\mathbf{L}\mathbf{F}^{\top})_{ij}$ \citep[e.g.][]{lee1999learning,gopalan2015scalable,zito2024compressive,landy2025bayesnmf}. This corresponds to assuming that the parts in the decomposition \eqref{eq:mf_goal1} contribute additively to the expectation. We refer to such approaches as \emph{standard Poisson NMF}.

While the additive assumption of standard Poisson NMF may be natural in some settings, it may be less appropriate in others where components may combine multiplicatively (e.g., gene expression, see \cite{sanford2020gene, zhou2024analysis}). Here we introduce a more flexible approach to Poisson NMF that incorporates a shifted-log link function to relate the expected counts to the elements of $(\mathbf{L}\mathbf{F}^{\top})_{ij}$.
Depending on the value of a single hyper-parameter, the shifted log link can capture a range of behaviors from additive to more multiplicative.

Our work makes three main contributions. First, we provide an algorithm to fit Poisson NMF with the shifted-log link function by maximum likelihood. Second, because the MLE for this model is computationally impractical for very large datasets, we develop an efficient approximation to the log-likelihood whose computational complexity scales only with the number of non-zero entries in the data, allowing efficient (approximate) maximum likelihood estimation for sparse datasets common in text and biological applications. Third, we demonstrate our methods on real and simulated data, showing that the choice of link function can substantially impact results, and that the shifted-log link can produce more interpretable parts-based representations than standard Poisson NMF in some settings.


Our new approach to Poisson NMF is a special case of a \emph{generalized bi-linear model} (GBM) \citep{choulakian1996generalized} (see also \cite{collins2001generalization}) which, analogous to a generalized linear model \citep{glms-book}, uses a link function to relate the expected value of a distribution in the exponential family to a bi-linear term (i.e., $\mathbf{L}\mathbf{F}^{\top}$). Previous work has provided methods to fit Poisson GBMs with a variety of link functions, including the canonical $\log$ link (\cite{townes2019feature,weine2024fast}) and the $\log(1+\exp(y))$ link
\citep{seeger2012fast}. However these link functions are not bijections on the non-negative real line, and so they do not naturally lead to NMF methods; rather, they are more like versions of PCA for Poisson data.  As far as we are aware, our paper is the first to provide a version of Poisson NMF with non-identity link function, and, furthermore, one that  
is practical for large, sparse count data.




\section{Poisson NMF with the shifted log
  link function, and connections to existing models}
  
Given a data matrix of counts $\mathbf{Y} \in \mathbb{N}_{0}^{n \times p}$ we assume the following Poisson NMF model: 
\begin{align} \label{eq:log1p_nmf_y}
y_{ij} &\overset{\textrm{indep.}}{\sim} \textrm{Poisson}(\lambda_{ij}) \\ \label{eq:log1p_nmf_link}
g(\lambda_{ij}; c)  &= b_{ij}  \\ 
\mathbf{B} &= \mathbf{L}\mathbf{F}^{\top}, \label{eq:log1p_nmf_B}
\end{align}
where $g$ denotes the following shifted-logarithm link function:
\begin{equation} \label{eq:shifted_log_link}
    g(\lambda; c) = \alpha_c\log\left(1 + \lambda/c\right).
\end{equation}
Here $\mathbf{L} \in \mathbb{R}_{\geq 0}^{n \times K}$, $\mathbf{F} \in \mathbb{R}_{\geq 0}^{p \times K}$, $c \in \mathbb{R}_{>0}$, and $\alpha_c := \max(1, c)$ is a scaling constant that we introduce for convenience to make the scale of $\mathbf{B}$ more comparable across values of $c$. 

Importantly, the shifted-log link, $g$, is a bijection on the non-negative real line, making it suited to NMF (unlike, say, the $\log$ link used in \cite{townes2019feature} and \cite{weine2024fast}).
Its inverse is $g^{-1}(b; c) = c \cdot (\exp(b / \alpha_c) - 1)$.
For brevity we refer to this link function as the ``log1p'' (``log 1 plus'') link, noting that it is really a family of link functions indexed by the choice of $c$. We similarly refer to the model \eqref{eq:log1p_nmf_y}-\eqref{eq:log1p_nmf_B} as the ``log1p NMF model'', or simply ``log1p NMF'', and denote its log-likelihood as $\ell_{\textrm{log1p}}(\mathbf{L}, \mathbf{F}, c; \mathbf{Y}).$

 
The behavior of the link function $g$ depends on the value of $c$ (Figure \ref{fig:fig_vis}). If $c$ is large, then $g(\lambda; c) \approx \lambda$, in which case log1p NMF becomes standard Poisson NMF where parts combine additively. If $c$ is small (i.e., near $0$), for a fixed $\lambda$ the quantity $\lambda / c$ gets large and thus $\log\left(1 + \frac{\lambda}c\right) \approx \log\left(\frac{\lambda}c\right) = \log(\lambda) - \log(c)$.  Thus for small $c$ the link function acts more like the $\log$ link and the parts combine more multiplicatively. That is, log1p NMF provides a bridge between standard Poisson NMF, in which parts combine additively, and a new set of NMF models in which parts combine more multiplicatively.  The next subsections provide more formal statements of these ideas.

\begin{figure}
\centering
\includegraphics[width=0.75\linewidth]{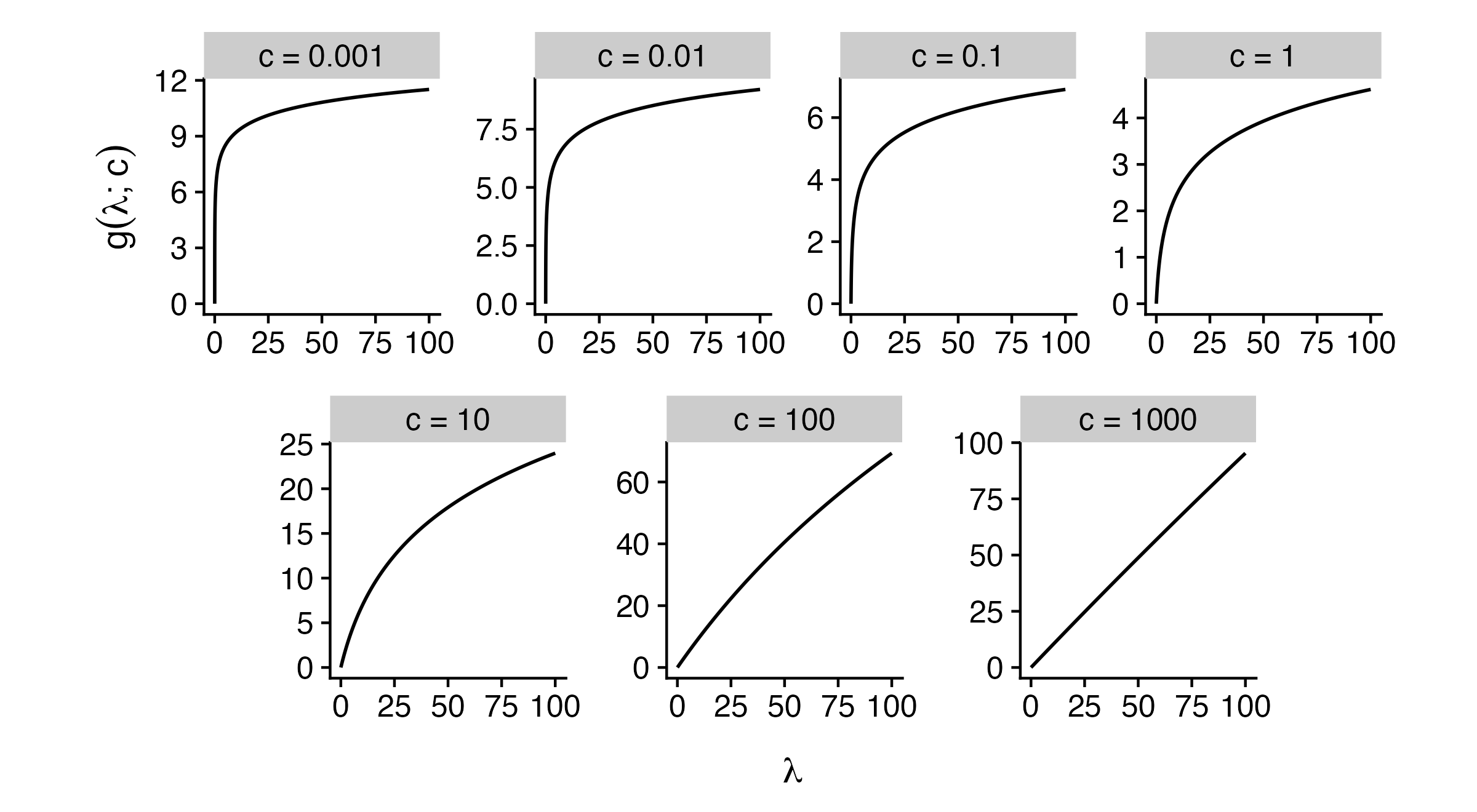}
\caption{Plots of the link function
$g(\lambda; c) = \alpha_c \times \log(1 + \lambda / c)$ for various values of $c$.}
\label{fig:fig_vis}
\end{figure}

\subsection{Connections with existing models and methods}

\subsubsection{Standard Poisson NMF}

The standard Poisson NMF model \citep{lee1999learning} is given by \eqref{eq:log1p_nmf_y}-\eqref{eq:log1p_nmf_B} but replacing $g(\lambda_{ij}; c)$ with $\lambda_{ij}$. That is, standard Poisson NMF
is equivalent to using the identity link function instead of a log1p link function. To reflect this we use $\ell_{\textrm{id}}( \mathbf{L}, \mathbf{F}; \mathbf{Y})$ to denote the log-likelihood for standard Poisson NMF.
As $c \rightarrow \infty$, the log1p link converges to the identity link ($g(\lambda; c) \rightarrow \lambda$) and so the log1p NMF model converges to the standard Poisson NMF model, as formalized in the following Theorem.
\begin{theorem}
\label{theorem:log1p-c-inf}
For any fixed $\mathbf{Y} \in \mathbb{N}_{0}^{n \times p}$, for all
$\mathbf{L} \in \mathbb{R}_{\geq 0}^{n \times K}$ and $\mathbf{F} \in
\mathbb{R}_{\geq 0}^{p \times K}$,
\begin{equation*}
\lim_{c \rightarrow \infty} \ell_{\textrm{log1p}}\left(
\mathbf{L}, \mathbf{F}, c; \mathbf{Y}\right) =
\ell_{\textrm{id}}( \mathbf{L}, \mathbf{F}; \mathbf{Y}).
\end{equation*}
\begin{proof}
See Appendix \ref{app:thm1_pf}.
\end{proof}
\end{theorem}

Note that the standard Poisson NMF model is also essentially equivalent to the multinomial factor model often referred to as the ``topic model'' \citep{carbonetto2021non}. Thus, the log1p NMF model with large $c$ is also essentially equivalent to the standard topic model.


\subsubsection{Poisson GLM-PCA}

Another factor model for count data is the \emph{Poisson GLM-PCA} model
\citep{townes2019feature, nicol2024model,weine2024fast}. This model uses the log link instead of the log1p link, and dispenses with the non-negative assumption: 
\begin{align}
y_{ij} &\overset{\textrm{indep.}}{\sim} \textrm{Poisson}(\lambda_{ij}) \notag \\
\log\left(\lambda_{ij}\right) &= b_{ij} \label{eq:glmpca} \\
\mathbf{B} &= \mathbf{L}\mathbf{F}^{\top}, \notag
\end{align}
where $\mathbf{L} \in \mathbb{R}^{n \times K}$ and $\mathbf{F} \in \mathbb{R}^{p \times K}$. (Typically, additional orthogonality constraints are placed on $\mathbf{L}$ and $\mathbf{F}$ for identifiability.)
Note that under this model we have 
\begin{equation*}
\lambda_{ij} = \prod_{k=1}^K \exp\left(l_{ik} f_{jk}\right),
\end{equation*} 
so Poisson GLM-PCA assumes that the latent factors combine \textit{multiplicatively} to influence $\lambda_{ij}$. The $k^{\textrm{th}}$ factor can then be interpreted as capturing a log-fold change in the Poisson mean due to inclusion of that factor. Indeed, if $\lambda_{ij}$ and $\lambda_{ij}^{(-k)}$ denote, respectively, the values of $\lambda$ including and excluding the $k^{\textrm{th}}$ factor, then in GLM-PCA
\begin{equation} \label{eq:effect_factor_K}
    \log\frac{\lambda_{ij}}{\lambda^{(-k)}_{ij}} = l_{ik} f_{jk}.
\end{equation}

A fundamental difference between GLM-PCA and log1p NMF is that GLM-PCA does not impose non-negative constraints
on $\mathbf{L}$ and $\mathbf{F}$. Thus, as its name suggests, GLM-PCA is much more like PCA than like NMF. Indeed,  non-negative constraints would usually not make sense in \eqref{eq:glmpca}, as they would imply $\lambda_{ij}\geq 1$.  However, the two models do have a connection: for values of $c \ll \lambda_{ij}^{(-k)}$, the interpretation of the $k^{\textrm{th}}$ factor is approximately \eqref{eq:effect_factor_K} (up to a constant that is equal to $1$ for $c \leq 1$).
To see this, we first establish the following theorem:
\begin{theorem}
\label{theorem:log1p-c-zero}
Let $\lambda_{ij} > 0$ and $\lambda_{ij}^{(-k)} > 0$ denote, respectively, the values of $\lambda$ including and excluding the $k^{\textrm{th}}$ factor in the log1p NMF model. Then,
\begin{equation} \label{eqn:lf_interpretation}
    \log\frac{\lambda_{ij}}{\lambda^{(-k)}_{ij}} = (1 / \alpha_c) l_{ik}  f_{jk} + \varepsilon_{ij},
\end{equation}
where 
\begin{equation}
    \label{eq:thm2_stmt2}
    \varepsilon_{ij} = \log\left(1 + c/\lambda_{ij}^{(-k)}\right) - \log(1 + c/\lambda_{ij}).
\end{equation}
\begin{proof}
See Appendix \ref{app:thm2_pf}.
\end{proof}
\end{theorem}

Thus, in cases where $c \ll \lambda_{ij}^{(-k)}$ (and therefore $c \ll \lambda_{ij}$ due to non-negativity constraints), both logarithms in \eqref{eq:thm2_stmt2} are close to zero, and so $\varepsilon_{ij} \approx 0$; indeed, $0 \leq \varepsilon_{ij} \leq c / \lambda_{ij}^{(-k)}$. In this regime, $\log(\lambda_{ij} / \lambda_{ij}^{(-k)}) \approx (1/\alpha_c) l_{ik} f_{jk}$, so the $k^{\textrm{th}}$ factor can be interpreted, just as in GLM-PCA \eqref{eq:effect_factor_K}, as capturing a log-fold change in the Poisson mean, scaled by the constant $1/\alpha_c$ (which is equal to $1$ when $c \leq 1$). This is what we mean when we say that factors combine ``more multiplicatively'' in log1p NMF with small $c$; throughout, we refer to this behavior, and to models exhibiting it, as ``quasi-multiplicative''.

\subsubsection{Frobenius-norm NMF applied to shifted log counts}

Another approach to dealing with count data, which is particularly common in the
analysis of single cell RNA-seq data, is to transform 
the counts and then apply methods designed for Gaussian data to the transformed
counts \citep{ahlmann2023comparison}. The shifted log \textit{transformation} is commonly used in this context, especially when performing NMF \citep{willwerscheidthesis,johnson2023inferring}. Specifically, for some fixed $c > 0$, these methods find the solution to 
\begin{gather}
    \min_{\mathbf{L} \in \mathbb{R}_{\geq 0}^{n \times K}, \mathbf{F} \in
\mathbb{R}_{\geq 0}^{p \times K}} ||\Tilde{\mathbf{Y}} - \mathbf{L}\mathbf{F}^\top ||_F \label{eq:frob_nmf} \\
    \intertext{where}
     \Tilde{y}_{ij} = \log(1 + y_{ij} / c). \label{eq:log1p_trans} 
\end{gather}
This is equivalent to maximum likelihood estimation of a Gaussian NMF model on the transformed counts, assuming the same residual variance for each element of the matrix.

While there are both theoretical and practical concerns with applying Gaussian methods to count data, especially in the context of single-cell RNA sequencing \citep{nicol2024model, townes2019feature}, in practice it can sometimes lead to reasonable results with good performance in downstream tasks \citep{ahlmann2023comparison}. In Appendix \ref{app:approx_comp}, we investigate empirically how closely results from fitting \eqref{eq:frob_nmf}-\eqref{eq:log1p_trans} match those from log1p link Poisson NMF.  One theoretical disadvantage of the transformation approach is that the parameter $c$ controls both (a) the relationship between the latent factors and the underlying mean and (b) the variance stabilization properties of the transformation \citep{ahlmann2023comparison}. It is possible that, in the transformation approach, the value of $c$ that best stabilizes variance does not correspond with the desired value of $c$ for downstream interpretation of the factors, and thus there may be a trade-off between these two goals. Indeed, for very large settings of $c$ in equation \eqref{eq:log1p_trans}, $\Tilde{\mathbf{Y}} \approx \frac{1}{c}\mathbf{Y}$, which when solving equation \eqref{eq:frob_nmf} is equivalent to assuming that each element of $\mathbf{Y}$ is Gaussian with approximately the \textit{same} variance. In our log1p model $c$ simply controls the relationship between the factors and the underlying mean structure; the variance stabilization issue is avoided by direct use of the Poisson likelihood for the count data. 


\section{Algorithms for fitting log1p NMF}

\subsection{Fitting model parameters with block coordinate ascent}

We take a maximum likelihood approach to fitting the log1p NMF model described in equations \eqref{eq:log1p_nmf_y}-\eqref{eq:log1p_nmf_B}. The log-likelihood of the log1p NMF model is
\begin{equation}
    \ell_{\textrm{log1p}}(\mathbf{L}, \mathbf{F}, c; \mathbf{Y}) = \sum_{i = 1}^{n}\sum_{j = 1}^{p} \left( y_{ij} \log\left[\exp\left\{ \frac{1}{\alpha_c} \sum_{k = 1}^{K} l_{ik} f_{jk} \right\} - 1 \right] - c \cdot \exp\left\{ \frac{1}{\alpha_c} \sum_{k = 1}^{K} l_{ik} f_{jk} \right\} \right), \label{eq:log1p_ll}
\end{equation}
where we have omitted terms that do not involve $\mathbf{L}$ or $\mathbf{F}$. 

Maximizing $\ell_{\textrm{log1p}}(\mathbf{L}, \mathbf{F}, c; \mathbf{Y})$ with respect to $\mathbf{L}$ and $\mathbf{F}$ is a high-dimensional, non-convex optimization problem. However, as shown in Theorem \ref{theorem:biconcave} (see Appendix \ref{app:thm3}), $\ell_{\textrm{log1p}}(\mathbf{L}, \mathbf{F}, c; \mathbf{Y})$ with $c$ fixed is a bi-concave function of $\mathbf{L}$ and $\mathbf{F}$ (so $-\ell_{\textrm{log1p}}(\mathbf{L}, \mathbf{F}, c; \mathbf{Y})$ is bi-convex). This motivates an alternating optimization approach, which alternates repeatedly between optimizing for $\mathbf{L}$ with $\mathbf{F}$ fixed, and optimizing for $\mathbf{F}$ with $\mathbf{L}$ fixed. Conveniently, these sub-problems break down into a series of simpler tasks that can be performed in an embarrassingly parallel way, as we now describe. 

To begin this description, consider the following Poisson regression model with log1p link:
\begin{align*}
y_{i} &\overset{\textrm{indep.}}{\sim} \textrm{Poisson}(\lambda_{i}) \\
g(\lambda_i; c) &= \mathbf{x}_{i}^{\top} \boldsymbol{\beta},
\end{align*}
where $\mathbf{y} \in \mathbb{N}^{N}_{0}$ is a vector of counts, $\mathbf{X} \in \mathbb{R}^{N \times q}_{\geq 0}$ is a fixed matrix of non-negative ``covariates'' with $i$th row $\mathbf{x}_i^\top$, and $\boldsymbol{\beta} \in  \mathbb{R}^{q}_{\geq 0}$ is an unknown vector of non-negative regression coefficients. This regression model has log-likelihood $\ell_{\textrm{log1pReg}}$, given by
\begin{equation}\label{eq:log1p_nn_reg_ll}
    \ell_{\textrm{log1pReg}}(\boldsymbol{\beta}, c; \boldsymbol{y}, \mathbf{X}) = \sum_{i = 1}^{N} \left( y_{i} \log\left\{\exp\left\{\mathbf{x}_{i}^{\top} \boldsymbol{\beta} / \alpha_c \right\} - 1 \right\} - c \cdot \exp\left\{ \mathbf{x}_{i}^{\top} \boldsymbol{\beta} /\alpha_c\right\} \right),
\end{equation}
which can also be shown to be concave in $\boldsymbol{\beta}$ by a similar argument to Theorem \ref{theorem:biconcave}.

The Poisson NMF model \eqref{eq:log1p_nmf_y}
can be considered as a series of regressions in two different ways: each column of $\mathbf{Y}$ is a regression on the columns of  $\mathbf{L}$ (with elements of $\mathbf{F}$ as the regression coefficients), or each row of $\mathbf{Y}$ is a regression on the columns of  $\mathbf{F}$ (with elements of $\mathbf{L}$ as the regression coefficients). More algebraically,
the log1p NMF log-likelihood \eqref{eq:log1p_ll} can be written as a sum of regression log-likelihoods in either of two ways: 
\begin{equation}\label{eq:fit_L}
    \ell_{\textrm{log1p}}(\mathbf{L},\mathbf{F}, c; \mathbf{Y}) = \sum\limits_{j = 1}^{p} \ell_{\textrm{log1pReg}}(\mathbf{f}_{j \cdot}, c; \mathbf{y}_{\cdot j}, \mathbf{L}) = \sum\limits_{i = 1}^{n} \ell_{\textrm{log1pReg}}(\mathbf{l}_{i \cdot}, c; \mathbf{y}_{i \cdot}, \mathbf{F}),
\end{equation}
where $\mathbf{y}_{\cdot j}$ (respectively  $\mathbf{y}_{i \cdot}$)  denotes the column vector containing the $j^{\textrm{th}}$ column (respectively $i^{\textrm{th}}$ row) of the matrix $\mathbf{Y}$.
Thus, with $\mathbf{L}$ fixed, optimizing over $\mathbf{F}$ involves solving $p$ independent
non-negative Poisson regression problems. Similarly, with $\mathbf{F}$ fixed, optimizing over $\mathbf{L}$ involves solving $n$ independent non-negative Poisson regression problems. Because these steps involve solving independent regression problems, they can be done in parallel (see Algorithm \ref{alg:cap}). Fitting each regression problem requires a numerical solver, for which we use coordinate ascent as described in Appendix \ref{app:ccd}. 
The resulting algorithm essentially extends the “Alternating Poisson Regression” approach \citep{carbonetto2021non,weine2024fast} used for other Poisson factor models to accommodate the shifted-log link function.

\begin{algorithm}[!t]
\caption{Alternating Optimization Method for Fitting Poisson log1p NMF Model. Row $i$
  and column $j$ of $\mathbf{Y}$ are denoted, respectively, by $\mathbf{y}_{i \cdot}$ and $\mathbf{y}_{\cdot j}$.}
\label{alg:cap}
\begin{algorithmic}[1]

\Require Count data ${\boldsymbol Y} \in \mathbb{N}_{0}^{n \times p}$, initial
estimates $\mathbf{L} \in \mathbb{R}_{\geq 0}^{n \times K}$, $\mathbf{F} \in \mathbb{R}_{\geq 0}^{p \times K}$, constant $c > 0$, and a function $\mbox{\sc Pois-reg-log1p}(\mathbf{X}, \mathbf{y}, c)$ that returns the constrained MLE of $\boldsymbol{\beta}$ in a non-negative Poisson regression with the log1p link.

\While{not converged}

\For{$i = 1, \dots, n$}
\Comment{These can be performed in parallel.}

\State $\mathbf{l}_{i \cdot} \gets \mbox{\sc Pois-reg-log1p}(\mathbf{F}, \mathbf{y}_{i \cdot}, c)$

\State Store $\mathbf{l}_{i \cdot}$ in the $i$th row of $\mathbf{L}$.

\EndFor

\For{$j = 1, \dots, p$}
\Comment{These can be performed in parallel.}

\State $\mathbf{f}_{j \cdot} \gets \mbox{\sc Pois-reg-log1p}(\mathbf{L}, \mathbf{y}_{\cdot j}, c)$

\State Store $\mathbf{f}_{j \cdot}$ in the $j$th row of $\mathbf{F}$.

\EndFor

\EndWhile

\State \Return $\mathbf{L}, \mathbf{F}$
\end{algorithmic}
\end{algorithm}

\subsection{Computational complexity and an approximation for sparse data}

The computational complexity of each outer-loop iteration of Algorithm \ref{alg:cap} is $\mathcal{O}(npK)$, which scales linearly with the size, $np$, of the data matrix $\mathbf{Y}$. For standard Poisson NMF, with sparse data matrices $\mathbf{Y}$, computation can be reduced to $\mathcal{O}(\omega K +(n +p)K)$ where $\omega$ denotes the number of non-zero entries in $\mathbf{Y}$; see \cite{carbonetto2021non}. Large sparse data matrices are ubiquitous in count data (e.g., single cell RNA-seq, document-term matrices), and $np$ may be orders of magnitude larger than $\omega$. In such cases, despite its parallel nature, Algorithm \ref{alg:cap} may become computationally impractical, even though similar algorithms for standard Poisson NMF are practical.
To address this, we now introduce an approximate maximum likelihood algorithm for log1p NMF that scales with $\omega$ instead of $np$, while remaining accurate.

To highlight the role of data sparsity in the log1p NMF log-likelihood, we re-write the log-likelihood \eqref{eq:log1p_ll} as
\begin{equation*}
    \ell_{\textrm{log1p}}(\mathbf{L}, \mathbf{F}, c; \mathbf{Y}) = \sum_{(i,j) \notin \mathcal{I}_{0}} y_{ij} \log\left(\exp\left\{ \frac{1}{\alpha_c}\sum_{k = 1}^{K} l_{ik} f_{jk} \right\} - 1 \right) - c\sum_{i=1}^{n}\sum_{j=1}^{p} \exp\left(\frac{1}{\alpha_c}\sum_{k = 1}^{K} l_{ik}f_{jk}\right),
\end{equation*}
where $\mathcal{I}_{0} = \{(i, j): y_{ij} = 0\}$ is the index set of $0$ counts in the matrix $\mathbf{Y}$. Computing the first term scales with $\omega = np - \lvert \mathcal{I}_0 \rvert$, while computing the second term is the computational bottleneck: it sums $n p$ exponential terms, each of which requires $K$ operations to compute.  

The reason that standard Poisson NMF can be made efficient for sparse data matrices is that this problematic ``sum of exponentials'' term does not occur; instead there is a sum of linear terms, which can be computed efficiently in $\mathcal{O}((n+p) K)$. A simple way to make the log1p NMF computationally tractable would be to approximate the exponential terms with linear terms. Unfortunately, in general, this would yield a very bad approximation: $\exp(x)$ is accurately approximated by a linear function only for $x$ very close to 0. Therefore we improve this naive idea in two important ways. First, following previous work in approximate GLM / GBM inference \citep{huggins2017pass, zoltowski2018scaling, keeley2020efficient}, we use a quadratic approximation $\exp(x) \approx \eta_{0} + \eta_{1} x + \eta_{2} x^{2}$, where $\eta_{0}, \eta_{1},$ and $\eta_{2}$ are chosen either by Taylor approximation of $\exp(x)$ about some $x_{0}$, or by Chebyshev approximation over some interval $[x_{L}, x_{U}]$. Second, we approximate {\it only the terms in the sum corresponding to $y_{ij}=0$} (i.e.~$(i,j) \in \mathcal{I}_{0}$). The intuition is that for such terms the corresponding estimates of $\lambda_{ij}$ will typically be small (for $\mathbf{L,F}$ consistent with the data), and so the quadratic approximation will be accurate in the parts of the space that matter.

These two ideas, when combined, give the following approximate log-likelihood (omitting an additive constant):
\begin{align}\label{eq:approx_log1p_ll}
    \ell_{\textrm{log1p}}(\mathbf{L}, \mathbf{F}, c; \mathbf{Y}) \approx & \sum_{(i,j) \notin \mathcal{I}_{0}} y_{ij} \log\left(\exp\left\{ \frac{1}{\alpha_c} \sum_{k = 1}^{K} l_{ik} f_{jk} \right\} - 1 \right) - c\sum_{(i,j) \notin \mathcal{I}_{0}} \exp\left(\frac{1}{\alpha_c}\sum_{k = 1}^{K} l_{ik} f_{jk}\right) \notag \\
& - \frac{\eta_{1}c}{\alpha_c} \sum_{(i,j) \in \mathcal{I}_{0}} \sum_{k = 1}^{K} l_{ik} f_{jk} - \frac{\eta_{2}c}{\alpha_c^{2}}\sum_{(i,j) \in \mathcal{I}_{0}} \left(\sum_{k = 1}^{K} l_{ik} f_{jk}\right)^{2}.
\end{align}
Computing the first two terms of equation \eqref{eq:approx_log1p_ll} requires $\mathcal{O}(\omega K)$ operations. While naively it appears that the subsequent linear and quadratic terms require $\mathcal{O}(\lvert \mathcal{I}_{0} \rvert K)$ operations, both terms can actually be computed much more efficiently by simple algebraic rearrangements (see Appendix \ref{app:comp_complexity} for more details). The total computational complexity of this approximate log-likelihood becomes
\begin{equation*}
\mathcal{O}\left((\omega + n + p)K + (n + p)K^{2}\right).
\end{equation*}

 Table \ref{tab:comp_complexity} summarizes and compares the computational complexity of computing the log-likelihood for different Poisson matrix factorization models on sparse data. As shown in Figure \ref{fig:comp_scaling}, for large, sparse data, especially with relatively small settings of $K$, our approximation can be orders of magnitude faster to compute than is the exact log-likelihood (as well as GLM-PCA, which has the same computational complexity as \eqref{eq:log1p_ll}; \cite{weine2024fast}). Moreover, for large, sparse data, our approximate log-likelihood is nearly as fast to compute as the log-likelihood of standard Poisson NMF. 


\begin{table}[t]
\centering
\renewcommand{\arraystretch}{1.4}
\setlength{\tabcolsep}{10pt}

\begin{tabularx}{\textwidth}{|l|X|}
\hline
\textbf{Model} & \textbf{Log-likelihood Complexity} \\
\hline
Standard Poisson NMF
& $\mathcal{O}((\omega + n + p)K)$ \citep{carbonetto2021non} \\
\hline
Poisson GLM-PCA
& $\mathcal{O}(npK)$ \citep{weine2024fast} \\
\hline
Log1p Poisson NMF Exact
& $\mathcal{O}(npK)$ \\
\hline
Log1p Poisson NMF Approximate
& $\mathcal{O}\!\left((\omega + n + p)K + (n + p)K^{2}\right)$ \\
\hline
\end{tabularx}

\caption{Computational complexity of log-likelihood for Poisson matrix factorization variants on sparse data matrices (size $n \times p$ with $\omega$ non-zero entries).}
\label{tab:comp_complexity}
\end{table}

\begin{figure}
    \centering
    \includegraphics[width=1.0\linewidth]{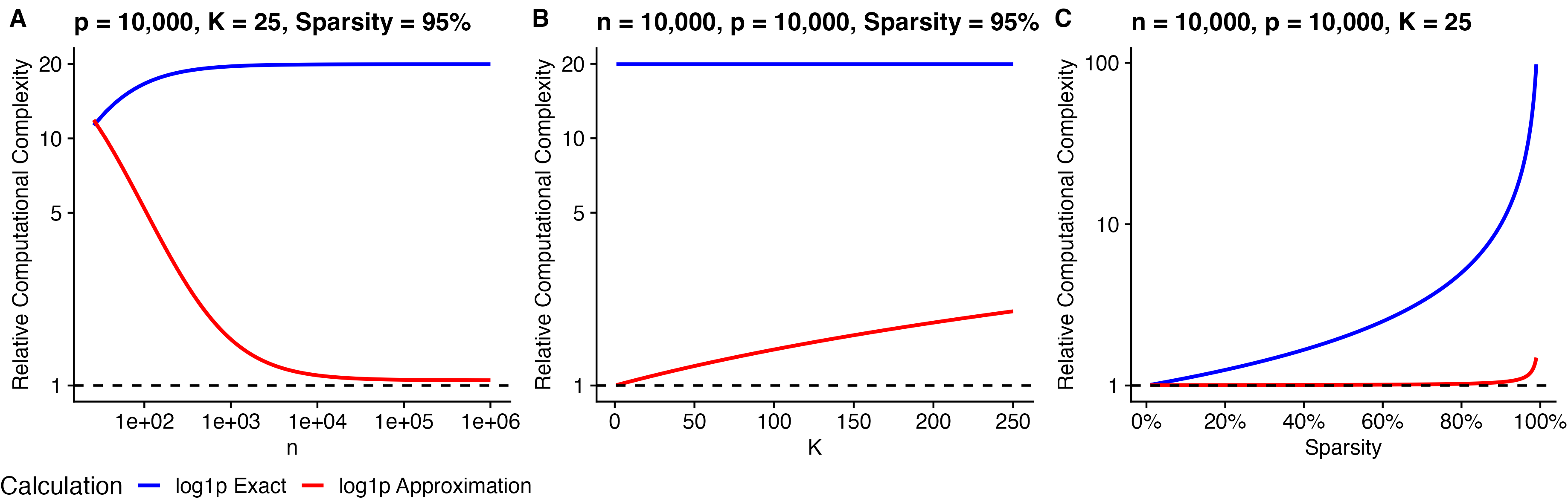}
    \caption{Scaling of computational complexity of equations \eqref{eq:log1p_ll} and \eqref{eq:approx_log1p_ll} with data characteristics and choice of $K$. The y-axes are the ratio of the computational complexity of the specified calculation relative to the complexity of computing the log-likelihood of the standard Poisson NMF model, where $1$ indicates the two calculations have the same complexity. (A) Scaling with respect to $n$ with all other variables fixed (note that this is equivalent to scaling with $p$). (B) Scaling with $K$ with all other variables fixed. (C) Scaling with sparsity $\left(\left[1 - \frac{\omega}{np}\right] \cdot 100\%\right)$ with all other variables fixed.}
    \label{fig:comp_scaling}
\end{figure}

Using the approximate log-likelihood of equation \eqref{eq:approx_log1p_ll}, we can use the same computational approach as Algorithm \ref{alg:cap}. That is, we repeatedly optimize equation \eqref{eq:approx_log1p_ll} for $\mathbf{L}$ with $\mathbf{F}$ fixed, and for $\mathbf{F}$ with $\mathbf{L}$ fixed. Just as with the exact log-likelihood, these subproblems are concave and embarrassingly parallel.

\subsection{Accuracy of the sparse computational approximation}
\label{sec:accuracy_approx}
To assess the accuracy of the approximate log-likelihood, we simulated data and examined likelihood ratios between the fitted models (using the exact and approximate optimization schemes). Specifically, we generated data with $n = p = 500$ and $K = 5$ from a log1p NMF model with values of $c = 10^{-3}$, $c = 1$, or $c = \infty$ (i.e., standard Poisson NMF), following the procedure of Simulation 2 (Section \ref{sec:sim2}). We generated $30$ replicate datasets for $c \in \{1, \infty\}$ and $100$ replicate datasets for $c = 0.001$ due to large variation across replicates in that setting. Then, for each dataset, we fit log1p NMF to convergence using both the approximate and exact objective functions for a grid of settings of $c$, ranging between $10^{-4}$ and $10^4$. We set $\eta_{0}, \eta_{1}$, and $\eta_{2}$ using either i) a second order Taylor approximation of $\exp(x)$ about $x = 0$ or ii) a Chebyshev approximation over the interval $[0, \log(1 + 1/c)]$ for the setting of $c$ used to fit the model. 

The results are shown in Figure \ref{fig:approx_quality}. Overall, the Chebyshev approximation method performs very accurately, especially when the model is fit with $c \geq 0.01$. The Taylor approximation method performs poorly for small values of $c$, likely because the range of the optimal value for $b_{ij}$ can still be quite large when $y_{ij} = 0$. More concretely, if for some $y_{ij} = 0$ the constrained MLE $\hat{\lambda}_{ij} = \varepsilon$, then the corresponding MLE of $b_{ij}$ is
\begin{equation*}
    \hat{b}_{ij} = \alpha_c\log\left(1 + \frac{\varepsilon}{c}\right).
\end{equation*}
Even when $\varepsilon$ is small, if $c$ is small relative to $\varepsilon$, then $\hat{b}_{ij}$ can be relatively large. This will make the Taylor approximation of $\exp(b_{ij})$ about $b_{ij} = 0$ inaccurate near $\hat{b}_{ij}$. In principle, the Chebyshev approximation approach will also become less accurate when $c$ becomes very small, but at least in our simulations using the adaptive approximation interval of $[0, \log(1 + 1/c)]$ greatly improved performance.

In addition to being less accurate for fitting the log1p NMF model with small $c$, the Taylor approximation approach seems to result in an objective that is much more difficult to optimize. For the smallest values of $c$ used to fit the model, optimizing the log-likelihood with a Taylor approximation requires an order of magnitude more iterations to converge than the Chebyshev approach, which appears to have similar convergence properties to the exact log-likelihood (Figure \ref{fig:sim2_convergence}). Interestingly, we notice that across all simulation scenarios and with both the exact and approximate methods, optimizing the log1p NMF model with large $c$ requires far fewer iterations to converge than does optimizing with small $c$.

We present the results of an additional benchmarking study of the Chebyshev approximation approach on real data in Appendix \ref{app_sec:approx}.

\begin{figure}
    \centering
    \includegraphics[width=1.0\linewidth]{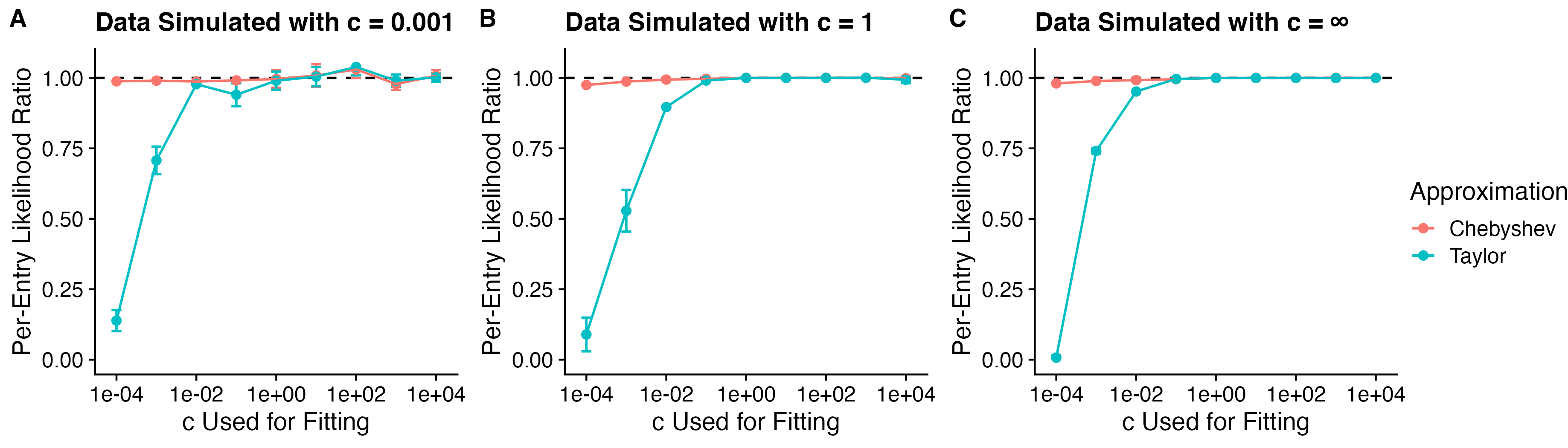}
    \caption{Per-entry likelihood ratio (the likelihood of the data when optimized with the approximate objective divided by the likelihood when optimized with the exact objective, raised to the power $1/np$) of factor models fit with $K = 5$ and varying settings of $c$. Points show means over simulation replicates ($30$ per generative regime; $100$ for $c = 10^{-3}$) and error bars show $95\%$ confidence intervals. All data are generated with $n = p = 500$ and $K = 5$. (A) Likelihood ratios when data are generated from the log1p model with $c = 10^{-3}$. (B) Likelihood ratios when data are generated from the log1p model with $c = 1$. (C) Likelihood ratios when data are generated from the log1p model with $c = \infty$.}
    \label{fig:approx_quality}
\end{figure}

\subsection{Adding row-specific scaling constants}

In practical applications, different rows of the data matrix may have somewhat different scales. For example, in single-cell RNA sequencing, the total number of mRNA molecules captured in each cell (sometimes referred to as ``library size'') is thought to mostly be a result of technical randomness in the measurement process \citep{sarkar2021separating}. Or, in text data, the total number of words in a document is typically less interesting than the \textit{relative} term usage in a document. To capture this kind of effect we adapt the model \eqref{eq:log1p_nmf_y}--\eqref{eq:log1p_nmf_B} to incorporate a fixed ``size factor'' $s_i$ for each row, replacing the assumption $y_{ij} \sim \textrm{Poisson}(\lambda_{ij})$ with $y_{ij} \sim \textrm{Poisson}(s_i \lambda_{ij})$ for fixed $s_i$ values. The resulting model is equivalent (in terms of estimating $\mathbf{L,F}$) to using a different value of $c$ for each row; specifically, it is equivalent to setting the link hyper-parameter to $c \cdot s_i$ in row $i$ while keeping the scaling constant $\alpha_c = \max(1, c)$ unchanged:
\begin{align}
y_{ij} &\sim \textrm{Poisson}(\lambda_{ij}) \notag \\
\alpha_c \log\left(1 + \frac{\lambda_{ij}}{c \cdot s_i}\right) &= b_{ij} \label{eq:log1p_nmf_with_s_reformulation} \\
\mathbf{B} &= \mathbf{L}\mathbf{F}^{\top}. \notag
\end{align}
Since $c \rightarrow 0$ implies that $c \cdot s_{i} \rightarrow 0$ and $c \rightarrow \infty$ implies that $c \cdot s_{i} \rightarrow \infty$, all results regarding the role of $c$ in the model \eqref{eq:log1p_nmf_y}--\eqref{eq:log1p_nmf_B} also hold for the modified model \eqref{eq:log1p_nmf_with_s_reformulation}. In our applications we set the size factors on the order of $1$ so that the role of $c$ is not obfuscated by the scale of $s_{i}$. In particular, we set
\begin{equation*}
    s_{i} = \frac{\sum_{j = 1}^{p}y_{ij}}{\frac{1}{n}\sum_{i = 1}^{n}\sum_{j = 1}^{p}y_{ij}},
\end{equation*}
so that $s_{1}, \dots, s_{n}$ have mean $1$.

Note that using size factors $s_i=\sum_j y_{ij}$ is closely connected to fitting a multinomial model to the data. This is because, with these size factors, the Poisson log-likelihood $y_{ij} \sim \text{Poisson}(s_i \lambda_{ij})$ is equivalent to the multinomial log-likelihood 
\begin{equation*}
    y_{i1}, \dots, y_{ip} \lvert s_i \sim \textrm{Multinomial}(s_i; \lambda_{i1}, \dots, \lambda_{ip}),
\end{equation*}
provided $\sum_j \lambda_{ij}=1$ \citep{baker1994multinomial}.
Furthermore, for standard Poisson NMF it turns out that the constraint $\sum_j \lambda_{ij}=1$ is automatically satisfied by the maximum likelihood estimates under the Poisson model, and so the multinomial model can be fit by simply fitting the Poisson model, without imposing the constraint directly \citep{carbonetto2021non}. In the log1p model this constraint is no longer guaranteed to be satisfied, so fitting the modified log1p Poisson model is not exactly equivalent to fitting a multinomial model with log1p link (although it could be viewed as an approximation to this). We leave investigation of methods for exactly fitting the multinomial model with log1p link to future work.




\subsection{Initialization}
\label{sec:init}

To initialize the log1p model with $K$ factors, we first optimize a rank-$1$ model to the data and then we initialize other entries to small positive numbers. That is, we fit each model in two steps:
\begin{enumerate}
    \item Fit the log1p NMF with $K = 1$ to $\mathbf{Y}$ using Algorithm \ref{alg:cap} with random initializations of $\mathbf{L}$ and $\mathbf{F}$. This yields initial estimates $\hat{\boldsymbol{l}}$ and $\hat{\boldsymbol{f}}$.
    \item Fit the log1p NMF with $K$ factors to $\mathbf{Y}$ using Algorithm \ref{alg:cap} with  initializations of $\mathbf{L} = [\, \hat{\boldsymbol{l}} \; \boldsymbol{u}_1 \; \dots \: \boldsymbol{u}_{K-1} \,]$ and $\mathbf{F}= [\, \hat{\boldsymbol{f}} \; \boldsymbol{v}_1 \; \dots \: \boldsymbol{v}_{K-1} \,]$, where $\boldsymbol{u}_1, \dots, \boldsymbol{u}_{K-1}$ and $\boldsymbol{v}_1, \dots, \boldsymbol{v}_{K-1}$ are column vectors with very small positive numbers in each entry.
\end{enumerate}
This initialization procedure was used to encourage the fitted models to find a ``baseline'' factor (i.e., a factor that is loaded on all samples), which could improve interpretability for quasi-multiplicative models. Such a baseline is especially natural for small values of $c$ (i.e., $c \ll 1$), as in this case $g(\lambda; c) \approx \log(\lambda) - \log(c)$, and so regardless of the data each sample intrinsically has a loading of approximately $1$ on a factor that is uniformly $-\log(c)$. In fact, as described in Appendix \ref{sec:app_geometry}, as $c \rightarrow 0^+$, the log1p NMF model becomes degenerate unless the constant vector $\mathbf{1}$ is contained in the cone spanned by the columns of $\mathbf{F}$, which can be achieved by adding such an intercept.

While this initialization approach may improve interpretability for quasi-multiplicative models, empirically we find that the log-likelihoods from this initialization strategy are no better, and perhaps somewhat worse, than those from random initializations (see Figure \ref{fig:bbc_init_stability}), so further study of initialization strategies may be warranted.


\section{Visualization and evaluation of log1p NMF fits}
Before evaluating the log1p NMF model on real data, we discuss a few matters related to the evaluation and visualization of these models.
\subsection{Re-scaling and re-ordering inferred loadings and factors}\label{sec:rescaling}
As with most matrix factorization models, there is some inherent non-identifiability in the scale of  the estimates of $\mathbf{L}$ and $\mathbf{F}$. In particular, one can multiply each element of $\mathbf{l}_{\cdot k}$ by any constant $a_k$, and divide each element of $\mathbf{f}_{\cdot k}$ by the same constant, and the resulting likelihood will be unchanged (because $\mathbf{L}\mathbf{F}^\top$ is unchanged). Thus, before plotting or comparing results from log1p models with different settings of $c$ (or different runs of the algorithm with the same $c$) it is important to scale the factors and loadings in some standardized way. 

Here, after obtaining estimates of $\mathbf{L}$ and $\mathbf{F}$, we scale each column of these matrices such that $\max_i l_{ik} = 1$ (i.e.,~we rescale using $a_k = 1/\max_i l_{ik}$). If one thinks of 
$l_{ik}$ as representing the ``membership'' of sample $i$ in factor $k$ then this corresponds to assuming that the maximum membership in each factor is 1. 
This also means that the interpretation of $f_{jk}$ from \eqref{eqn:lf_interpretation} is relative to ``maximal membership'' in factor $k$. Thus, for example, as $c \rightarrow \infty$, $f_{jk}$ represents the additive change in gene expression associated with full membership in factor $k$, and as $c\rightarrow 0$, $f_{jk}$ represents the log-fold change in gene expression associated with full membership in factor $k$.
We note that the scaling, while not affecting the fit, does affect visualization of the results. With our scaling, every column of $\mathbf{L}$ will show up somewhat equally in plots of the loadings (since each column has a maximum value of $1$), no matter how strong the impact of that factor on the data (i.e., how large the corresponding column of $\mathbf{F}$ is). This can be useful for highlighting subtler structure, since factors that have a small effect on the data will still be visible in the plot. However, in some settings it might be preferable to down-weight factors with small effects in the visualization (which could be achieved, for example, by using  $a_k = \max_j f_{jk}$ or $a_k = \sum_j f_{jk}$).

We discuss additional choices we made in the Results regarding visualization of the log1p NMF models in Appendix \ref{app_sec:rescaling}.

\subsection{Evaluating Sparsity}
One advantage of NMF methods, relative to other methods that do not have non-negative constraints (e.g., PCA), is that they often produce sparse estimates of $\mathbf{L}$ and $\mathbf{F}$. Below, we assess how different settings of $c$ impact the sparsity of the estimated $\mathbf{L}$ and $\mathbf{F}$. To measure sparsity, we use Hoyer's metric \citep{hoyer2004non}, defined for a non-negative $n$-vector $\mathbf{x}$ (with at least $1$ non-zero entry) as 
\begin{equation}\label{eq:hoyer_def}
        H(\mathbf{x}) = \frac{\sqrt{n} - \lvert\lvert \mathbf{x} \rvert\rvert_{1} / \lvert\lvert \mathbf{x} \rvert\rvert_{2} }{\sqrt{n} - 1}.
\end{equation}
While $H$ is arguably measuring skewness, rather than sparsity, it has the advantage that it is invariant to re-scaling of $\mathbf{x}$, which is particularly important when comparing log1p NMF fits with different values of $c$ (see \cite{hurley2009comparing} for a broad discussion of sparsity metrics). For simplicity, when examining the sparsity of the matrices $\mathbf{L}$ and $\mathbf{F}$, we examine the mean sparsity of the columns of these matrices.

\section{Results}
\subsection{Simulations}\label{sec:sim}
\subsubsection{A simple toy example}\label{sec:sim_toy}

We first show that even when the true data generating process is such that the log1p NMF model is correctly specified for {\it any} $c$, the value of $c$ used to fit the model affects downstream interpretation.

To illustrate this, we generated a $1000 \times 400$ data matrix $\mathbf{Y}$ from a rank-3 Poisson NMF model. We first defined a ``baseline'' factor $\mathbf{f}_0$ and two sparse change factors $\mathbf{f}_A$ and $\mathbf{f}_B$. Ordering the $p = 400$ features in four blocks of $100$, we defined these factors as

\begin{align*}
\mathbf{f}_0 &= \big(\underbrace{2, \ldots, 2}_{100},\ \underbrace{1000, \ldots, 1000}_{100},\ \underbrace{2, \ldots, 2}_{100},\ \underbrace{1000, \ldots, 1000}_{100}\big), \\
\mathbf{f}_A &= \big(\underbrace{18, \ldots, 18}_{100},\ \underbrace{100, \ldots, 100}_{100},\ \underbrace{0, \ldots, 0}_{100},\ \underbrace{0, \ldots, 0}_{100}\big), \\
\mathbf{f}_B &= \big(\underbrace{0, \ldots, 0}_{100},\ \underbrace{0, \ldots, 0}_{100},\ \underbrace{18, \ldots, 18}_{100},\ \underbrace{100, \ldots, 100}_{100}\big).
\end{align*}
We designed the change factors $\mathbf{f}_A$ and $\mathbf{f}_B$ so that each one acts on two kinds of features,
  distinguished by their baseline rate. The two kinds respond very differently when
  their corresponding change factor has loading of $1$:

  \begin{center}
  \begin{tabular}{lcccc}
  \toprule
  Feature Type  & Baseline Rate ($\mathbf{f}_0$) & Rate if Loading $=1$ (on $\mathbf{f}_A, \mathbf{f}_B$)& Fold $\Delta$ & Absolute $\Delta$ \\
  \midrule
  High-Baseline & $1000$ & $1100$ & $1.1\times$ & $+100$ \\
  Low-Baseline  & $2$    & $20$   & $10\times$  & $+18$  \\
  \bottomrule
  \end{tabular}
  \end{center} 
  That is, high-baseline features undergo a large \emph{absolute} change (i.e., $\mathbf{f}_{X} - \mathbf{f}_{0}$) but a small
  \emph{relative} (fold) change (i.e., $\mathbf{f}_{X} / \mathbf{f}_{0}$), whereas low-baseline features undergo a large
  \emph{relative} change but a small \emph{absolute} change.

We then defined four groups of samples based on their loadings on these factors:

\begin{equation*}
\mathbf{l}_{i \cdot} =
\begin{cases}
(1, 0, 0) & \text{if } 1 \leq i \leq 250 \\
(1, 1, 0)  & \text{if } 250 < i \leq 500 \\
(1, 0, 1)  & \text{if } 500 < i \leq 750 \\
(1, 1, 1)  & \text{if } 750 < i \leq 1000,
\end{cases}
\end{equation*}
where each loading is ordered $(l_{i0}, l_{iA}, l_{iB})$. For each sample $i$ we defined $$\boldsymbol{\lambda}_i = l_{i0} \mathbf{f}_0 + l_{iA} \mathbf{f}_A + l_{iB} \mathbf{f}_B,$$ and drew $y_{ij} \sim \textrm{Poisson}(\lambda_{ij})$ independently across $i$ and $j$. The group definitions are shown in Figure \ref{fig:factor_sim}A.

\begin{figure}
    \centering
    \includegraphics[width=1.0\linewidth]{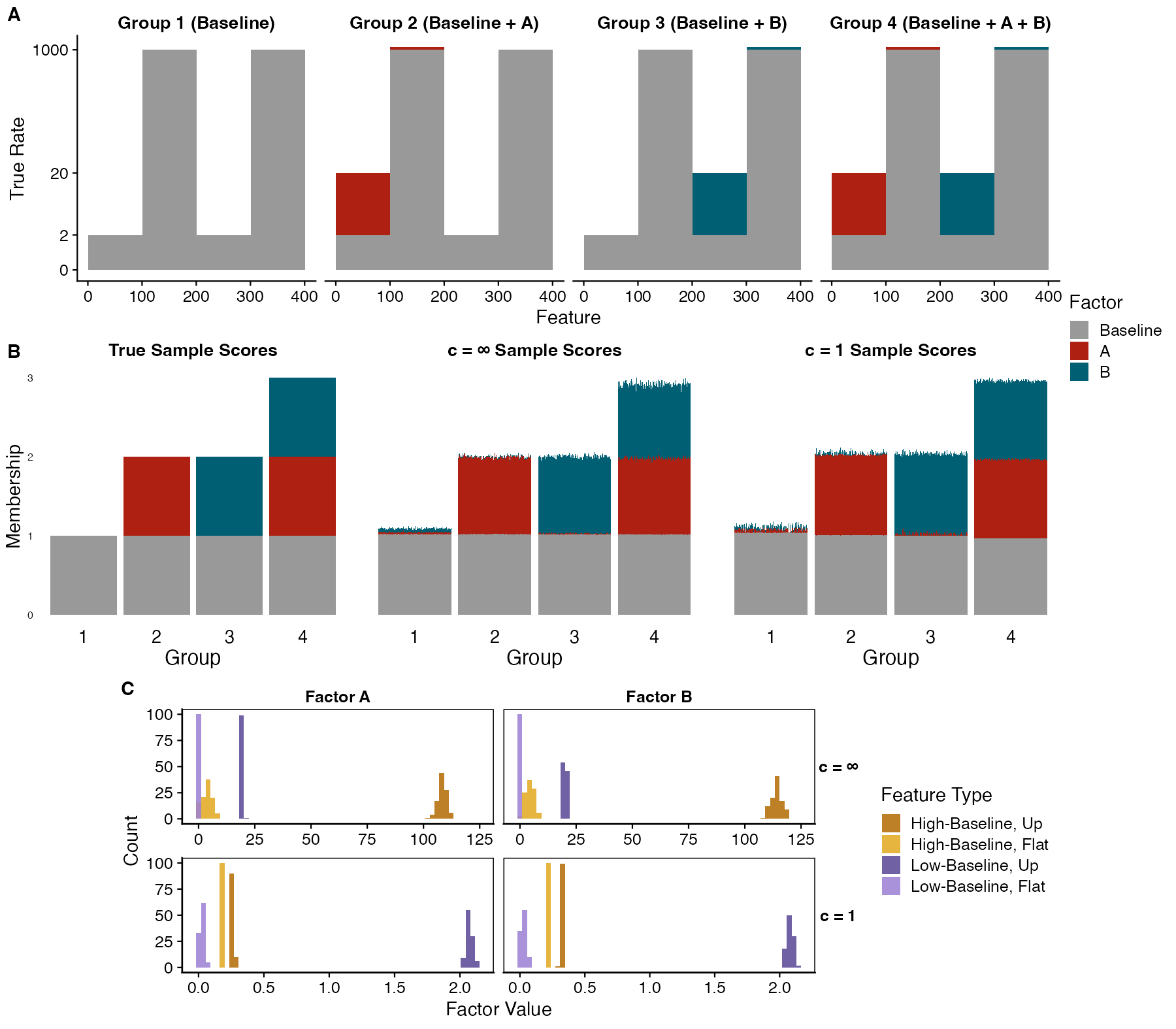}
    \caption{Toy example results. (A) True feature rates in each of the four sample groups. (B) True sample scores, and estimated sample scores ($\mathbf{L}$) for log1p NMF with $c = \infty$ and $c = 1$. (C) Distributions of estimated factor values (columns of $\mathbf{F}$) for the two change factors, with features classified by their baseline level and by whether the change factor increases them.}
    \label{fig:factor_sim}
\end{figure}

While we have written this model using an additive link, it can equivalently be written as a log1p NMF model with any value of $c$. Specifically, given $c$, suppose we keep the loadings $\mathbf{l}_{i \cdot}$ unchanged and replace the factors as follows:
\begin{align*}
\mathbf{f}_0^{(c)} &= \alpha_c \log\!\left(1 + \tfrac{1}{c}\,\mathbf{f}_0\right), \\
\mathbf{f}_A^{(c)} &= \alpha_c \left[\log\!\left(1 + \tfrac{1}{c}(\mathbf{f}_0 + \mathbf{f}_A)\right) - \log\!\left(1 + \tfrac{1}{c}\,\mathbf{f}_0\right)\right], \\
\mathbf{f}_B^{(c)} &= \alpha_c \left[\log\!\left(1 + \tfrac{1}{c}(\mathbf{f}_0 + \mathbf{f}_B)\right) - \log\!\left(1 + \tfrac{1}{c}\,\mathbf{f}_0\right)\right],
\end{align*}
where $\log(\cdot)$ is applied elementwise. With these factors, $$\alpha_c \log(1 + \lambda_{ij} / c) = l_{i0} f^{(c)}_{0j} + l_{iA} f^{(c)}_{Aj} + l_{iB} f^{(c)}_{Bj}$$ for every sample $i$ and feature $j$. Such a construction exists for any $c$ only because factors $A$ and $B$ have disjoint support, in the sense that no feature is non-zero in both change factors, so the nonlinear link never has to represent an interaction between them. (And also because the loadings $l$ take only two values, $0$ and $1$.)

We show the results of fitting the log1p model with $c = 1$ and $c = \infty$ in Figure \ref{fig:factor_sim}B,C. Unsurprisingly, since models with both values of $c$ are correctly specified, the estimated loadings returned by the two fits are nearly indistinguishable, and match the true values closely (Figure \ref{fig:factor_sim}B). However, the estimated factors differ greatly between the $c = 1$ and $c = \infty$ fits (Figure \ref{fig:factor_sim}C). Indeed, because the link function is additive when $c = \infty$, the top features in each factor for the $c = \infty$ fit are those that have a high baseline ($1000$) and increase by $100$ when the factor has a loading of $1$. In absolute terms, these features increase by much more than the features with a low baseline ($2$) that increase by only $18$. Conversely, because the link function with $c = 1$ behaves closely to a log, which represents changes on a relative scale, we see that the highest features in the $c = 1$ estimated factors are the low-baseline features. These features experience a fold change of $10$ when the change factors are loaded at $1$, whereas the high-baseline features experience a fold change of only $1.1$.

This simple example highlights that even when changing the value of $c$ does not change the goodness of fit of the model to the data, it can substantially change the interpretation of the results. In particular, when $c$ is small, the log1p NMF model will tend to produce factors that have high scores on features with large fold changes in some samples relative to others, whereas with large $c$ the estimated factor values tend to emphasize large absolute changes. The example also illustrates that, at least in principle, $c$ can be non-identifiable, a point we return to below.

\subsubsection{\texorpdfstring{Impact of \(c\) on sparsity and log-likelihood of fit}
{Impact of c on sparsity and log-likelihood of fit}}
\label{sec:sim2}

In preliminary investigations of real datasets, fitting the log1p model using a range of values of $c$, we observed the following two general trends. First, fits using larger values of $c$ tend to produce sparser estimates of both $\mathbf{L}$ and $\mathbf{F}$ (by Hoyer's measure of sparsity \eqref{eq:hoyer_def}; Figure \ref{fig:sparsity}). Second, fits using smaller values of $c$ often produced higher log-likelihoods (e.g., Figure \ref{fig:bbc_init_stability}). In this section we illustrate these trends on simulated data and discuss potential explanations.

For these simulations we generated counts  from the log1p NMF model \eqref{eq:log1p_nmf_y}--\eqref{eq:log1p_nmf_B} with $c_{\textrm{true}} \in \{10^{-3}, 1, \infty\}$, representing the quasi-multiplicative, intermediate, and additive regimes. For each of $30$ random seeds we simulated an $n \times p$ count matrix with $n = p = 500$ and $K = 5$ factors, as follows:
\begin{enumerate}
    \item The rows of the loadings matrix $\mathbf{L} \in \mathbb{R}_{\geq 0}^{n \times K}$ were drawn independently as $$\mathbf{l}_{i \cdot} \sim \textrm{Dirichlet}(a, \dots, a)$$ with $a = 0.3$. 
    \item The entries of the factor matrix $\mathbf{F} \in \mathbb{R}_{\geq 0}^{p \times K}$ were generated as a mixture of a point mass at $0$ and a truncated log-$t$ distribution. That is, for each feature $j$ and factor $k$ we drew
    \begin{equation*}
        u_{jk} \sim \textrm{Uniform}(0, 1), \qquad \phi_{jk} \sim t_4
    \end{equation*}
    independently, and set
    \begin{align*}
        \widetilde{m}_{jk} &= \exp\left(\mu^{(c_{\textrm{true}})} +\sigma^{(c_{\textrm{true}})} \phi_{jk}\right), \\
        m_{jk} &= \min\left(\max\left(\widetilde{m}_{jk},\; T_{\textrm{low}}^{(c_{\textrm{true}})}\right),\; T_{\textrm{high}}^{(c_{\textrm{true}})}\right), \\
        f_{jk} &= m_{jk} \, \mathbf{1}\left\{u_{jk} > \pi_0^{(c_{\textrm{true}})}\right\}.
    \end{align*}
 In the case where this procedure resulted in a feature $j$ with $u_{jk} \leq \pi_0^{(c_{\textrm{true}})}$ for every $k$, we calculated $k^{*} = \arg\max_k u_{jk}$, and set $f_{jk^{*}} = m_{jk^{*}}$, leaving the remaining entries at zero. 
    \item With $\mathbf{B} = \mathbf{L}\mathbf{F}^{\top}$, we drew $y_{ij} \sim \textrm{Poisson}(\lambda_{ij})$ with $\lambda_{ij} = g^{-1}(b_{ij}; c_{\textrm{true}})$. In the rare case that an entire row or column of $\mathbf{Y}$ was $0$, the Poisson draw was repeated for that row or column.
\end{enumerate}
The constants $\left(\pi_0^{(c_{\textrm{true}})}, \mu^{(c_{\textrm{true}})}, \sigma^{(c_{\textrm{true}})}, T_{\textrm{low}}^{(c_{\textrm{true}})}, T_{\textrm{high}}^{(c_{\textrm{true}})}\right)$ were calibrated at each $c_{\textrm{true}}$ to match the characteristics of the pancreas dataset of Section \ref{sec:pancreas}, as described in Appendix \ref{app:sim2_calibration}.

To each simulated dataset we fit the log1p NMF model with $K = 5$ and optimized the exact log-likelihood at every value of $c$ on the grid $\{10^{-4}, 10^{-3}, \dots, 10^{4}\}$, together with $c = \infty$. All fits used the rank-1 initialization of Section \ref{sec:init} and were fit until the change in the log-likelihood between successive iterations fell below $10^{-6}$.

The mean Hoyer sparsity for the estimated $\mathbf{L,F}$ are shown in Figure \ref{fig:sim2}A,B. In both cases $H$ is a nearly monotonically increasing function of $c$. That is, large $c$ fits are sparser than small $c$ fits {\it regardless of the value of $c_{\textrm{true}}$}. 

The consistency of this trend (which also occurs in real data; see Figure \ref{fig:sparsity}) suggests that there may be a deeper theoretical explanation. Although directly analyzing the maximum likelihood estimate of $\mathbf{L}$ and $\mathbf{F}$ as a function of $c$ seems challenging,
we have developed a less direct argument that we believe helps to explain the trend. 

To describe this argument, we first relate the sparsity of $\mathbf{L}$ and $\mathbf{F}$ to a simpler quantity, the peak-to-mean ratio of the matrix $\mathbf{B}=\mathbf{L}\mathbf{F}^\top$, defined as 
\begin{equation}\label{eq:pm_def}
    \operatorname{pm}(\mathbf{B}) = \frac{\max(\mathbf{B})}{\operatorname{mean}(\mathbf{B})}.
\end{equation}
 In Theorem \ref{cor:hoyer_sum}, we bound the sparsity of each column of $\mathbf{L}$ and $\mathbf{F}$ in terms of an increasing function of $\operatorname{pm}(\mathbf{B})$. Specifically, for large $n$ and $p$, the Theorem implies (to a good approximation)
\begin{equation}\label{eq:hoyer_res}
    H(\mathbf{l}_{\cdot k}) + H(\mathbf{f}_{\cdot k}) \leq 2\left(1 - \left(\frac{w_k}{\operatorname{pm}(\mathbf{B})}\right)^{1/4}\right),
\end{equation}
where $w_k = \frac{\operatorname{mean}(\mathbf{l}_{\cdot k})\operatorname{mean}(\mathbf{f}_{\cdot k})}{\operatorname{mean}(\mathbf{B})}$ and $\sum_{k = 1}^K w_k = 1$. 

 The next key observation is that the pm value of the ``true'' $\mathbf{B}$, in the log1p model with parameter $c$, is monotonic in $c$. 
Specifically, given Poisson rates $\mathbf{\Lambda}$, define the corresponding true value of $\mathbf{B}_c$ entrywise as  $b_{ij}(c) = g(\lambda_{ij}; c)$. Proposition \ref{prop:budget_monotone} shows that for any $0 < c_1 < c_2 \leq \infty$,
\begin{equation} \label{eq:Bmonotone}
    \operatorname{pm}(\mathbf{B}_{c_1}) \leq \operatorname{pm}(\mathbf{B}_{c_2}).
\end{equation}
While the fitted $\mathbf{B}$ will inevitably differ from the true $\mathbf{B}$, in our simulation the 
$\operatorname{pm}$ of the fitted $\mathbf{B}$ closely matches the $\operatorname{pm}$ of the true $\mathbf{B}$ (Figure \ref{fig:sim2}C).

Putting these observations together, Equation \eqref{eq:Bmonotone} shows that 
increasing $c$ increases the peak-to-mean ratio of the true $\mathbf{B}$, which will tend to increase the peak-to-mean ratio of the fitted $\mathbf{B}$, which by 
\eqref{eq:hoyer_res} may allow the sparsity of the estimated $\mathbf{L}$ and $\mathbf{F}$ to increase. 

Next we turn to how the log-likelihood of the model depends on the assumed value of $c$. Interestingly, even when the value of $c$ used to fit the model does not match the simulation truth, the log1p NMF model appears to provide reasonably accurate fitted values of $\mathbf{B}$, both in terms of the log-likelihood of the data (Figure \ref{fig:sim2}D) and in terms of the reconstruction error of the true $\mathbf{B}$ (Figure \ref{fig:rmse_sim}A). The largest effects of model mis-specification occur when $c_{\textrm{true}}$ is small and the value of $c$ used for fitting is large. 

Perhaps the most surprising observation is that the log-likelihood is not always maximized at the true value of $c$, and indeed tends to increase slightly as $c$ gets very small ($c = 10^{-4}$ achieves the highest mean and median log-likelihoods for $c_{\textrm{true}} = 10^{-3}$ and $c_{\textrm{true}} = \infty$). We do not fully understand this observation, but our best explanation is that, for fixed $K$, the non-linearity of the log1p NMF model with small $c$ provides some additional flexibility compared with the more linear model (large $c$). In Appendix \ref{sec:app_geometry} we provide some preliminary theoretical support for this. Focusing on the simplest case where $K=2$, we characterize, as a function of $\mathbf{F}$, the set of directions the rate vector $\boldsymbol{\lambda}_i$ can achieve in the limit $c \rightarrow 0^+$.  We show that there are choices of $\mathbf{F}$ (in the log1p NMF model with $c \rightarrow 0^+$) for which this achievable set is larger than the achievable set of the standard (linear) Poisson NMF model under \textit{any} choice of $\mathbf{F}$. This suggests that, for a fixed value of $K$, the log1p NMF model with small $c$ has more ``effective'' parameters than the standard Poisson NMF model. We thus suspect that there may be issues with identifying $c$ in practice, and instead leave selection of this value as an exploratory choice.
\begin{figure}
    \centering
    \includegraphics[width=0.87\linewidth]{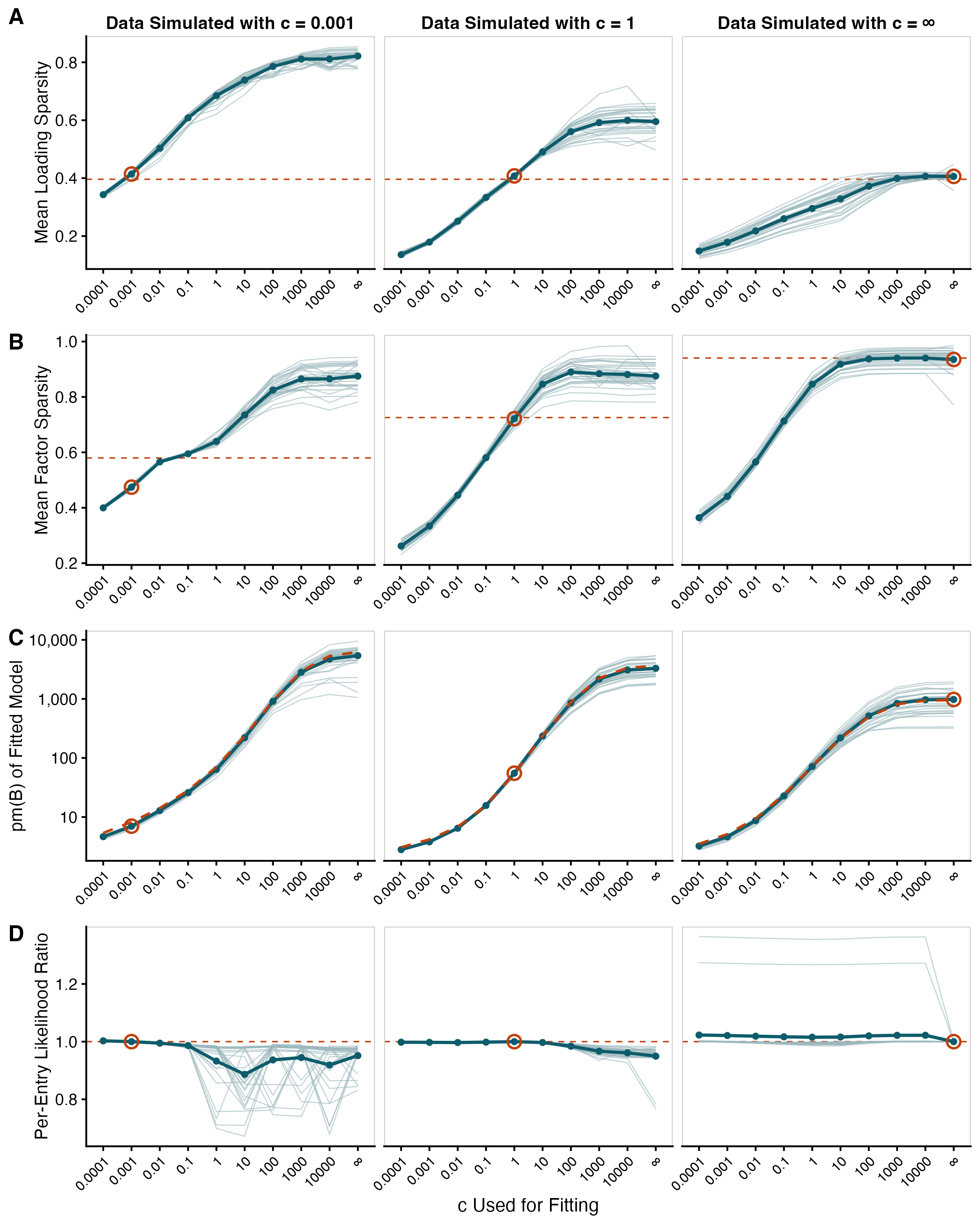}
    \caption{Results of mis-specification simulation: Thin lines show individual simulation replicates, thick lines their mean, and the open circle marks $c_{\textrm{fit}} = c_{\textrm{true}}$. (A, B) Mean Hoyer sparsity of the columns of $\mathbf{L}$ and of $\mathbf{F}$; the dashed line shows the same quantity for the true loadings and factors. (C) Peak-to-mean ratio $\operatorname{pm}(\mathbf{B})$ of the fitted model; the dashed line shows the value predicted from the true rates alone --- $\operatorname{pm}(\log(1 + \boldsymbol{\Lambda}/c))$ at finite fitting $c$, and $\operatorname{pm}(\boldsymbol{\Lambda})$ at $c = \infty$ --- averaged over replicates. (D) Per-entry likelihood ratio relative to the same replicate's fit at $c_{\textrm{true}}$, i.e.\ $\exp\{[\ell_{\textrm{log1p}}(c) - \ell_{\textrm{log1p}}(c_{\textrm{true}})]/np\}$; the dashed line marks $1$.}
    \label{fig:sim2}
\end{figure}

\subsection{Real Data Applications}
\label{sec:real_data}

\subsubsection{MCF-7 bulk RNA-seq}
As a simple initial real-data example, we examine bulk RNA-sequencing data from the human breast carcinoma cell line MCF-7 under four conditions \citep{sanford2020gene}. This dataset is particularly well suited for evaluating log1p NMF across values of $c$, as it was explicitly generated to probe whether the effects of multiple treatments on gene expression combine additively or multiplicatively. Briefly, MCF-7 cells were either treated with all-trans retinoic acid (RA), transforming growth factor beta (TGF-$\beta$), or their combination at various concentrations. Some cells were also kept as a control (CTRL). In total, mRNA expression across $41$ samples ($10$ CTRL, $11$ RA, $9$ TGF-$\beta$, $11$ RA $+$ TGF-$\beta$) was measured. After filtering down to mRNA corresponding to protein-coding genes, and excluding any genes that were not detected in at least $4$ samples, we obtained a $41 \times 16{,}733$ count matrix with approximately $9\%$ $0$ entries. 

Figure \ref{fig:mcf7} shows the results of fitting the log1p model with $c=1$ and $c = \infty$ (i.e., standard Poisson NMF) to the MCF-7 data with rank $3$. The inferred  ``sample scores'' ($\mathbf{L}$)  are shown in Figure \ref{fig:mcf7} A,B. Although different, the two model fits share some key similarities: both use one factor to capture treatment with RA (k2, orange), and another to capture treatment with TGF-$\beta$ (k3, blue). Both also use a factor to capture the control condition (k1, black). However, for $c=1$ all samples have a high score on this factor, so it effectively acts as a ``baseline'' factor,  whereas for $c=\infty$ the sample scores decrease  in the single-treatment groups, and are close to 0 in the double-treatment (RA$+$TGF-$\beta$) group.

The gene scores ($\mathbf{F}$) for the treatment-related factors (k2, k3) in each model are shown in Figure \ref{fig:mcf7}C,D. It is immediately visually apparent that the gene scores for the two factors are more correlated under $c=\infty$ than for $c=1$ (Spearman $\rho \approx 0.91$ vs. $0.67$). Coloring the genes according to which ones are differentially expressed in each treatment vs. control (using {\it DESeq2}; \cite{love2014moderated}) shows that the genes with highest scores for $c=1$ in each factor are generally identified as differentially expressed by {\it DESeq2}, consistent with the interpretation of the sample scores for this model as capturing treatment effects.  

One way to view these results is that the $c=1$ results separate out a baseline factor
from the treatment factors, whereas the $c=\infty$ results absorb some of the baseline into all three factors, which causes the gene scores for different factors to be highly correlated, and more focused on the highest-expressed genes.
Note that this happens for $c=\infty$ even though we used an initialization strategy that might encourage it to separate out a baseline factor.
The results for $c=\infty$ are not ``wrong'' and the sample scores do capture the treatment structure in the data. However, by separating out a baseline factor, the $c=1$ results make it easier to identify the ``key genes'' that are responding to each treatment, simply by looking at which genes have the highest gene scores in each factor (see Table in Figure \ref{fig:mcf7}). For example, the top two genes in factor k2, \textsl{CYP26B1} and \textsl{CYP26A1}, are the main enzymes responsible for metabolizing RA in the human body \citep{topletz2012comparison}, and the next two genes with the highest scores, \textsl{SLC5A5} and \textsl{STRA6}, have been previously implicated in cellular response to RA: \textsl{STRA6} is believed to be the main protein responsible for transport of RA across cell membranes \citep{kelly2015stra6}, and \textsl{SLC5A5} is known to be up-regulated in response to RA in MCF-7 \citep{kogai2000retinoic}. 
  For $c=\infty$ the genes with the highest scores are very similar across factors, and tend to be genes that are highly expressed across all samples.
It is possible that a more sophisticated approach to identifying ``key genes'' could help here; e.g., ~see \cite{carbonetto.gomde}.

\setlength{\tabcolsep}{2.5pt}
\begin{figure}[htbp]
    \centering
    \begin{minipage}{\linewidth}
      \centering
      \includegraphics[width=0.64\linewidth]{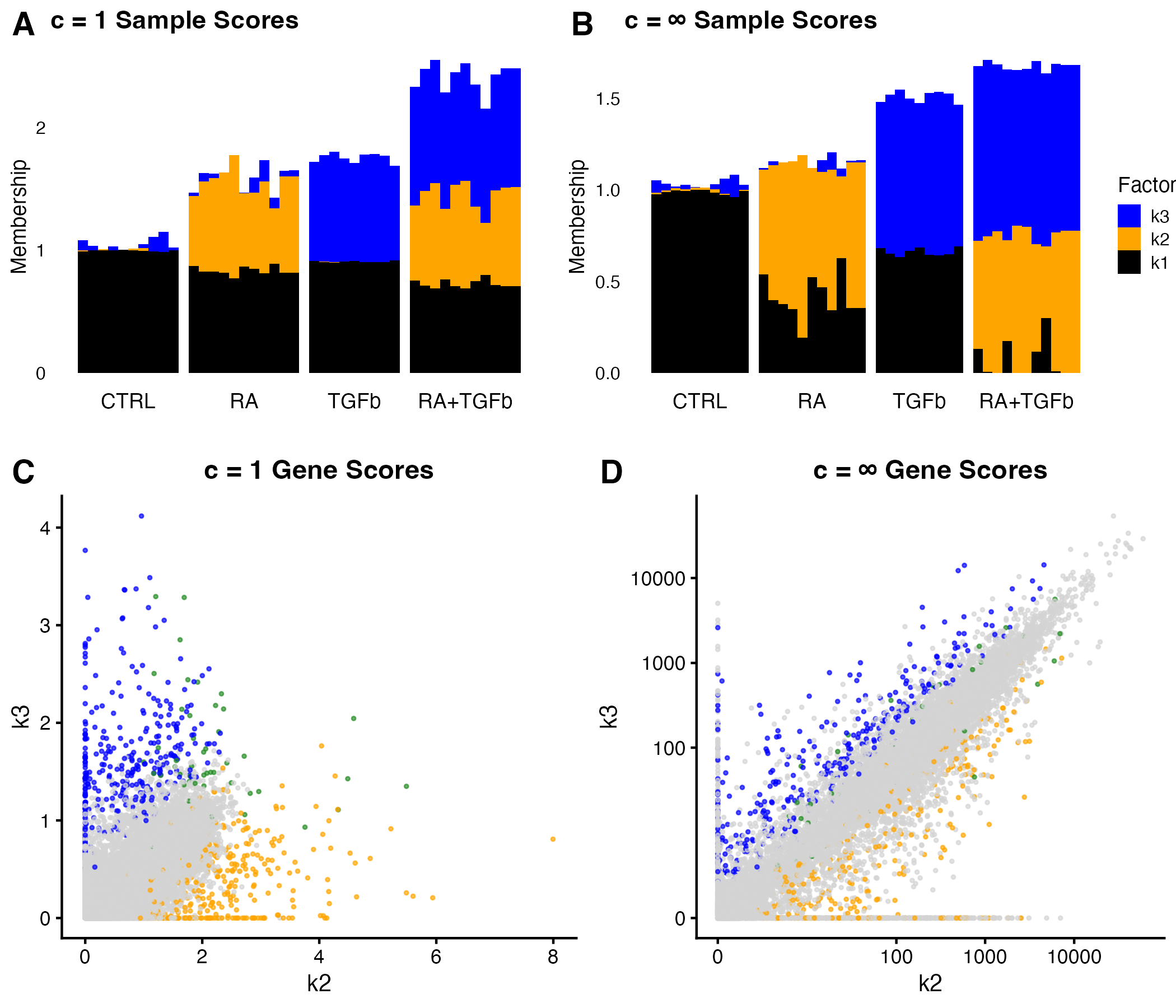}
    \end{minipage}
    
    \begin{minipage}{\linewidth}

\centering
{\footnotesize\renewcommand{\arraystretch}{1.02}%
\begin{tabularx}{\textwidth}{@{} l I I @{}}
\toprule
Factor & { \normalfont Top Genes -- $c = 1$} & {\normalfont Top Genes -- $c = \infty$}\\
\midrule
k1 & KRT8, COX1, ND4, ND5, ATP6, CYTB, COX2, COX3, ACTB, EEF1A1 \newline {\upshape\scriptsize GO: aerobic electron transport chain} & \cellcolor{gray!10}{KRT8, COX1, ND4, ND5, ATP6, CYTB, COX2, COX3, ACTB, EEF1A1 \newline {\upshape\scriptsize GO: aerobic electron transport chain}}\\
k2 & \cellcolor{gray!10}{CYP26B1, CYP26A1, SLC5A5, STRA6, SHROOM1, HOXA1, LCN2, KDR, PTPRH, SERPINA3 \newline {\upshape\scriptsize GO: negative regulation of retinoic acid receptor signaling pathway}} & ND4, CYTB, ATP6, ND5, COX1, COX2, EEF1A1, COX3, ACTB, PABPC1 \newline {\upshape\scriptsize GO: aerobic electron transport chain}\\
k3 & TP63, GABRP, TCIM, SPOCK1, DOCK4, MGP, KIAA2012, COL4A3, NCF2, PTPRB \newline {\upshape\scriptsize GO: osteoblast differentiation} & \cellcolor{gray!10}{KRT8, COX1, ND4, ACTB, EEF1A1, ACTG1, COX2, EEF2, ND5, ATP6 \newline {\upshape\scriptsize GO: oxidative phosphorylation}}\\
\bottomrule
\end{tabularx}
}
    \end{minipage}
    
    \caption{Combined figure and table for the MCF-7 analysis. (A-B) Visual representation of estimated $\mathbf{L}$ matrices for log1p NMF with $c = 1$ and $c = \infty$. Each column represents a row of $\mathbf{L}$, where each color corresponds to a column of $\mathbf{L}$.  (C-D) Scatterplots of factors $2$ and $3$. Each point corresponds to a single gene (row of $\mathbf{F}$).  Points are colored based on results of differential expression using DESeq2 \citep{love2014moderated}. Points in green have Benjamini-Hochberg \citep{benjamini1995controlling} adjusted p-values $< 0.01$ and log2FC $> 1$ in both the RA and TGF-$\beta$ groups. Points in orange meet these conditions in only the RA group, points in blue meet these conditions only in the TGF-$\beta$ group, and points in grey meet these conditions in neither group. (Table): Top $10$ genes in each column of $\mathbf{F}$ from the fitted models. Below each gene list, in upright type, is the most significantly enriched Gene Ontology \citep{ashburner2000gene, aleksander2023gene} biological process term among the factor's top $30$ genes, computed by a hypergeometric test against all genes in the dataset \citep{boyle2004go}, with Benjamini--Hochberg adjusted $p < 0.05$.}
    \label{fig:mcf7}
\end{figure}

\subsubsection{Murine pancreas single cell RNA-seq data}
\label{sec:pancreas}

Our second example is a more complex single cell RNA-seq dataset, derived from murine pancreas cells stimulated with cytokines \citep{stancill2021single}. Briefly, cells isolated from the pancreas of eight mice were first pooled and then separated into four samples. One sample was treated with interleukin-1 beta (IL-1$\beta$), another was treated with interferon gamma (IFN$\gamma$), a third was treated with both IL-1$\beta$ and IFN$\gamma$, and a final sample was left untreated. After filtering for cells with between $2000$ and $60000$ unique molecular identifiers (UMIs), removing cells with greater than $10\%$ of UMIs coming from mitochondrial genes, removing genes expressed in fewer than $3$ cells, and removing mitochondrial genes, genes coding for ribosomal proteins, and the gene \textsl{Malat1}, we obtained a $7{,}606 \times 18{,}195$ (cells $\times$ genes) count matrix with approximately $82\%$ $0$ counts. These data contain eight different ``cell types'' (``acinar'', ``ductal'', ``endothelial/mesenchymal'', ``macrophage'', ``alpha'', ``beta'', ``delta'', and ``gamma'' cells, with labels assigned based on the marker genes used 
in \cite{stancill2021single}), and we are interested in how the different NMF models capture this cell type structure in addition to the effects of cytokine treatment.

Setting $K = 13$, we fit the log1p NMF model with $c = 1$ and $c = \infty$. Both models produced some cell type-related factors and some treatment-related factors, with little overlap between these sets (the exception being factor k10 for $c=\infty$). Figure \ref{fig:pancreas_structure_c} shows the cell scores ($\mathbf{L}$) for the cell type-related factors and Figure \ref{fig:pancreas_structure_t} shows them for the treatment-related factors, and we now discuss these results in turn.  

Examining the cell type-related factors,
both models show clear differences among cell types, but the two representations are nonetheless quite different. In broad terms, the $c=\infty$ results are more ``clustered'', with each cell type being associated with one or two factors, whereas the $c=1$ results are more ``modular'', with several cell types being represented as a combination of three or more factors. If one is primarily interested in clustering the cells into cell types then $c=\infty$ results may appear cleaner. On the other hand, if one is interested in understanding the underlying processes that define cell types, and which processes are shared among cell types, the $c=1$ results may be more useful. For example, in $c=1$ the factor k8 (red) is strongly present in most alpha, delta and gamma cells, and also present in beta cells. Similarly, the factor k13 is strongly present in gamma cells, but also present in delta and alpha cells. These factors therefore represent processes that are shared across these cell types (all of which are islet cells). 

To further illustrate these differences, we consider in more detail the  delta and gamma cells, which are 
 transcriptionally very similar, with just a few genes -- most notably \textsl{Ppy}, \textsl{Sst}, and \textsl{Rbp4} -- showing strong differences in expression  (Figure \ref{fig:lsa_dg}C). The \textsl{Ppy} gene is the canonical marker gene for gamma cells (which, indeed, are also called PP cells; \cite{Inzani_Rindi_Serra_2020}), while \textsl{Sst} and \textsl{Rbp4} are canonical markers for delta cells in the murine pancreas \citep{thielert2025decoding}. The $c=\infty$ model essentially assigns a factor to each cell type (k3, k13) with these factors having correlated gene scores (Figure \ref{fig:lsa_dg}B) and the key genes lying away from the strong main diagonal. The $c=1$ model also captures the difference between the cells with two factors (k3,k13) but these factors show much less correlation (no strong diagonal) and focus primarily on the key genes. And k13, while strongest in gamma cells, is also present in other cells, again perhaps highlighting some shared processes. At a high level, the $c=\infty$ model captures similarity/differences between the gamma and delta cells by using two similar (but different) factors, whereas the $c=1$ model captures it by their shared membership in rather different factors. This type of behavior explains why, in Figure \ref{fig:pancreas_structure_c}, the cell scores for $c=\infty$ more clearly delineate cell types, whereas those for $c=1$ better convey which cell types are similar to one another.
 
Besides this high-level difference between the $c=1$ and $c = \infty$ fits, we note two other differences between the results. First, the $c=\infty$ results suggest a gradient of variation across alpha cells that is not evident in $c=1$. This is due to a factor (k10) that is also related to  treatment, and is discussed further below. 
Second, the $c=\infty$ results highlight a gradient of variation in the endothelial and 
mesenchymal cells (factor k9 vs. k12) that is absent or less evident in the $c=1$ results.  The top genes in these two factors include genes specifically related to endothelial cells (e.g.,  \textsl{Igfbp7}; \cite{van2012proteomic, he2024igfbp7}) and to mesenchymal cells (e.g., \textsl{Vim}; \cite{usman2021vimentin}) and this gradient may be related to the documented endothelial to mesenchymal transition \citep{piera2019endothelial}.  

Turning now to the treatment-associated factors (Figure \ref{fig:pancreas_structure_t}) the $c=1$ results (panel A) are somewhat analogous to our first data example above: one factor (k10) captures treatment with IL-1$\beta$ and another factor (k4) captures treatment with IFN$\gamma$; samples treated with both IL-1$\beta$ and IFN$\gamma$ show membership in both these factors. Many of the key genes for these treatment-associated factors are biologically connected with the treatments. For example, the top gene in factor k4 (associated with IFN$\gamma$ treatment) is \textsl{Cxcl10}, which is also known as interferon gamma induced protein 10 \citep{liu2011cxcl10}, and the top genes in factor k10 include several genes known to be regulated by IL-1$\beta$ (e.g., \textsl{Lcn2}; \cite{hu2015lipocalin}, \textsl{Cebpd}; \cite{moore2012transcription}, \textsl{Cxcl1}; \cite{diana2014macrophages}).

In comparison, for $c=\infty$, the treatment associated factors are harder to interpret (Figure \ref{fig:pancreas_structure_t}B). One factor (k10) is specific to IL-1$\beta$ treatment, but it is also exclusive to alpha cells (Figure \ref{fig:pancreas_structure_t}D). As a result, the top gene in this factor (by far) is \textsl{Gcg}, which is a canonical marker of alpha cells \citep{stancill2021single}, and unlikely to be specifically related to treatment. The factor most associated with IFN$\gamma$ treatment (k7) is also present in untreated cells and in cells treated with IL-1$\beta$, but largely absent from the cells treated with both cytokines, making it hard to interpret in terms of the treatments. 
One possible reason for these results is that the effects of cell type and treatment in these data may combine more multiplicatively than additively, and thus align better with $c=1$ than $c=\infty$.

\setlength{\tabcolsep}{2.5pt}
\begin{figure}[htbp]
    \centering
    \begin{minipage}{\linewidth}
      \centering
      \includegraphics[width=0.7\linewidth]{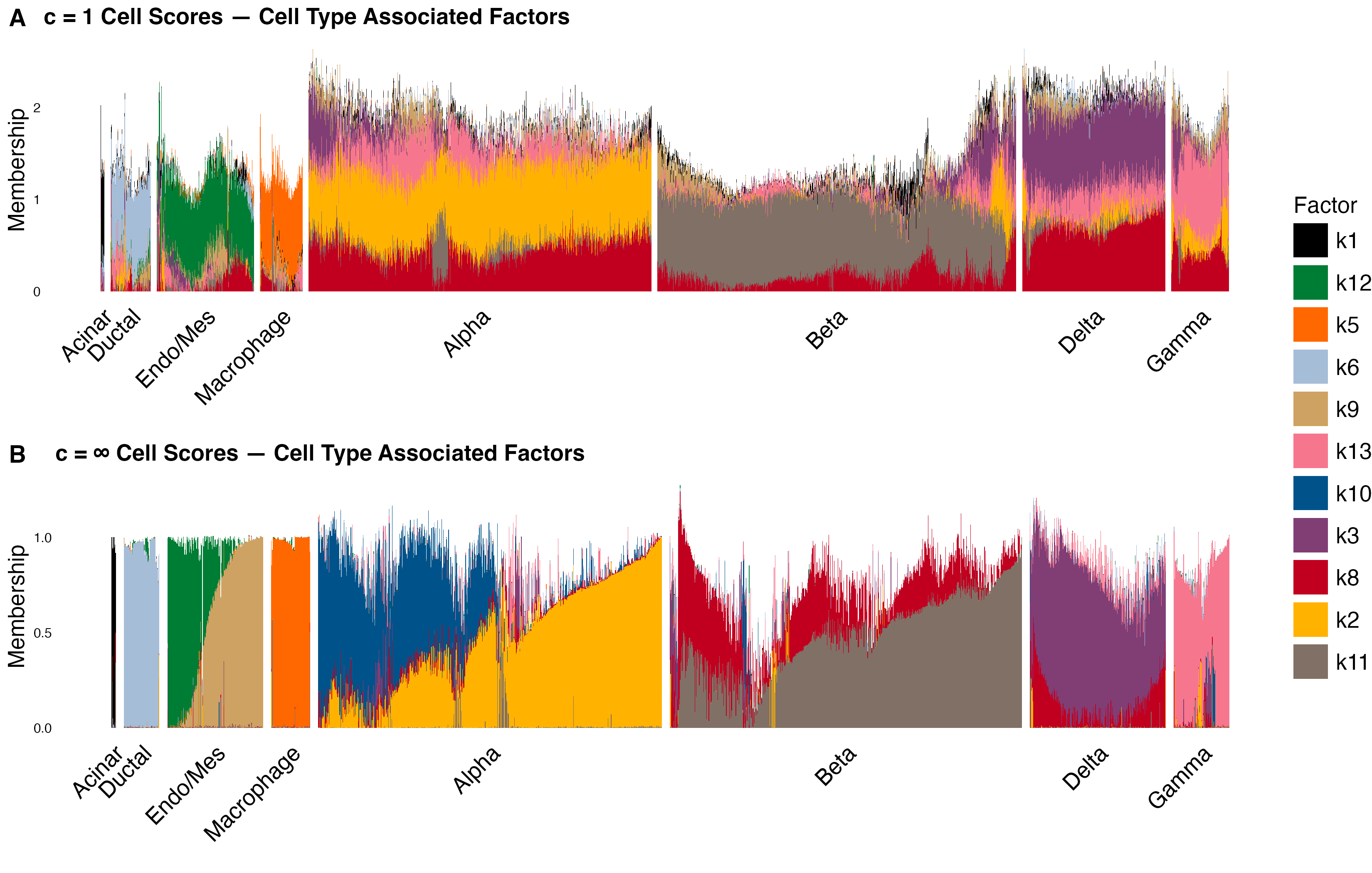}
    \end{minipage}
    
    \begin{minipage}{\linewidth}
\centering
{\scriptsize\renewcommand{\arraystretch}{1.0}%
\begin{tabularx}{\textwidth}{@{} l I I @{}}
\toprule
Factor & { \normalfont Top Genes -- $c = 1$} & {\normalfont Top Genes -- $c = \infty$}\\
\midrule
k1 & Ctrb1, Prss2, Try5, Reg1, Cela1, Cela3b, Cpa1, Try4, Reg3b, Ins2 \newline {\upshape\scriptsize GO: peptide hormone processing} & \cellcolor{gray!10}{Ctrb1, Prss2, Try5, Reg1, Cela1, Clu, Cela3b, Cpa1, Reg3b, Try4 \newline {\upshape\scriptsize GO: peptide hormone processing}}\\
k2 & \cellcolor{gray!10}{Gcg, Ttr, Spp1, Pyy, Gpx3, Rbp4, Tmem27, Gnas, Gc, Higd1a \newline {\upshape\scriptsize GO: hormone secretion}} & Gcg, Pyy, Ttr, Tpt1, Gnas, Chga, Tmem27, Eef1a1, Resp18, Spp1 \newline {\upshape\scriptsize GO: hormone secretion}\\
k3 & Sst, Pyy, Ppy, Rbp4, Arg1, Clu, Chgb, Ins1, Fam159b, Cd24a \newline {\upshape\scriptsize GO: pancreas development} & \cellcolor{gray!10}{Sst, Pyy, Iapp, Chgb, Rbp4, Tpt1, Gnas, Resp18, Scg2, Meg3 \newline {\upshape\scriptsize GO: peptide hormone processing}}\\
k5 & \cellcolor{gray!10}{Fth1, Apoe, Tmsb4x, Ftl1, Lgals3, Ctsb, Actb, Prdx1, C1qb, Fcer1g \newline {\upshape\scriptsize GO: antigen processing and presentation of exogenous peptide antigen via MHC class II}} & Fth1, Ftl1, Tmsb4x, Apoe, Lgals3, Actb, Ctsb, Mt1, Prdx1, Psap \newline {\upshape\scriptsize GO: lymphocyte mediated immunity}\\
k6 & Clu, Spp1, Lcn2, Tmsb4x, Krt8, Krt18, Eef1a1, Cxcl5, Tpt1, Epcam \newline {\upshape\scriptsize GO: response to tumor necrosis factor} & \cellcolor{gray!10}{Spp1, Clu, Lcn2, Eef1a1, Tpt1, Tmsb4x, Krt18, Krt8, Actg1, Actb \newline {\upshape\scriptsize GO: cell killing}}\\
k8 & \cellcolor{gray!10}{Iapp, Tpt1, Scg2, Ubb, Cd63, Resp18, Ssr4, Eef1a1, Scg5, Chgb \newline {\upshape\scriptsize GO: platelet aggregation}} & Iapp, Chga, Tpt1, Chgb, Eef1a1, Scg2, Resp18, Cd63, Cpe, Pcsk2 \newline {\upshape\scriptsize GO: peptide hormone processing}\\
k9 & Ghrl, Ppy, Hspa1a, Mt1, Gcg, Mgp, Hspa1b, Mt2, Mif, Bnip3 \newline {\upshape\scriptsize GO: canonical glycolysis} & \cellcolor{gray!10}{Mgp, Tpt1, Igfbp7, Eef1a1, Igfbp5, Ifitm3, Cxcl10, Actb, Tmsb4x, Ptma \newline {\upshape\scriptsize GO: response to interferon-beta}}\\
k11 & \cellcolor{gray!10}{Ins2, Ins1, Iapp, Chga, Ftl1, Tpt1, Ubb, Eef1a1, Scg2, Resp18 \newline {\upshape\scriptsize GO: digestion}} & Ins2, Ins1, Iapp, Chga, Ftl1, Ubb, Resp18, Scg2, Pcsk1n, Tpt1 \newline {\upshape\scriptsize GO: digestion}\\
k12 & Mgp, Igfbp7, Igfbp5, Tmsb4x, Sparc, Vim, Ifitm3, Eef1a1, Tpt1, Ccl2 \newline {\upshape\scriptsize GO: cell chemotaxis} & \cellcolor{gray!10}{Tpt1, Tmsb4x, Eef1a1, Cxcl10, Fth1, Ccl2, Vim, Actb, Cxcl1, Spp1 \newline {\upshape\scriptsize GO: neutrophil chemotaxis}}\\
k13 & \cellcolor{gray!10}{Ppy, Pyy, Ins2, Spp1, Chga, Tspan8, Resp18, Fth1, Pcsk1n, Clu \newline {\upshape\scriptsize GO: peptide hormone processing}} & Ppy, Pyy, Tpt1, Resp18, Eef1a1, Ubb, Cst3, Chgb, Scg2, Spp1 \newline {\upshape\scriptsize GO: peptide hormone processing}\\
\bottomrule
\end{tabularx}
}
    \end{minipage}
    
    \caption{Combined figure and table for the cell type associated factors of the Pancreas analysis. (A-B) Visual representation of estimated $\mathbf{L}$ matrices for the log1p NMF with $c = 1$ and $c = \infty$, grouped by cell type. Each column represents a row of $\mathbf{L}$, where each color corresponds to a column of $\mathbf{L}$. (Table) Top genes for the factors of each model, excluding treatment associated factors. GO terms are computed as described in the caption of Figure \ref{fig:mcf7}.}
    \label{fig:pancreas_structure_c}
\end{figure}

\begin{figure}
    \centering
    \includegraphics[width=1.0\linewidth]{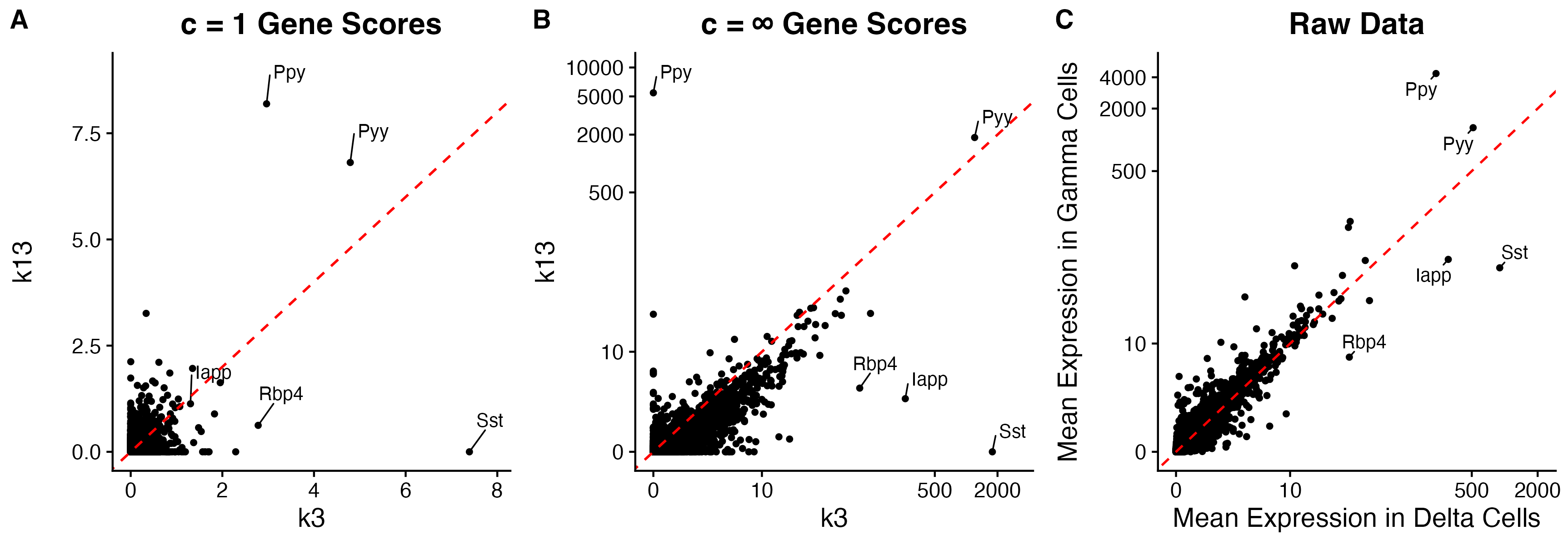}
    \caption{Comparison between delta and gamma cells in single cell pancreas data. (A-B) Gene scores ($\mathbf{F}$) from $c = 1$ and $c = \infty$ models corresponding to the factor most associated with delta cells (k3) and the factor most associated with gamma cells (k13). Before plotting, each matrix $\mathbf{L}$ was normalized so that the maximum value of each column was $1$, and $\mathbf{F}$ was scaled accordingly. (C) Mean expression of genes in delta and gamma cells.}
    \label{fig:lsa_dg}
\end{figure}

\setlength{\tabcolsep}{2.5pt}
\begin{figure}[htbp]
    \centering
    \begin{minipage}{\linewidth}
      \centering
      \includegraphics[width=0.74\linewidth]{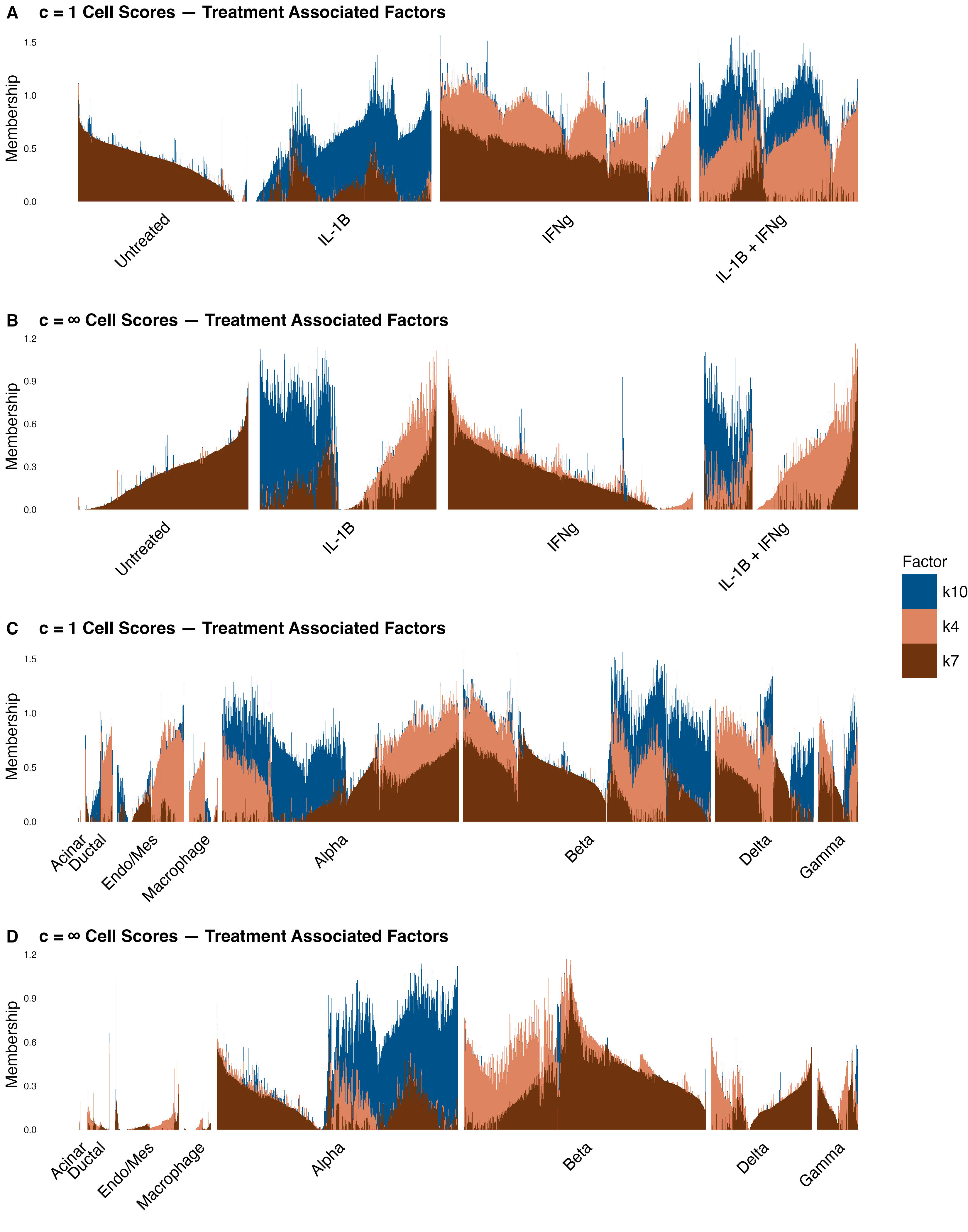}
    \end{minipage}
    
    \begin{minipage}{\linewidth}
\centering
{\footnotesize\renewcommand{\arraystretch}{1.02}%
\begin{tabularx}{\textwidth}{@{} l I I @{}}
\toprule
Factor & { \normalfont Top Genes -- $c = 1$} & {\normalfont Top Genes -- $c = \infty$}\\
\midrule
k4 & Cxcl10, Gbp2, Iigp1, Gbp4, Igtp, Gbp7, Gbp3, Irgm1, Stat1, Cd274 \newline {\upshape\scriptsize GO: cellular response to type II interferon} & \cellcolor{gray!10}{Ins2, Mt1, Tpt1, Eef1a1, Mt2, Cxcl10, Gbp2, Actg1, Hspa8, Fth1 \newline {\upshape\scriptsize GO: cell killing}}\\
k7 & \cellcolor{gray!10}{Kcnq1ot1, Acly, Atp2a2, Peg3, Eef1a1, Zbtb20, Ccnd2, Gm42418, Gatsl2, Zdhhc2 \newline {\upshape\scriptsize GO: insulin secretion}} & Tpt1, Eef1a1, Ftl1, Ins2, Chga, Calr, Gm42418, Fau, Cd63, Eif1 \newline {\upshape\scriptsize GO: striated muscle contraction}\\
k10 & Defb1, Lcn2, Mt1, Mt2, Cebpd, Sod2, Cxcl1, Steap4, Ccl2, Pabpc1 \newline {\upshape\scriptsize GO: inorganic ion homeostasis} & \cellcolor{gray!10}{Gcg, Spp1, Tpt1, Ttr, Pyy, Eef1a1, Gnas, Resp18, Hamp, Pcsk2 \newline {\upshape\scriptsize GO: peptide hormone processing}}\\
\bottomrule
\end{tabularx}
}
    \end{minipage}
    
    \caption{Visualization of the treatment associated factors of the Pancreas analysis. (A-B) Cell scores grouped by treatment. (C-D) Cell scores grouped by cell type. (Table) Top genes for the factors of each model, excluding cell type associated factors. GO terms are computed as described in the caption of Figure \ref{fig:mcf7}.}
    \label{fig:pancreas_structure_t}
\end{figure}

\subsubsection{BBC News data}
\label{sec:bbc}

Finally, we analyze a text dataset of news articles collected from the BBC between $2004$ and $2005$ \citep{greene2005producing}. Each article is labelled based on its editorial category in the BBC publication as one of the following: ``business'', ``entertainment'', ``politics'', ``sports'', or ``tech''. After removing SMART stop words \citep{smart}, stripping punctuation, converting words to lower case, removing numbers, stemming the words using Porter's algorithm \citep{pstem}, and removing any word that was not used in at least $5$ articles, we created a document-term matrix of size $2{,}127 \times 5{,}861$. Approximately $98\%$ of entries in this matrix were $0$. We fit all models to these data with $K = 10$. 

For these data, results for $c=1$ looked somewhat similar to $c=\infty$, so to better highlight differences between small and large $c$ we focus on comparing $c=0.001$ vs. $c=\infty$ (Figure \ref{fig:bbc_structure}). Results for a range of values of $c$ are shown in Appendix \ref{app:add_figs} (Figure \ref{fig:supp_bbc}). 

Examining the document scores ($\mathbf{L}$), $c=\infty$ again produces a relatively ``clustered'' representation, with each document having strong membership in just one or two components, whereas $c=0.001$ produces a much more modular representation, with many documents having membership in multiple components. Comparing the top words (``keywords'') in each component, the two fits capture some similar topics with highly overlapping keywords (Figure \ref{fig:bbc_structure}, Table). For example, both fits identify a politics topic (k1, whose shared keywords include ``labour'', ``tori(es)'', ``blair'', ``vote'', and ``tax''); a business topic (k6, with shared keywords ``market'', ``price'', ``growth'', etc); and an entertainment topic (k2, with shared keywords ``film'', ``actor'', ``oscar'', ``nomin(ate)'', ``actress''). However, there are some differences in keywords that reflect the additive vs. multiplicative behavior  of the two models: for example, in the entertainment topic (k2) the keywords for $c=0.001$ results place greater emphasis on rarer proper nouns (e.g., ``dicaprio'', ``foxx''). These proper nouns are not close to the most used words in entertainment articles (``dicaprio'' and ``foxx'' are the $317^\textrm{th}$ and $162^\textrm{nd}$ most used words among all entertainment articles, respectively), but they are \textit{relatively} much more frequent in entertainment documents than in other documents (indeed, neither ``dicaprio'' nor ``foxx'' occur in documents in other categories).

The difference between a more clustered as opposed to a more modular/layered representation is nicely illustrated by the different ways the two models represent heterogeneity among sports articles. The $c=\infty$ result essentially divides these articles into three distinct clusters (k5, k7, k10), which from the keywords seem to correspond roughly to rugby, football (soccer), and other sports. In contrast the $c=0.001$ results yield one factor (k5) that almost every sports article is loaded on, and heterogeneity among sports articles is captured by loadings on a variety of additional components (e.g., k7 captures the soccer-related articles). Some of these additional components seem to be doing ``double-duty'': for example, the keywords of k8 include ``athlet(e)'' and ``olymp(ic)'', but also ``music'', ``hip'', ``hop'' and ``soul''. The existence of such double-duty components may be a side-effect of the quasi-multiplicative nature of the model, where the effect of a component can depend very much on what it is being added to: intuitively, adding component k8 to other components that already contain music-related keywords will most increase the fitted values of those music-related words, whereas adding it to other components that contain sport-related words will most increase the fitted values of those sport-related words (to state the obvious, multiplying big numbers by something creates a bigger absolute change than multiplying small numbers by the same thing). These double-duty components are also somewhat harder to interpret (k10 for c=0.001 is particularly challenging); it is possible that increasing the number of factors, and/or introducing methods that encourage sparsity of the word scores could help here.

While the more modular/layered representation of $c=0.001$ may sometimes be harder to interpret, it can also sometimes highlight subtler structure that is missing in the clustered representation of $c=\infty$. For example, the soccer factor (k7) appears only in sports articles for $c=\infty$, but appears across all 5 document classes in the $c = 0.001$ fit. Manually examining the documents with high scores on k7 (with $c=0.001$) reveals that they typically mention words related to sports, and sometimes specifically British soccer, once or twice in the article. For example, the politics article with the highest score on k7 is about a controversial comment made by then London Mayor Ken Livingstone, and mentions soccer in reference to a similarly-controversial comment made by Boris Johnson regarding Liverpool fans \citep{bbcNEWSPolitics}.


\setlength{\tabcolsep}{2.5pt}
\begin{figure}[htbp]
    \centering
    \begin{minipage}{\linewidth}
      \centering
      \includegraphics[width=1.0\linewidth]{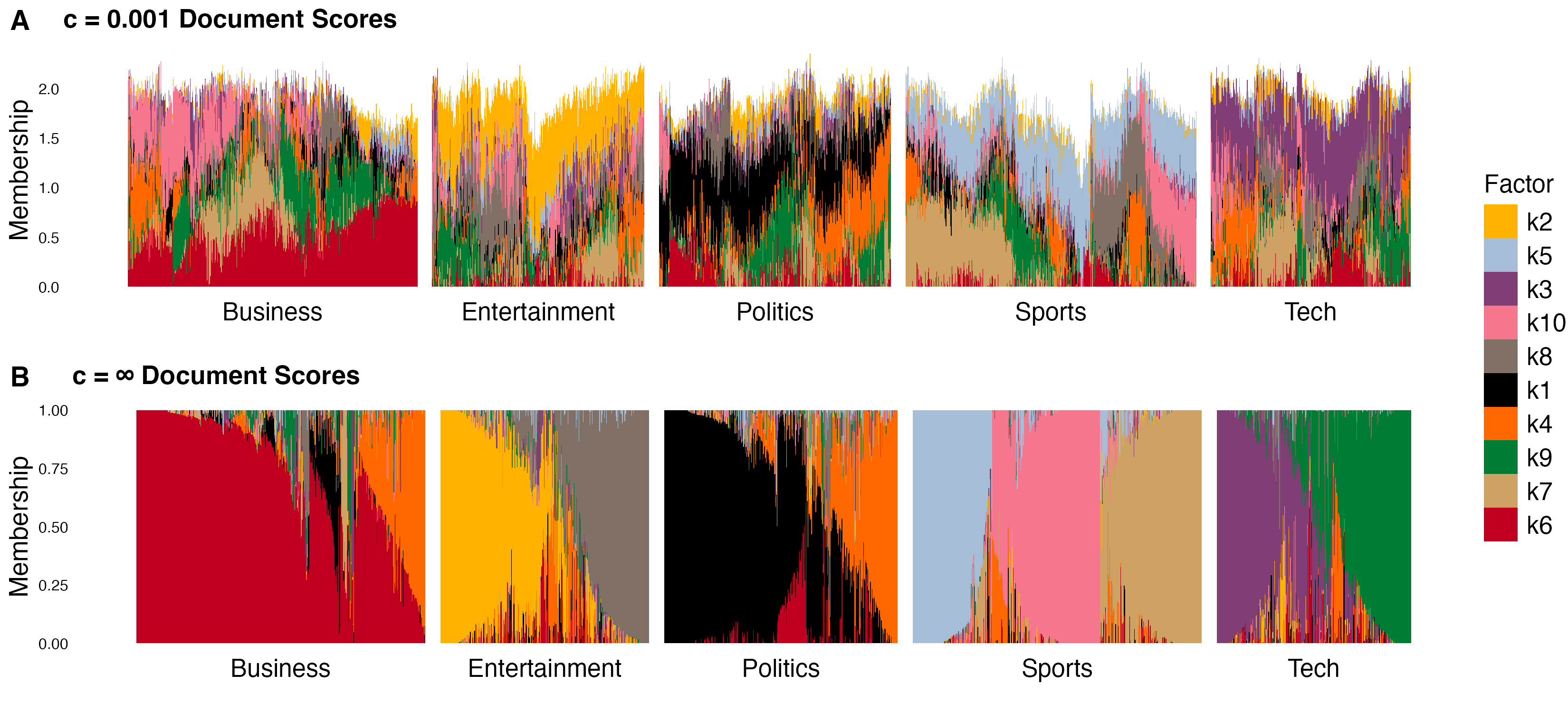}
    \end{minipage}
    
    \begin{minipage}{\linewidth}
\centering
{\scriptsize
\begin{tabularx}{\textwidth}{l Y Y}
\toprule
Factor & Top Words - $c = 0.001$ & Top Words - $c = \infty$\\
\midrule
k1 & tori, labour, elect, parti, blair,
kennedi, lib, conserv, dem, howard,
liber, democrat, voter, prime,
chancellor, tax, brown, vote, toni,
immigr & \cellcolor{gray!10}{labour, elect, parti, govern, peopl,
blair, minist, tori, plan, brown, tax,
howard, public, leader, prime, campaign,
work, polit, vote, year}\\
k2 & \cellcolor{gray!10}{film, actor, oscar, nomin, actress,
award, comedi, star, foxx, nomine,
aviat, drama, sideway, scorses,
staunton, dicaprio, bafta, hollywood,
ceremoni, drake} & film, award, star, year, show, actor,
director, oscar, nomin, includ, won,
movi, actress, role, win, prize, comedi,
festiv, bbc, british\\
k3 & user, technolog, broadband, digit,
mobil, devic, phone, web, content,
network, net, microsoft, comput,
download, program, search, internet,
video, pcs, gadget & \cellcolor{gray!10}{game, mobil, technolog, peopl, phone,
digit, video, year, player, make, time,
play, music, devic, develop, market,
gadget, high, work, comput}\\
k4 & \cellcolor{gray!10}{hunt, sentenc, crimin, guilti, prosecut,
murder, law, suspect, arrest, virus,
detain, lawyer, evid, trial, investig,
fraud, polic, offenc, ban, prosecutor} & law, court, govern, lord, case, rule,
polic, legal, told, claim, right, charg,
yuko, offic, compani, bill, trial, ban,
hunt, year\\
k5 & wale, win, ireland, england, victori,
injuri, match, play, franc, squad,
coach, itali, side, half, william,
minut, cup, player, game, score & \cellcolor{gray!10}{england, wale, ireland, game, nation,
rugbi, franc, play, half, side, win,
player, back, year, coach, scotland,
team, robinson, injuri, time}\\
\addlinespace
k6 & \cellcolor{gray!10}{growth, economi, price, rate, economist,
rise, profit, bank, inflat, quarter,
market, forecast, econom, rose, analyst,
manufactur, consum, export, dollar,
retail} & year, market, compani, bank, sale,
price, firm, share, growth, economi,
month, rate, expect, econom, countri,
rise, busi, report, profit, china\\
k7 & chelsea, arsenal, liverpool, wenger,
gerrard, parri, rover, club, mourinho,
ferguson, everton, striker, fiat,
consol, manchest, villa, footbal,
anfield, aston, morient & \cellcolor{gray!10}{club, game, play, player, unit, time,
chelsea, manag, footbal, leagu, goal,
team, liverpool, win, arsenal, back,
year, manchest, side, cup}\\
k8 & \cellcolor{gray!10}{athlet, olymp, women, urban, marathon,
album, radcliff, artist, indoor, athen,
hip, medal, hop, gadget, song, band,
music, drug, holm, soul} & music, year, song, band, record, album,
show, number, includ, top, singl, chart,
award, artist, perform, rock, singer,
releas, peopl, radio\\
k9 & iran, straw, palestinian, blog, isra,
aid, embargo, tsunami, turkey, peac,
egypt, china, foreign, iraqi, israel,
wto, cyprus, india, blogger, rugbi & \cellcolor{gray!10}{peopl, user, net, servic, network, site,
softwar, firm, internet, system, mail,
comput, search, secur, call, inform,
onlin, websit, broadband, virus}\\
k10 & \cellcolor{gray!10}{deutsch, yuko, boers, lse, sharehold,
takeov, euronext, russian, mci,
file, seed, gazprom, bid, ukip,
auction, houston, kilroy, chart, silk,
khodorkovski} & year, world, win, final, set, open,
olymp, play, time, champion, athlet,
match, race, titl, game, won, drug,
test, roddick, beat\\
\bottomrule
\end{tabularx}
}
    \end{minipage}
    
    \caption{Combined figure and table for the BBC analysis. (A-B) Visual representation of estimated $\mathbf{L}$ matrices for log1p NMF with $c = 0.001$ and $c = \infty$, grouped by document type. Each column represents a row of $\mathbf{L}$, where each color corresponds to a column of $\mathbf{L}$. (Table) Top words for each factor.}
    \label{fig:bbc_structure}
\end{figure}

\subsection{Quantitative assessment of log1p NMF fits across datasets}

In discussing the results above, particularly for the last two datasets, we noted that standard Poisson NMF ($c=\infty$) produces more ``clustered'' representations, where most samples (cells or documents) have appreciable scores on only one or two factors, and that the inferred factors tend to be highly correlated; in contrast, smaller values of $c$ produced more modular/layered representations, with each sample being represented as a combination of more factors (although still a modest number) that are less correlated.
In brief:  larger values of $c$ produced sparser loadings and more correlated factors. Interestingly, the first of these two observations is well supported by our simulations of Section \ref{sec:sim2} (Figure \ref{fig:sim2}A,B), whereas the second is less so (Figure \ref{fig:rmse_sim}B). 
 
 To quantify this trend more systematically, we fit log1p NMF to each dataset for a grid of values between $c = 10^{-3}$ and $c = 10^{3}$ (as well as $c = \infty$). For each fit, we measured the mean (absolute) Spearman correlation between columns of $\mathbf{F}$, and the mean column-wise sparsity (as defined in equation \eqref{eq:hoyer_def}) of both $\mathbf{L}$ and $\mathbf{F}$. The results (Figure \ref{fig:sparsity}) confirm that larger $c$ produced sparser loadings (and factors) and more correlated factors. Visualizations of the sample scores of these fitted models are provided in Appendix \ref{app:add_figs}.


\begin{figure}
    \centering
    \includegraphics[width=1.0\linewidth]{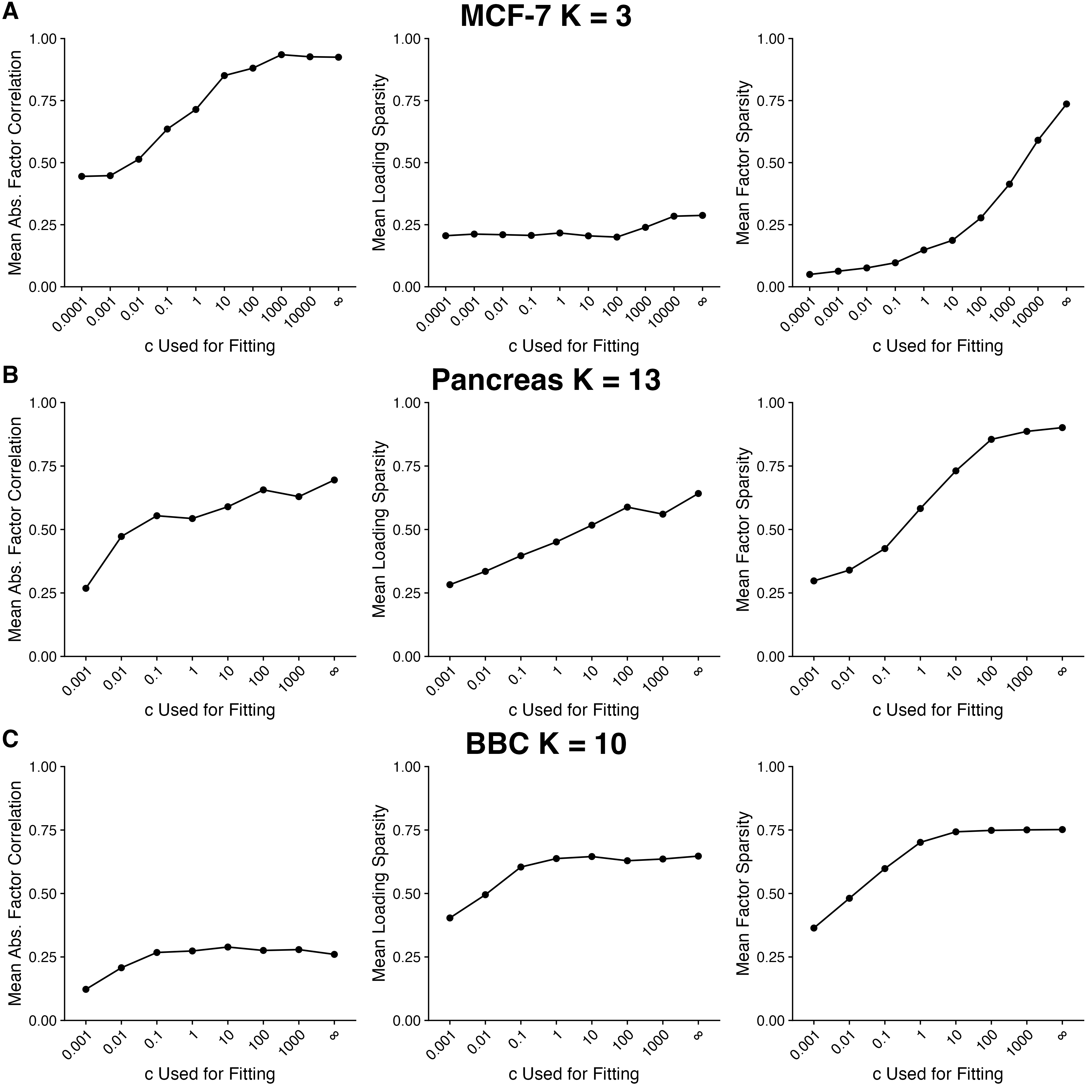}
    \caption{Summary of model fits across values of $c$ for real data applications. Spearman correlations are calculated pairwise between all factors (columns of $\mathbf{F}$) and averaged. Sparsity values are calculated for each loading / factor (columns of $\mathbf{L}$ or columns of $\mathbf{F}$) and then averaged. (A) Correlation and sparsity values for the MCF-7 dataset with $K = 3$. (B) Correlation and sparsity values for the pancreas dataset with $K = 13$. (C) Correlation and sparsity values for the BBC dataset with $K = 10$.}
    \label{fig:sparsity}
\end{figure}

\section{Discussion}

We have introduced log1p Poisson NMF, which is, to our knowledge, the first version of Poisson NMF with a non-identity link function. Our new approach has a parameter, $c$, which controls the behavior of the link function: large values of $c$ make it more linear, similar to standard Poisson NMF, and small values of $c$ make it more logarithmic (factors combine quasi-multiplicatively), and different from standard Poisson NMF.  In three examples we showed that different link functions (different values of $c$) can produce quite different results, potentially yielding different insights. Since Poisson NMF is widely used in practice, these new methods have the potential to provide new insights in a wide range of applications. 

Given the new ability to fit Poisson NMF with a range of link functions, it may seem natural to ask what the ``correct'' link function is for any given application. However, Poisson NMF may perhaps be best viewed as an exploratory data analysis tool, and this view encourages a different perspective: different link functions that provide different representations and different insights might complement one another rather than compete. For example, in our applications we saw that $c=\infty$ may produce a more clustered view of the data, whereas smaller $c$ produced a more modular/layered view. Neither view is ``correct'', but both may be useful in their own way (and so, in this sense, neither is ``incorrect''). In practice, we therefore suggest  beginning NMF analyses of count data by running log1p NMF with both $c=\infty$ and a smaller value of $c$ (e.g., $c=1$) to compare the results. Even in settings where one might argue for an additive link function on scientific grounds, comparing the results with $c=1$ could be an interesting exercise.

One trend we saw in our applications is that log1p NMF with small $c$ tends to produce less sparse solutions. Usually sparsity is considered helpful for interpretability of results, so producing less sparse solutions may seem a step backwards. However, in our applications this does not appear clear-cut: indeed, in these applications small $c$ arguably produces representations that are more ``parts-based'', in the sense of \cite{lee1999learning}, who highlighted this as a major benefit of NMF. Nonetheless, all things being equal, sparsity may indeed aid interpretation, and investigating sparse versions of log1p NMF could be an interesting area for future work. In particular, adapting empirical Bayes methods for Gaussian matrix factorization (e.g., \cite{wang2021empirical}) to the log1p NMF model could improve the practical utility of the model, and may also aid in determining appropriate values of $K$, another challenging practical problem when using NMF. It may also be interesting to extend other approaches to selecting $K$, such as the one in \cite{zito2024compressive},
to log1p NMF.

Finally, we also found empirically that small $c$ produced results with less-correlated factors. It is possible that this result is particular to the data sets we examined here, but it is also possible that this is a more general phenomenon.  We have performed some very preliminary investigations in this direction (see Appendix \ref{sec:app_geometry}), and can show that, for small $c$ ($c \rightarrow 0^+$)
the expressivity of a regression model with two factors depends on the upper right boundary of the convex hull when these two factors are plotted against one another. If the two factors are highly correlated then there will be few points on this hull, resulting in low expressivity. (Indeed, if the factors have the same maximal element then the upper right convex hull is a single point, and the 2-factor regression model collapses to a single factor).
Extending this intuition to $K$ factors suggests that small $c$ models may prefer less correlated factors simply because they are more expressive, resulting in a better fit to data. A rigorous mathematical study of this, and more generally of identifiability issues with log1p NMF, could also be an interesting area for further work.

\newpage
\acks{This research was supported by NIH grant HG002585 to MS and NIH grants R35GM131802 and R01HG005220 to RAI. We thank members of the Stephens and Irizarry labs for helpful comments and feedback.}


\newpage

\appendix

\counterwithin{figure}{section}
\renewcommand{\thefigure}{\thesection.\arabic{figure}}

\section{Derivations and proofs}

\subsection{Proof of Theorem \ref{theorem:log1p-c-inf}}
\label{app:thm1_pf}
\begin{proof}
        First, with $\alpha_c := \max(1, c)$, define the functions
        \begin{align*}
            \lambda_{\textrm{log1p}} (\mathbf{l}, \mathbf{f}, c) &= c \left( \exp\left\{\frac{1}{\alpha_c}\sum_{k = 1}^{K} l_{k}f_{k}\right\} - 1\right) \\
            \lambda_{\textrm{id}} (\mathbf{l}, \mathbf{f}) &= \sum_{k = 1}^{K} l_{k}f_{k},
        \end{align*}
where $\mathbf{l}, \mathbf{f}  \in \mathbb{R}^{K}_{\geq 0}$, which are the Poisson rate parameters of a single arbitrary element of $\mathbf{Y}$ for the log1p model and Poisson NMF model with identity link, respectively. Since the log-likelihood of data generated from a Poisson distribution depends on its parameters only through its rate, it suffices to show that for any $\mathbf{l}, \mathbf{f} \in \mathbb{R}^{K}_{\geq 0}$,
\begin{equation*}
    \lim_{c \rightarrow \infty} \lambda_{\textrm{log1p}} \left(\mathbf{l}, \mathbf{f}, c\right) = \lambda_{\textrm{id}} (\mathbf{l}, \mathbf{f}).
\end{equation*}
Now, let $b = \sum_{k = 1}^{K} l_{k} f_{k}$ and let $x = \frac{b}{c}$. If $b = 0$, then $\lambda_{\textrm{log1p}} \left(\mathbf{l}, \mathbf{f}, c\right) = \lambda_{\textrm{id}} (\mathbf{l}, \mathbf{f}) = 0$, in which case the statement holds for any $c$. If $b > 0$, we have

\begin{align*}
    \lim_{c \rightarrow \infty} \lambda_{\textrm{log1p}} \left(\mathbf{l}, \mathbf{f}, c\right) &= \lim_{c \rightarrow \infty} c \left( \exp\left\{\frac{1}{\alpha_c}\sum_{k = 1}^{K} l_{k}f_{k}\right\} - 1\right) \\
    &= b \lim_{x \rightarrow 0^{+}} \frac{e^{x} - 1}{x} \quad \textrm{(since $\alpha_c = c$ for $c \geq 1$)} \\
    &= b \lim_{x \rightarrow 0^{+}} e^{x} \quad \textrm{(L'Hospital's rule)} \\
    &= b 
\end{align*}

Substituting this back into the definition of
$\lambda_{\textrm{log1p}}(\cdot)$, we have

\begin{equation*}
    \lim_{c \rightarrow \infty} \lambda_{\textrm{log1p}} \left(\mathbf{l}, \mathbf{f}, c\right) = \sum_{k = 1}^{K} l_{k}f_{k} = \lambda_{\textrm{id}} (\mathbf{l}, \mathbf{f}).
\end{equation*}

\end{proof}

\subsection{Proof of Theorem \ref{theorem:log1p-c-zero}}
\label{app:thm2_pf}
\begin{proof}
Substituting the definition of $\varepsilon_{ij}$ in equation \eqref{eq:thm2_stmt2} into equation \eqref{eqn:lf_interpretation}, we can write
\begin{align*}
    \frac{1}{\alpha_c}l_{ik}f_{jk} + \varepsilon_{ij} &= \frac{1}{\alpha_c}l_{ik}f_{jk} + \log\left(1 + c/\lambda_{ij}^{(-k)}\right) - \log(1 + c/\lambda_{ij}) \\
    &= \frac{1}{\alpha_c}l_{ik}f_{jk} + \log \left(\frac{1 + c/\lambda_{ij}^{(-k)}}{1 + c/\lambda_{ij}} \right) \\
    &= \frac{1}{\alpha_c}l_{ik}f_{jk} + \log \left(\frac{\left(\frac{\lambda_{ij}^{(-k)} + c}{\lambda_{ij}^{(-k)}}\right)}{\left(\frac{\lambda_{ij} + c}{\lambda_{ij}}\right)}\right) \\
    &= \frac{1}{\alpha_c}l_{ik}f_{jk} + \log(\lambda_{ij}) - \log\left(\lambda_{ij}^{(-k)}\right) + \log\left(\lambda_{ij}^{(-k)} + c\right) - \log\left(\lambda_{ij} + c\right) \\
    &= \frac{1}{\alpha_c}l_{ik}f_{jk} + \log\left(\frac{\lambda_{ij}}{\lambda_{ij}^{(-k)}}\right) - \frac{1}{\alpha_c}l_{ik}f_{jk} \\
    &= \log\left(\frac{\lambda_{ij}}{\lambda_{ij}^{(-k)}}\right),
\end{align*}
where the second to last line follows by the definition of the link function \eqref{eq:shifted_log_link}.
\end{proof}

\subsection{Proof of bi-concavity of log1p NMF log-likelihood}
\label{app:thm3}
\begin{theorem} \label{theorem:biconcave}
With $c$ fixed, $\ell_{\textrm{log1p}}\left(\mathbf{L}, \mathbf{F}, c; \mathbf{Y}\right)$ is a bi-concave function of $\mathbf{L}$ and $\mathbf{F}$ on the domain $\Omega = \left\{\mathbf{L} \in \mathbb{R}^{n \times K}_{\geq 0}, \mathbf{F} \in \mathbb{R}^{p \times K}_{\geq 0} \mid \sum_{k = 1}^{K} l_{ik}f_{jk} > 0 \: \forall \: (i,j) \in \{1, \dots, n\} \times \{1, \dots, p\}\right\}$. That is, with $c$ and $\mathbf{L}$ fixed, $\ell_{\textrm{log1p}}\left(\mathbf{L}, \mathbf{F}, c; \mathbf{Y}\right)$ is a concave function of $\mathbf{F}$, and with $c$ and $\mathbf{F}$ fixed, $\ell_{\textrm{log1p}}\left(\mathbf{L}, \mathbf{F}, c; \mathbf{Y}\right)$ is a concave function of $\mathbf{L}$.
\end{theorem}
\begin{proof}
    First, for $b > 0$ and $y \geq 0$, define $$h_y(b) = y \log(e^{b/\alpha_c} - 1) - ce^{b/\alpha_c},$$ where $\alpha_c := \max(1, c)$ for $c > 0$. Then, we have that $$\frac{\partial^{2} h_y}{\partial b^{2}} = -\frac{e^{b/\alpha_c}}{\alpha_c^2}\,\left(\frac{y}{(e^{b/\alpha_c}-1)^2}+c\right).$$ Since $y, c, \alpha_c \geq 0$, for all $b > 0$ we have that $\frac{\partial^{2} h_y}{\partial b^{2}}(b) \leq 0$. Thus, $h_y(b)$ is a concave function of $b$ on the domain $(0, \infty)$ for any $y \geq 0$. Now, note that (up to a constant with respect to $\mathbf{L}$ and $\mathbf{F}$), we can write 
    \begin{equation}
        \ell_{\textrm{log1p}}\left(
\mathbf{L}, \mathbf{F}, c; \mathbf{Y}\right) = \sum_{i = 1}^{n} \sum_{j = 1}^{p} h_{y_{ij}}(b_{ij}) = \sum_{i = 1}^{n} \sum_{j = 1}^{p} h_{y_{ij}}\left(\sum_{k = 1}^{K} l_{ik} f_{jk}\right). \label{eq:ll_concave}
    \end{equation}
    Now, with $\mathbf{F}$ fixed and $\mathbf{L}, \mathbf{F} \in \Omega$, $h_{y_{ij}}\left(\sum_{k = 1}^{K} l_{ik} f_{jk}\right)$ is the composition of an affine function (in $\mathbf{L}$) and a concave function, and is thus itself concave in $\mathbf{L}$ \cite[Section~3.2.2]{boyd2004convex}. Then, since equation \eqref{eq:ll_concave} is a sum of concave functions of $\mathbf{L}$, the sum itself must be concave in $\mathbf{L}$ \cite[Section~3.2.1]{boyd2004convex}. This shows that with $\mathbf{F}$ fixed and $\mathbf{L}, \mathbf{F} \in \Omega$, $\ell_{\textrm{log1p}}\left(
\mathbf{L}, \mathbf{F}, c; \mathbf{Y}\right)$ is a concave function of $\mathbf{L}$.  An analogous argument shows that $\ell_{\textrm{log1p}}\left(
\mathbf{L}, \mathbf{F}, c; \mathbf{Y}\right)$ is a concave function of $\mathbf{F}$ with $\mathbf{L}$ fixed.
\end{proof}

\section{Computational complexity of approximate log-likelihood}
\label{app:comp_complexity}
In the main text, we suggest an approximation to the log-likelihood of the log1p model:
\begin{align*}
    \ell_{\textrm{log1p}}(\mathbf{L}, \mathbf{F}, c; \mathbf{Y}) \approx & \sum_{(i,j) \notin \mathcal{I}_{0}} y_{ij} \log\left(\exp\left\{ \frac{1}{\alpha_c} \sum_{k = 1}^{K} l_{ik} f_{jk} \right\} - 1 \right) - c\sum_{(i,j) \notin \mathcal{I}_{0}} \exp\left(\frac{1}{\alpha_c}\sum_{k = 1}^{K} l_{ik} f_{jk}\right) \notag \\
& - \frac{\eta_{1}c}{\alpha_c} \sum_{(i,j) \in \mathcal{I}_{0}} \sum_{k = 1}^{K} l_{ik} f_{jk} - \frac{\eta_{2}c}{\alpha_c^{2}}\sum_{(i,j) \in \mathcal{I}_{0}} \left(\sum_{k = 1}^{K} l_{ik} f_{jk}\right)^{2}.
\end{align*}
The first two terms in the above equation clearly require $\mathcal{O}(\omega K)$ operations, where $\omega$ is the number of non-zero entries of $\mathbf{Y}$. Naively, it appears that the final two terms would require $\mathcal{O}(\lvert \mathcal{I}_0 \rvert K)$ operations, as they require computing $\lvert \mathcal{I}_0 \rvert$ terms, each of which requires $K$ operations. However, observe the identities
\begin{align}
        \sum_{i=1}^{n}\sum_{j=1}^{p} \sum_{k = 1}^{K} l_{ik}f_{jk} &= \sum_{k = 1}^{K} \left( \sum_{i = 1}^{n} l_{ik} \right) \left( \sum_{j = 1}^{p} f_{jk} \right), \label{eq:ll_id_lin} \\
  \sum_{i=1}^{n}\sum_{j=1}^{p} \left(\sum_{k = 1}^{K} l_{ik}f_{jk}\right)^{2} &= \textrm{tr}\left(\mathbf{F} \mathbf{L}^{\top} \mathbf{L} \mathbf{F}^{\top} \right). \label{eq:ll_id_quad}
\end{align}
Thus, equation \eqref{eq:ll_id_lin} can be computed in $\mathcal{O}((n + p)K)$ operations and \eqref{eq:ll_id_quad} can be computed in $\mathcal{O}((n + p)K^{2})$ operations. 

Using the above identities, we can re-write the approximate log-likelihood as 
\begin{align}
    \ell_{\textrm{log1p}}(\mathbf{L}, \mathbf{F}, c) \approx & \sum_{(i,j) \notin \mathcal{I}_{0}} \Bigg[ y_{ij} \log\left(\exp\left\{ \frac{1}{\alpha_c}\sum_{k = 1}^{K} l_{ik} f_{jk} \right\} - 1 \right) \label{eq:approx_log1p_ll_clean1} \\
     & - c \left(\exp\left\{\frac{1}{\alpha_c}\sum_{k = 1}^{K} l_{ik} f_{jk}\right\} - \frac{\eta_{1}}{\alpha_c}\sum_{k = 1}^{K} l_{ik} f_{jk} -  \frac{\eta_{2}}{\alpha_c^{2}}\left\{\sum_{k = 1}^{K} l_{ik} f_{jk}\right\}^{2} \right) \Bigg] \label{eq:approx_log1p_ll_clean2}  \\
 & - c\left[\frac{\eta_{1}}{\alpha_c} \sum_{k = 1}^{K} \left( \sum_{i = 1}^{n} l_{ik} \right) \left( \sum_{j = 1}^{p} f_{jk} \right) + \frac{\eta_{2}}{\alpha_c^{2}}\cdot \textrm{tr}\left(\mathbf{F} \mathbf{L}^{\top} \mathbf{L} \mathbf{F}^{\top} \right)\right]. \notag
 \end{align}
 The first two lines \eqref{eq:approx_log1p_ll_clean1}-\eqref{eq:approx_log1p_ll_clean2} above can still be computed in $\mathcal{O}(\omega K)$ operations because each term involves the sum of the same $K$ values $\left(\sum_{k=1}^{K}l_{ik}f_{jk}\right)$, and the final line has complexity as described above. This brings the total computational complexity of the approximate log-likelihood to 
 \begin{equation*}
\mathcal{O}\left((\omega + n + p)K + (n + p)K^{2}\right),
\end{equation*}
as described in the main text.

\section{Fitting non-negative Poisson GLMs with a log1p link using cyclic coordinate ascent}
\label{app:ccd}
Algorithm \ref{alg:cap} involves repeatedly fitting the model
\begin{align*}
y_{i} &\overset{\textrm{indep.}}{\sim} \textrm{Poisson}(\lambda_{i})  \\
\alpha_c \log\left(1 + \frac{\lambda_{i}}c\right) &= \mathbf{x}_{i}^{\top} \boldsymbol{\beta},
\end{align*}
where $\mathbf{y} \in \mathbb{N}^{N}_{0}$ is a vector of counts, $\mathbf{X} \in \mathbb{R}^{N \times q}_{\geq 0}$ is a fixed matrix of non-negative ``covariates'', $\boldsymbol{\beta} \in  \mathbb{R}^{q}_{\geq 0}$ is an unknown vector of non-negative regression coefficients, $c \in \mathbb{R}_{> 0}$ is a fixed constant, and $\alpha_c := \max(1, c)$. This model's log-likelihood, $\ell_{\mathrm{log1pReg}}(\boldsymbol{\beta}, c; \boldsymbol{y}, \mathbf{X})$, is written in equation \eqref{eq:log1p_nn_reg_ll}.
Fitting the non-negative Poisson GLM via maximum likelihood thus reduces to solving the problem
\begin{align}
\hat{\boldsymbol{\beta}}
&=
\mathrm{argmax}_{\boldsymbol{\beta}}
\;\;
\ell_{\mathrm{log1pReg}}(\boldsymbol{\beta}, c; \boldsymbol{y}, \mathbf{X}) \label{eq:ccd1}
\\
\text{subject to} \quad & \beta_j \ge 0,\;\; j = 1,\dots,q . \label{eq:ccd2}
\end{align}
We solve \eqref{eq:ccd1}-\eqref{eq:ccd2} using cyclic coordinate ascent (CCA) \citep{bertsekas-1999, wright-2015} due to its simplicity and good performance in similar Poisson matrix factorization problems \citep{carbonetto2021non,weine2024fast}. The CCA algorithm performs the following 1-d optimization for each $j
= 1, \ldots, q$,
\begin{equation}
\beta_j^{\mathrm{new}} \leftarrow 
\mathrm{argmax}_{\beta_j \geq 0} \;
\ell_{\mathrm{log1pReg}}(\boldsymbol{\beta}, c; \boldsymbol{y}, \mathbf{X}),
\label{eq:ccd3}
\end{equation}
and repeats these 1-d optimizations for some fixed number of cycles or until the iterates reach a stationary point (by default our implementation uses a maximum of $3$ cycles). 

To solve each 1-d optimization problem of the form \eqref{eq:ccd3}, we use a simple projected Newton's method with a line search \citep{bertsekas1982projected}. This algorithm is very efficient, and the projection step can be performed in constant time because we are solving a 1-d problem.

We note that when maximizing our approximate log-likelihood of equation \eqref{eq:approx_log1p_ll}, we also solve an approximate version of problem \eqref{eq:ccd1}-\eqref{eq:ccd2}. Specifically, using the approximate log-likelihood of equation \eqref{eq:approx_log1p_ll} and performing a decomposition analogous to equation \eqref{eq:fit_L} leads to a GLM with approximate log-likelihood
\begin{align*}
    \ell_{\textrm{log1pReg}}(\boldsymbol{\beta}, c; \boldsymbol{y}, \mathbf{X}) &\approx \sum_{i \notin \mathcal{I}_{0}} \Bigg[ y_{i} \log\left\{\exp\left\{\mathbf{x}_{i}^{\top} \boldsymbol{\beta} / \alpha_c \right\} - 1 \right\} \\
    & - c \left(\exp\left\{ \mathbf{x}_{i}^{\top} \boldsymbol{\beta} /\alpha_c\right\}  - \frac{\eta_1}{\alpha_c} \mathbf{x}_{i}^{\top} \boldsymbol{\beta} - \frac{\eta_2}{\alpha_{c}^{2}} \left\{ \mathbf{x}_{i}^{\top} \boldsymbol{\beta}\right\}^{2} \right)\Bigg] \\
    & - c \left(\frac{1}{\alpha_c}\boldsymbol{\beta}^{\top} \mathbf{X}^{\top}\boldsymbol{\eta_1} + \frac{\eta_2}{\alpha_{c}^{2}} \boldsymbol{\beta}^{\top} \mathbf{X}^{\top}\mathbf{X}\boldsymbol{\beta}\right),
\end{align*}
where $\mathcal{I}_0$ is the index set corresponding to the $0$ counts of $\mathbf{y}$, $\eta_1$ and $\eta_2$ are the coefficients used in the approximation of $\exp(z)$, and $\boldsymbol{\eta_1}$ is an $N$-vector with $\eta_1$ in each entry. We note that over multiple iterations of CCA (and indeed, over the updates over the rows of $\mathbf{L}$ and $\mathbf{F}$), $\mathbf{X}^{\top}\boldsymbol{\eta_1}$ and $\mathbf{X}^{\top}\mathbf{X}$ only need to be computed \textit{once}. That is, for each outer loop of Algorithm \ref{alg:cap}, before updating $\mathbf{L}$ we can pre-compute $\mathbf{F}^{\top}\boldsymbol{\eta_1}$ and $\mathbf{F}^{\top}\mathbf{F}$, and before updating $\mathbf{F}$ we can pre-compute $\mathbf{L}^{\top}\boldsymbol{\eta_1}$ and $\mathbf{L}^{\top}\mathbf{L}$. 

\section{Additional details on visualization of log1pNMF fits}\label{app_sec:rescaling}

In addition to the invariance to re-scaling discussed in Section \ref{sec:rescaling}, the log1p NMF model is also invariant to re-ordering of the columns of $\mathbf{L}$ and $\mathbf{F}$. That is, if we swap columns $k$ and $m$ of both $\mathbf{L}$ and $\mathbf{F}$, the fitted Poisson rates $\lambda_{ij}$ are unchanged, and therefore so is the log-likelihood of the model. This invariance can hinder comparison of log1p NMF fits with different values of $c$, as in the Results we assign each factor a distinct color when visualizing its scores (where we use the same color to column mapping across all fits). 
Whenever we display multiple fits of the same dataset together, we therefore match their factors before plotting. One fit is designated as the reference (in our analyses, the $c = 1$ fits for the MCF-7 and pancreas data, and the $c = 0.001$ fit for the BBC data), and the factors of every other fit are placed in one-to-one correspondence with the reference factors by maximizing the sum of Spearman correlations between the paired columns of $\mathbf{F}$. The optimal matching between the factors can be computed efficiently with the Hungarian algorithm \citep{kuhn1955hungarian}. 

Finally, log1p NMF fits are also invariant to re-ordering of the rows of $\mathbf{L}$ (so long as the same re-ordering is done to the rows of $\mathbf{Y}$), but this ordering can substantially impact visualizations. For each dataset, we therefore first re-order the rows of $\mathbf{L}$ according to a categorization given to the corresponding row of $\mathbf{Y}$ by the original authors (for genomics applications, this is typically cell type or treatment, for text applications this is document class). Then, within each category, we re-order the rows of $\mathbf{L}$ using the function \texttt{structure\_plot} in the \texttt{R} package \texttt{fastTopics} \citep{carbonetto2021non}, which computes a one-dimensional t-SNE embedding \citep{van2008visualizing} of the rows of $\mathbf{L}$ and then sorts the rows within each category by their position along this embedding. Generally, this results in samples with similar values of $\mathbf{l}_{i \cdot }$ being placed next to one another. Note that because this ordering is recomputed separately for each fit, when comparing fits across values of $c$ the precise ordering of the rows of $\mathbf{L}$ may not be the same.

\section{Reproducibility of results}
An \texttt{R} package with associated code to fit log1p NMF on our real data examples and reproduce our figures can be found at \url{https://github.com/eweine/log1pNMF}.

\section{Benchmarking the approximate log-likelihood approach on real data}\label{app_sec:approx}

In Section \ref{sec:accuracy_approx} we assessed the accuracy of the approximate log-likelihood in simulations. Here, we compare the exact and approximate approaches on real data in terms of the quality of fit achieved per unit of wall-clock time. Specifically, we fit the log1p NMF model with $c = 1$ to three datasets: the pancreas dataset of Section \ref{sec:pancreas} ($K = 13$), the BBC dataset of Section \ref{sec:bbc} ($K = 10$), and --- to provide an example at the scale of modern single-cell atlases --- the mouse organogenesis cell atlas (MOCA) of \cite{cao2019single}, a dataset of approximately $1.33$ million cells with $26{,}182$ genes measured ($K = 25$). The MOCA dataset is also extremely sparse, with approximately $2.6\%$ of entries greater than $0$.

For each dataset we fit the model twice: once optimizing the exact log-likelihood \eqref{eq:log1p_ll} and once optimizing the approximate log-likelihood \eqref{eq:approx_log1p_ll}, with the polynomial coefficients set by the Chebyshev approximation. Each fit was initialized with the rank-1 procedure of Section \ref{sec:init}, carried out under the fit's own objective. Each fit was then allowed up to $10$ hours of optimization time on $64$ threads, stopping earlier if the absolute change in its objective fell below $10^{-8}$. Throughout each fit we recorded the wall-clock time and the exact log-likelihood of the current iterate. Figure \ref{fig:real_data_approx_ll} shows, for each dataset, the gap between this exact log-likelihood and the best exact log-likelihood attained by either approach on that dataset.

The benefits of the approximate approach are most apparent for the MOCA dataset, which is by far the largest dataset under consideration (and is quite sparse). Interestingly, on the BBC dataset, the approximate approach  converges to a better exact log-likelihood than directly optimizing the exact log-likelihood. This presumably results from the exact approach getting stuck in a local optimum on this dataset, which is unsurprising given the non-convex nature of the optimization problem. Finally, on the pancreas dataset, the exact approach clearly beats the approximate approach. While the approximate approach appears to make faster initial progress (as it is able to complete far more iterations per unit time), eventually the exact approach exceeded the log-likelihood of the approximate approach and continued to improve. Given enough time (and, save for convergence to local optima), the exact approach will \textit{eventually} achieve a higher exact log-likelihood than the approximate approach. However, in practice, especially with extremely large datasets, fitting NMF models to convergence (and in particular log1p NMF models) will take an extremely long time, and thus running the approximate approach some fixed number of iterations will often lead to reasonable results in a much shorter period of time.

\begin{figure}
    \centering
    \includegraphics[width=1.0\linewidth]{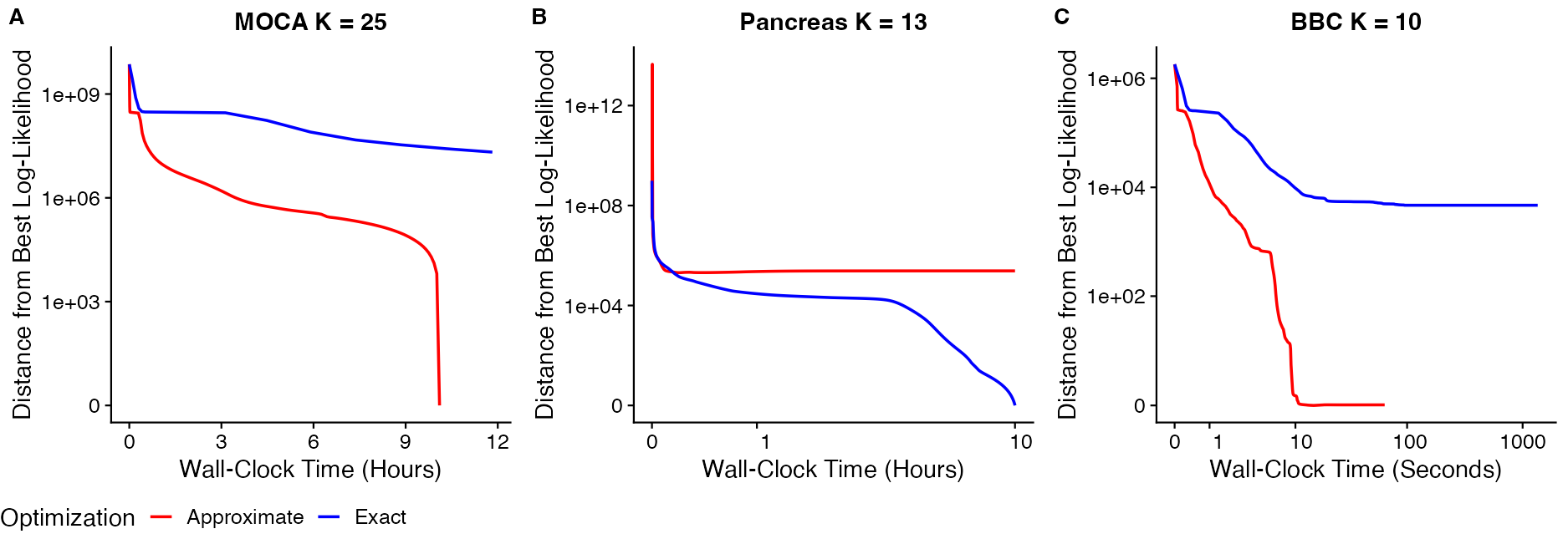}
    \caption{Progress of the exact and approximate (Chebyshev) optimization approaches on three real datasets, fitting log1p NMF with $c = 1$ from the rank-1 initialization. Each panel shows, as a function of wall-clock time (including initialization time), the distance between the exact log-likelihood of the current iterate and the best exact log-likelihood achieved by either approach on that dataset; a curve reaching $0$ attained the best solution found. Distances are plotted on a $\log(1 + x)$ scale, and the time axes of panels (B) and (C) are on the same scale, so that $0$ (the initialization) remains representable. (A) MOCA. (B) Pancreas. (C) BBC.}
    \label{fig:real_data_approx_ll}
\end{figure}

\FloatBarrier

\section{Comparing log1p NMF to Frobenius NMF on the log1p transformed counts}
\label{app:approx_comp}
In the main text, we introduced an approximate log-likelihood that is much faster to compute than the exact likelihood of the log1p NMF model and was shown to be reasonably accurate in simulations. Here, we compare the results of using this approximation as opposed to the exact log-likelihood for fitting the log1p model to the pancreas dataset with $c = 1$. In addition, we compare these approaches to fitting Frobenius NMF to the log1p transformed count data (i.e., equations \eqref{eq:frob_nmf}-\eqref{eq:log1p_trans} with $c = 1$). 

All models were fit from a common initialization, constructed by first fitting the log1p model with the exact log-likelihood and $K = 1$, and then running a single iteration of the $K = 13$ model (again with the exact log-likelihood) initialized from that rank-1 fit. This was done in order to encourage the factors of the different models to be comparable.

The cell scores for each of these three approaches are shown in Figure \ref{fig:supp_lsa_approx1} and Figure \ref{fig:supp_lsa_approx2}. Generally speaking, all three models discover relatively similar ``cell type associated'' structure. However, interestingly, only the two approximate methods appear to fit factors that model a spectrum between endothelial and mesenchymal cells. However, the cost of this factor appears to be a poor representation of the acinar cells, which do not appear to be represented very cleanly in either approximation approach.

The ``treatment associated'' cell scores are extremely similar between the model fit with the exact log-likelihood (Figure \ref{fig:supp_lsa_approx2}A) and the model fit with the quadratic approximation to the log-likelihood (Figure \ref{fig:supp_lsa_approx2}B). Specifically, both models have 1 factor each that represents the Untreated, IL-1$\beta$, and IFN$\gamma$ treatments, respectively, and use a combination of these factors to represent the combination of treatments. However, the results of fitting Frobenius NMF to the log1p transformed counts are not as parsimonious. While this model does appear to identify one factor each associated with IL-1$\beta$ and IFN$\gamma$ treatment (k10 and k4, respectively), the combination group is represented as a combination of k4 and an entirely separate factor k13. We additionally fit Frobenius NMF to the log1p transformed counts with $K=12$ and $K=11$ using the same initialization strategy as above (data not shown), and these models similarly identified a unique factor to the combination treatment group. 

\begin{figure}
    \centering
    \includegraphics[width=1.0\linewidth]{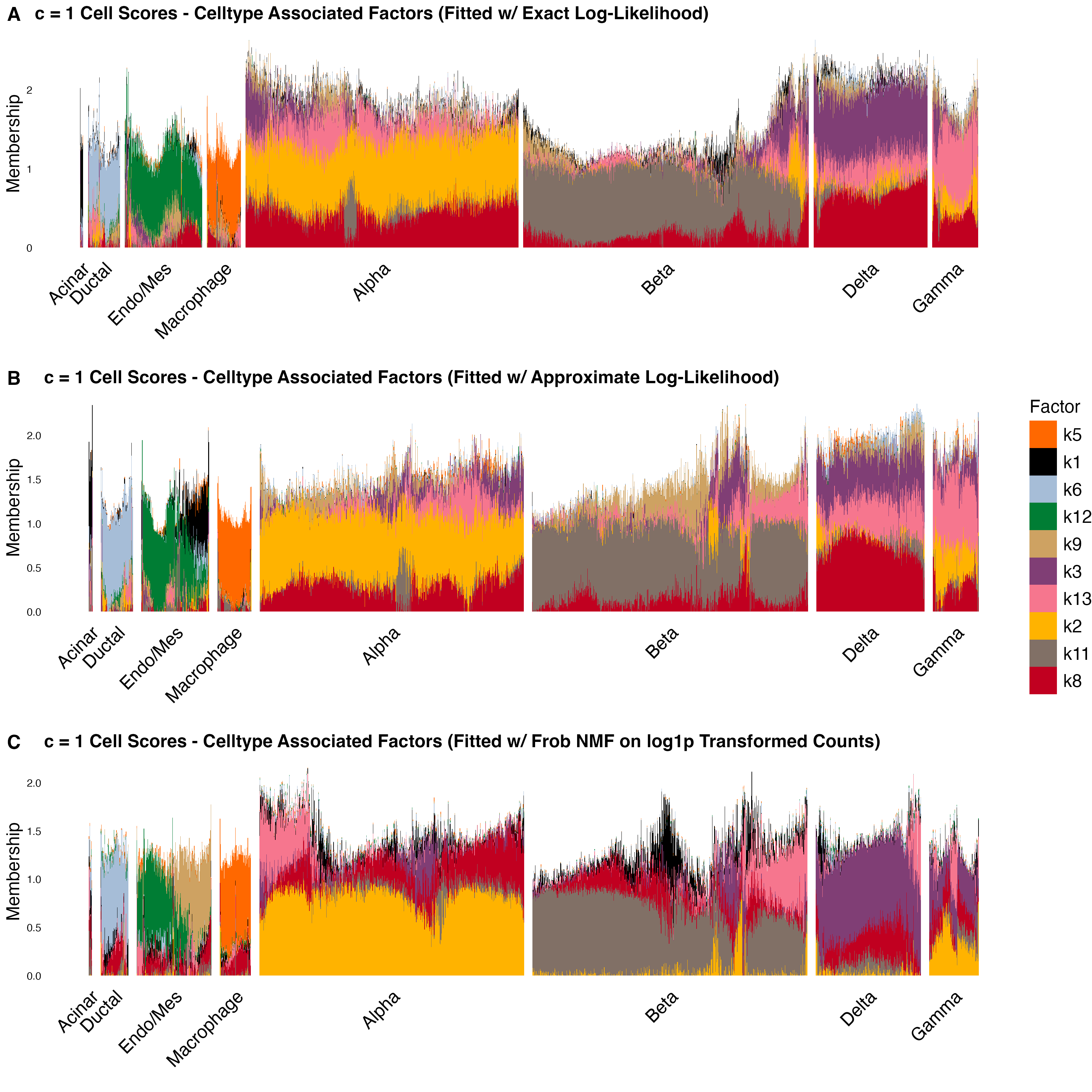}
    \caption{Celltype associated factors for the log1p NMF model fit by maximizing the exact log-likelihood (A), as well as with the Chebyshev approximation approach (B), and by fitting Frobenius NMF to the log1p transformed counts (C).}
    \label{fig:supp_lsa_approx1}
\end{figure}

\begin{figure}
    \centering
    \includegraphics[width=1.0\linewidth]{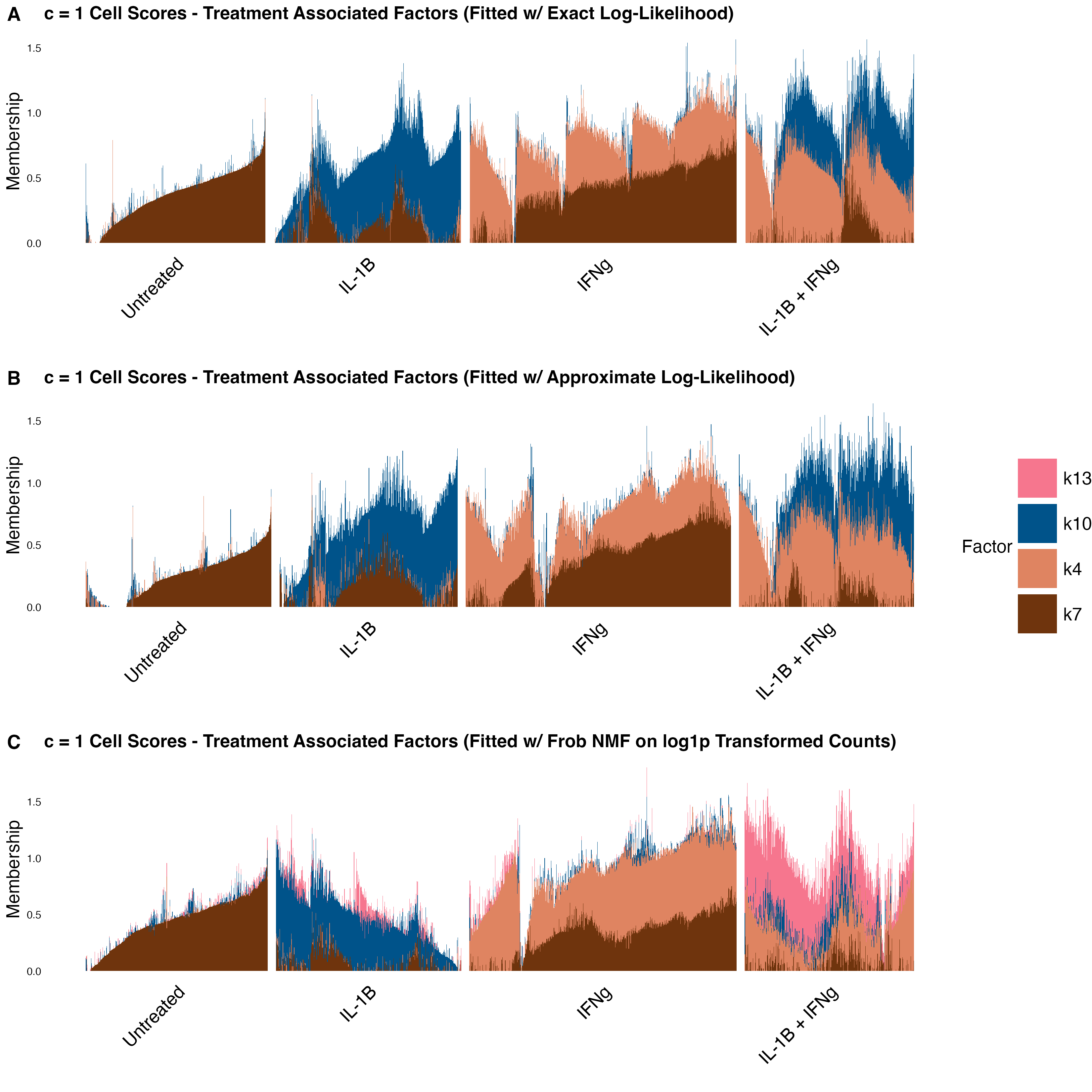}
    \caption{Treatment associated factors for the log1p NMF model fit by maximizing the exact log-likelihood (A), as well as with the Chebyshev approximation approach (B), and by fitting Frobenius NMF to the log1p transformed counts (C).}
    \label{fig:supp_lsa_approx2}
\end{figure}

Overall, both approximation approaches appear qualitatively reasonably similar for this setting of $c$ on this particular dataset. However, the quadratic approximation approach is clearly more similar to the exact log-likelihood approach, and in particular is more parsimonious in its representation of the combined treatment group than fitting Frobenius NMF to the log1p transformed counts.

\FloatBarrier

\section{Hoyer Sparsity Results}
\label{app:hoyer}

\begin{lemma} \label{thm:hoyer_pm}
Let $\mathbf{x}$ be a non-negative $n$-vector (for $n \geq 2$) with at least one non-zero entry, and define $\operatorname{pm}(\mathbf{x})$ as in \eqref{eq:pm_def}. Then,
\begin{equation*}
\frac{\|\mathbf{x}\|_1}{\|\mathbf{x}\|_2} \geq \sqrt{\frac{n}{\operatorname{pm}(\mathbf{x})}}.
\end{equation*}
\end{lemma}
\begin{proof}
Because $\|\mathbf{x}\|_2^2 \leq \|\mathbf{x}\|_\infty \|\mathbf{x}\|_1$ and $\frac{\|\mathbf{x}\|_1}{\|\mathbf{x}\|_\infty} = n \frac{\operatorname{mean}(\mathbf{x})}{\max(\mathbf{x})} = \frac{n}{\operatorname{pm}(\mathbf{x})}$, we can write
\begin{equation*}
\frac{\|\mathbf{x}\|_1}{\|\mathbf{x}\|_2} \;\ge\; \frac{\|\mathbf{x}\|_1}{\sqrt{\|\mathbf{x}\|_\infty\|\mathbf{x}\|_1}} = \sqrt{\frac{\|\mathbf{x}\|_1}{\|\mathbf{x}\|_\infty}} = \sqrt{\frac{n}{\operatorname{pm}(\mathbf{x})}}.
\end{equation*}
\end{proof}

\begin{lemma} \label{thm:pm_budget}
    Let $\mathbf{L} \in \mathbb{R}_{\geq 0}^{n \times K}$ and $\mathbf{F} \in \mathbb{R}_{\geq 0}^{p \times K}$, where $\mathbf{B} = \mathbf{L}\mathbf{F}^\top$. Let $\mathbf{l}_{\cdot k}$ and $\mathbf{f}_{\cdot k}$ denote columns $k$ of $\mathbf{L}$ and $\mathbf{F}$, respectively, and suppose that every column of $\mathbf{L}$ and of $\mathbf{F}$ has at least one non-zero entry. Further, denote $$w_k = \frac{\operatorname{mean}(\mathbf{l}_{\cdot k})\operatorname{mean}(\mathbf{f}_{\cdot k})}{\operatorname{mean}(\mathbf{B})}.$$ Then, $$\operatorname{pm}(\mathbf{l}_{\cdot k})\operatorname{pm}(\mathbf{f}_{\cdot k})w_k \leq \operatorname{pm}(\mathbf{B}).$$
\end{lemma}
\begin{proof}
    First, let $i^{*}$ be the smallest $i \in \{1, \dots, n\}$ such that $l_{i^* k} = \max(\mathbf{l}_{\cdot k})$, and similarly let $j^{*}$ be the smallest $j \in \{1, \dots, p\}$ such that $f_{j^* k} = \max(\mathbf{f}_{\cdot k})$. Then, we have
    \begin{align*}
        \max(\mathbf{B}) &\geq b_{i^*j^*} \\
        &= \sum_{m = 1}^{K} l_{i^* m}f_{j^* m} \\
        & \geq l_{i^* k}f_{j^* k} \\
        &= \max(\mathbf{l}_{\cdot k})\max(\mathbf{f}_{\cdot k}).
    \end{align*}
    Then, we have
    \begin{align*}
        \operatorname{pm}(\mathbf{B}) &= \frac{\max(\mathbf{B})}{\operatorname{mean}(\mathbf{B})} \\
        & \geq \frac{\max(\mathbf{l}_{\cdot k})\max(\mathbf{f}_{\cdot k})}{\operatorname{mean}(\mathbf{B})} \\
        &= \frac{\max(\mathbf{l}_{\cdot k})\max(\mathbf{f}_{\cdot k})}{\operatorname{mean}(\mathbf{B})} \cdot \frac{\operatorname{mean}(\mathbf{l}_{\cdot k})\operatorname{mean}(\mathbf{f}_{\cdot k})}{\operatorname{mean}(\mathbf{l}_{\cdot k})\operatorname{mean}(\mathbf{f}_{\cdot k})} \\
        &= \operatorname{pm}(\mathbf{l}_{\cdot k})\operatorname{pm}(\mathbf{f}_{\cdot k})w_k.
    \end{align*}
\end{proof}
Note that $\sum_{k = 1}^{K} w_k = 1$ because 
\begin{align*}
    \operatorname{mean}(\mathbf{B}) &= \frac{1}{n} \frac{1}{p} \sum_{i = 1}^{n} \sum_{j = 1}^{p} \sum_{k = 1}^{K} l_{ik}f_{jk} \\
    &=   \sum_{k = 1}^{K} \left(\frac{1}{n}\sum_{i = 1}^{n} l_{ik}\right) \left(\frac{1}{p}\sum_{j = 1}^{p}  f_{jk}\right) \\
    &= \sum_{k = 1}^{K} \operatorname{mean}(\mathbf{l}_{\cdot k}) \operatorname{mean}(\mathbf{f}_{\cdot k}).
\end{align*}

While Lemma \ref{thm:pm_budget} bounds the $\operatorname{pm}$ of $\mathbf{l}_{\cdot k}$ and $\mathbf{f}_{\cdot k}$, our ultimate goal is to bound the sparsity of these objects, which we do with the following Lemma and Theorem.

\begin{lemma} \label{thm:hoyer_joint}
    For a non-negative $n$-vector $\mathbf{x}$ with at least one non-zero entry, define
    $$d(\mathbf{x}) = \frac{\|\mathbf{x}\|_1}{\sqrt{n}\|\mathbf{x}\|_2}.$$
    Then, under the conditions of Lemma \ref{thm:pm_budget}, for each $k \in \{1, \dots, K\}$,
    $$d(\mathbf{l}_{\cdot k}) \, d(\mathbf{f}_{\cdot k}) \geq \sqrt{\frac{w_k}{\operatorname{pm}(\mathbf{B})}}.$$
\end{lemma}
\begin{proof}
    Starting from Lemma \ref{thm:hoyer_pm}, we have
    \begin{alignat*}{2}
       & & \frac{\|\mathbf{x}\|_1}{\|\mathbf{x}\|_2} & \;\geq\; \sqrt{\frac{n}{\operatorname{pm}(\mathbf{x})}} \\
       & \Rightarrow \qquad & \frac{\|\mathbf{x}\|_1}{\sqrt{n}\|\mathbf{x}\|_2} & \;\geq\; \frac{1}{\sqrt{\operatorname{pm}(\mathbf{x})}} \\
       & \Rightarrow \qquad & d(\mathbf{x}) & \;\geq\; \frac{1}{\sqrt{\operatorname{pm}(\mathbf{x})}} 
    \end{alignat*}
     Re-arranging the final line to $\frac{1}{d(\mathbf{x})^2}  \;\leq\; \operatorname{pm}(\mathbf{x})$ and applying this to both $\mathbf{l}_{\cdot k}$ and $\mathbf{f}_{\cdot k}$, and then using Lemma \ref{thm:pm_budget}, we have
    \begin{align*}
        \frac{1}{d(\mathbf{l}_{\cdot k})^{2}d(\mathbf{f}_{\cdot k})^{2}} \leq \operatorname{pm}(\mathbf{l}_{\cdot k})\operatorname{pm}(\mathbf{f}_{\cdot k}) \leq \frac{\operatorname{pm}(\mathbf{B})}{w_k}.
    \end{align*}
    Taking reciprocals and then square roots gives the result.
\end{proof}

\begin{theorem} \label{cor:hoyer_sum}
    Under the conditions of Lemma \ref{thm:pm_budget}, and writing $\kappa_m = \frac{\sqrt{m}}{\sqrt{m}-1}$,
    $$H(\mathbf{l}_{\cdot k}) + H(\mathbf{f}_{\cdot k}) \leq \kappa_n + \kappa_p - 2\sqrt{\kappa_n \kappa_p}\left(\frac{w_k}{\operatorname{pm}(\mathbf{B})}\right)^{1/4}.$$
\end{theorem}
\begin{proof}
    First, note that for a non-negative $m$-vector (for $m \geq 2$) $\mathbf{x}$ with at least one non-zero entry, we have 
    \begin{align*}
        \kappa_m(1 - d(\mathbf{x})) &= \frac{\sqrt{m}}{\sqrt{m}-1} - \frac{\sqrt{m}}{\sqrt{m}-1}\frac{\|\mathbf{x}\|_1}{\sqrt{m}\|\mathbf{x}\|_2} \\
        &= \frac{\sqrt{m} - \|\mathbf{x}\|_1 / \|\mathbf{x}\|_2}{\sqrt{m} - 1} \\
        &= H(\mathbf{x})
    \end{align*}
    Applying this fact to $\mathbf{l}_{\cdot k}$ and $\mathbf{f}_{\cdot k}$, we have
    \begin{align*}
    H(\mathbf{l}_{\cdot k}) + H(\mathbf{f}_{\cdot k}) &= \kappa_n (1 - d(\mathbf{l}_{\cdot k})) + \kappa_p (1 - d(\mathbf{f}_{\cdot k})) \\
    &= \kappa_n + \kappa_p - \left(\kappa_n d(\mathbf{l}_{\cdot k}) + \kappa_p d(\mathbf{f}_{\cdot k})\right) \\
        &\leq \kappa_n + \kappa_p - 2\sqrt{\kappa_n \kappa_p \, d(\mathbf{l}_{\cdot k})d(\mathbf{f}_{\cdot k})} \qquad (\textrm{AM-GM inequality})\\
        &\leq \kappa_n + \kappa_p - 2\sqrt{\kappa_n \kappa_p}\left(\frac{w_k}{\operatorname{pm}(\mathbf{B})}\right)^{1/4} \qquad (\textrm{Lemma \ref{thm:hoyer_joint}}).
    \end{align*}
\end{proof}
For sufficiently large $n$ and $p$ we have $\kappa_n, \kappa_p \approx 1$, in which case the bound is approximately:
$$H(\mathbf{l}_{\cdot k}) + H(\mathbf{f}_{\cdot k}) \leq 2\left(1 - \left(\frac{w_k}{\operatorname{pm}(\mathbf{B})}\right)^{1/4}\right).$$

Finally, we relate these bounds back to $c$ with the following Proposition. 

\begin{proposition} \label{prop:budget_monotone}
    Let $\boldsymbol{\Lambda} \in \mathbb{R}_{\geq 0}^{n \times p}$ be a fixed matrix of rates with at least one positive entry, and for $c \in (0, \infty]$ define $\mathbf{B}_c$ entrywise by $b_{ij}(c) = g(\lambda_{ij}; c)$, with $\mathbf{B}_{\infty} = \boldsymbol{\Lambda}$. Then, for any $0 < c_1 < c_2 \leq \infty$,
    $$\operatorname{pm}(\mathbf{B}_{c_1}) \leq \operatorname{pm}(\mathbf{B}_{c_2}).$$
\end{proposition}
\begin{proof}
    Since $\operatorname{pm}$ is invariant to multiplying its argument by a positive constant, we may drop the factor $\alpha_c$ from \eqref{eq:shifted_log_link} and work with $\theta_{ij}(c) = \log(1 + \lambda_{ij}/c)$.

    The proof has two steps. First, we show that $\boldsymbol{\theta}(c_1)$ is obtained from $\boldsymbol{\theta}(c_2)$ by applying, entrywise, an increasing concave function $\varphi$ with $\varphi(0) = 0$. Second, we show that applying such a function can never increase $\operatorname{pm}$.

    For the first step, suppose first that $c_2 < \infty$, and let $a = c_2/c_1 > 1$. Then, defining
    \begin{equation*}
        \varphi(t) = \log\left(1 + a(e^{t} - 1)\right)
    \end{equation*}
    for $t \geq 0$, we have
    \begin{equation*}
        \varphi(\theta_{ij}(c_2)) = \log\left(1 + \frac{c_2}{c_1}(1 + \lambda_{ij}/c_2 - 1)\right) 
        = \log(1 + \lambda_{ij}/c_1) 
        = \theta_{ij}(c_1).
    \end{equation*}

    Clearly $\varphi(0) = 0$, and differentiating, we have
    $$\varphi'(t) = \frac{ae^{t}}{1 + a(e^{t} - 1)} > 0, \qquad \varphi''(t) = \frac{ae^{t}(1 - a)}{\left(1 + a(e^{t} - 1)\right)^{2}} \leq 0,$$
    where the sign of $\varphi''$ follows from $a > 1$. If instead $c_2 = \infty$, then $\theta_{ij}(c_2) = \lambda_{ij}$ and the same conclusions hold for $\varphi(t) = \log(1 + t/c_1)$.

    For the second step, we use the fact that $\varphi$ is concave and $\varphi(0) = 0$, and thus the graph of $\varphi$ lies on or above the line connecting the origin to any point on the graph. Precisely, let $M = \max_{ij} \theta_{ij}(c_2)$, which is positive because at least one rate is positive. Any $t \in [0, M]$ can be written as the convex combination $t = \tfrac{t}{M} \cdot M + \left(1 - \tfrac{t}{M}\right) \cdot 0$, so concavity gives
    $$\varphi(t) \;\geq\; \frac{t}{M}\,\varphi(M) + \left(1 - \frac{t}{M}\right)\varphi(0) \;=\; \frac{t}{M}\,\varphi(M).$$
    Since every entry of $\boldsymbol{\theta}(c_2)$ lies in $[0, M]$, we can apply this bound with $t = \theta_{ij}(c_2)$ for every $(i, j)$, which gives
    $$\theta_{ij}(c_1) = \varphi\left(\theta_{ij}(c_2)\right) \;\geq\; \frac{\varphi(M)}{M}\, \theta_{ij}(c_2),$$
    and averaging over all $np$ entries,
    $$\operatorname{mean}\left(\boldsymbol{\theta}(c_1)\right) \;\geq\; \frac{\varphi(M)}{M}\, \operatorname{mean}\left(\boldsymbol{\theta}(c_2)\right).$$
    
    Finally, because $\varphi$ is increasing, the largest entry of $\boldsymbol{\theta}(c_1)$ is the image of the largest entry of $\boldsymbol{\theta}(c_2)$ under $\varphi$, so $\max\left(\boldsymbol{\theta}(c_1)\right) = \varphi(M)$. That is, passing from $\boldsymbol{\theta}(c_2)$ to $\boldsymbol{\theta}(c_1)$ shrinks the maximum by exactly the factor $\varphi(M)/M$, while the mean shrinks by at most that factor. Combining, we have
    $$\operatorname{pm}\left(\boldsymbol{\theta}(c_1)\right) = \frac{\varphi(M)}{\operatorname{mean}\left(\boldsymbol{\theta}(c_1)\right)} \;\leq\; \frac{\varphi(M)}{\frac{\varphi(M)}{M}\operatorname{mean}\left(\boldsymbol{\theta}(c_2)\right)} = \frac{M}{\operatorname{mean}\left(\boldsymbol{\theta}(c_2)\right)} = \operatorname{pm}\left(\boldsymbol{\theta}(c_2)\right).$$
\end{proof}

\section{Calibration of simulation parameters to the pancreas data}
\label{app:sim2_calibration}

Here we describe how the constants $\left(\pi_0^{(c_{\textrm{true}})}, \mu^{(c_{\textrm{true}})}, \sigma^{(c_{\textrm{true}})}, T_{\textrm{low}}^{(c_{\textrm{true}})}, T_{\textrm{high}}^{(c_{\textrm{true}})}\right)$ of the simulation in Section \ref{sec:sim2} were chosen. 

We set $T_{\textrm{low}}^{(c_{\textrm{true}})} = \log(1 + 0.05/c_{\textrm{true}})$ (and $T_{\textrm{low}}^{(\infty)} = 0.05$). The remaining parameters were set by matching (as closely as possible) the simulated dataset to two properties of the pancreas dataset analyzed in Section \ref{sec:pancreas}. Specifically, we set the parameters so that simulated data had the same mean count as the pancreas dataset ($\approx 0.674$) and approximately the same peak to mean curve
\begin{align*}
    c \;\mapsto\; \operatorname{pm}\left(\log(1 + \mathbf{Y}/c)\right)
\end{align*}
as the pancreas dataset. 

To do this, we minimize the sum of the squared distances between the real curve $c \;\mapsto\; \operatorname{pm}\left(\log(1 + \mathbf{Y}/c)\right)$ for the pancreas dataset and the expected $\operatorname{pm}$ curve, evaluated at the points $c = (10^{-4}, \dots, 10^4)$. Precisely, each $c_{\textrm{true}}$ is calibrated separately (we suppress the superscripts below), and we minimize
$$\ell(\pi_0, \sigma, T_{\textrm{high}}) = \sum_{c \in \{10^{-4}, \dots, 10^{4}\}} \left[\log \operatorname{pm}\left(\log(1 + \mathbf{Y}^{*}/c)\right) - \log \widehat{\operatorname{pm}}_{c}\left(\pi_0, \sigma, T_{\textrm{high}}\right)\right]^2,$$
where $\widehat{\operatorname{pm}}_{c}$ is the value of $\operatorname{pm}\left(\log(1 + \mathbf{Y}_{\textrm{sim}}/c)\right)$ averaged over three simulated datasets $\mathbf{Y}_{\textrm{sim}}$. Because a $500 \times 500$ simulated matrix cannot be expected to attain the maximum count of the much larger pancreas matrix, $\mathbf{Y}^{*}$ denotes the pancreas matrix with its maximum count replaced by $t^{*} = 8{,}820$, the median of the maximum of $n \cdot p = 250{,}000$ draws from the empirical distribution of the pancreas counts.

We minimize $\ell$ over $(\pi_0, \sigma, T_{\textrm{high}})$ using Nelder--Mead. The parameter $\mu$ is not included in this search; instead, it is \textit{profiled}: at every evaluation of $\ell$, given the current $(\pi_0, \sigma, T_{\textrm{high}})$, we set $\mu$ to the value that makes the expected mean count of the simulated data equal to the mean count of the pancreas dataset, found by one-dimensional root-finding. The calibrated values are shown below.

\begin{center}
\begin{tabular}{lccccc}
\toprule
$c_{\textrm{true}}$ & $\pi_0$ & $\mu$ & $\sigma$ & $T_{\textrm{high}}$ & $T_{\textrm{low}}$ \\
\midrule
$10^{-3}$ & $0.959$ & $2.196$  & $0.099$ & $15.57$  & $3.93$ \\
$1$       & $0.838$ & $-0.963$ & $0.665$ & $7.82$   & $0.049$ \\
$\infty$  & $0.901$ & $-2.418$ & $1.467$ & $911.22$ & $0.05$ \\
\bottomrule
\end{tabular}
\end{center}
\section{The geometry of log1p NMF for \texorpdfstring{$c \rightarrow 0^+$}{c -> 0+}}
\label{sec:app_geometry}

\def\blambda{\boldsymbol{\lambda}}

Let $\boldsymbol{\lambda}_i=(\lambda_{i1},\dots,\lambda_{ip})$ denote the vector of mean values for the $i$th row of $\mathbf{Y}$. In standard Poisson NMF ($c \rightarrow \infty$) the range of possible values of $\boldsymbol{\lambda}_i$ has a simple geometric relationship with the factors (columns of $\mathbf{F}$): it is simply the conical hull of the columns of $\mathbf{F}$ (i.e., the set of non-negative linear combinations of the columns). The implications of this simple geometry for identifiability of NMF have been well studied (e.g., \cite{donoho2003does}). Here we provide preliminary results on the corresponding geometry for the log1p NMF model in the case $c \rightarrow 0^+$. 

 Since the range of possible values of the vector $\boldsymbol{\lambda}_i$ is the same for each $i$ we drop the subscript $i$, and study the achievable values of $\boldsymbol{\lambda}=(\lambda_1,\dots,\lambda_p)$ where
\begin{equation} \label{eq:log1p_k2}
    \alpha_c \log(1 + \lambda_{j} / c) 
    = \mathbf{l} \cdot \mathbf{f}_{j \cdot}  = \eta_{j} \text{ say}.
\end{equation}
The question we aim to answer is: what are the achievable values of $\boldsymbol{\lambda}$ when $\mathbf{l} \in \mathbb{R}^K_{\geq 0}$ as $c \rightarrow 0^+$? And, how does this achievable set compare to the achievable set under standard Poisson NMF?

To simplify matters, in this section we impose the constraint $\sum_j \lambda_j = 1$ on model \eqref{eq:log1p_k2}. This normalization allows us to study the geometry of the achievable \textit{directions} of $\boldsymbol{\lambda}$. Note that these directions are invariant to the scaling factor; the set of achievable directions would be identical if we imposed $\sum_j \lambda_j = A$ for any $A > 0$. Consequently, the set of achievable values of $\boldsymbol{\lambda}$ (in the limit $c \rightarrow 0^+$) forms a cone generated by the achievable directions (the set $\mathcal{L}_{0^+}(\mathbf{F})$ defined below).

Before characterizing the achievable directions of $\boldsymbol{\lambda}$ in the log1p NMF model as $c \rightarrow 0^+$, it is useful to first understand the asymptotic behavior of the link function under the sum to $1$ constraint. Let $\Delta^{p-1} = \{\boldsymbol{\lambda} \geq 0 : \sum_j \lambda_j = 1\}$ denote the probability simplex in $\mathbb{R}^p$, and let $\Delta^{p-1}_{>0} = \{\boldsymbol{\lambda} \in \Delta^{p-1} : \lambda_j > 0 \text{ for all } j\}$ denote its strictly positive elements. Recalling the softmax function $\sigma: \mathbb{R}^p \rightarrow \Delta^{p-1}$ defined component-wise as $$\sigma(\mathbf{x})_i = \frac{\exp(x_i)}{\sum_{j = 1}^p \exp(x_j)},$$ we have the following lemma:
\begin{lemma}
\label{lem:softmax}
Consider the vector $\boldsymbol{\lambda}$ defined in \eqref{eq:log1p_k2} with the constraint $\sum_{j = 1}^p \lambda_j = 1$. For $c \in (0, 1)$ (so that $\alpha_c = 1$), let $\eta_{\max} := \max_j \eta_j$. As $c \to 0^+$, the following properties hold:
\begin{enumerate}
    \item $\eta_{\max} \to \infty$.
    \item The vector $\boldsymbol{\lambda}$ converges to $\sigma(\boldsymbol{\eta})$ with the following bound:
    \begin{equation*}
        \lambda_j = \sigma(\boldsymbol{\eta})_j + \mathcal{O}(p e^{-\eta_{\max}}).
    \end{equation*}
\end{enumerate}
\end{lemma}

This lemma allows us to characterize the achievable directions of $\boldsymbol{\lambda}$ as $c \rightarrow 0^+$ in terms of the softmax function $\sigma$. By analogy to standard Poisson NMF, one might suspect that the achievable directions of $\boldsymbol{\lambda}$ are simply the image of the conical hull of the columns of $\mathbf{F}$ under $\sigma$. This turns out to be true for the \textit{strictly positive} achievable directions, \textit{provided that $\mathbf{F}$ contains the constant vector $\mathbf{1}$ in the conical hull of its columns}.

Before describing this more precisely, we introduce additional definitions and notation. First, to represent the cone generated by the columns of $\mathbf{F}$, we write $\operatorname{cone}(\mathbf{F}) := \{\mathbf{F} \mathbf{v} : \mathbf{v} \geq 0\}$. Second, we introduce notation for the range of achievable \textit{directions} of $\boldsymbol{\lambda}$, viewed as a function of $\mathbf{F}$. For fixed $c$, let
$$\Lambda_c(\mathbf{F}) := \Big\{ \boldsymbol{\lambda} \in \Delta^{p-1} \;:\; \alpha_c \log(1 + \boldsymbol{\lambda}/c) \in \operatorname{cone}(\mathbf{F}) \Big\},$$
where the logarithm is applied entrywise. Then, for $c' \in [0, \infty]$, define
$$\mathcal{L}_{c'}(\mathbf{F}) := \Big\{ \boldsymbol{\lambda}^{*} \in \Delta^{p-1} \;:\; \boldsymbol{\lambda}(t) \rightarrow \boldsymbol{\lambda}^{*} \text{ for some sequences } c_t \rightarrow c' \text{ and } \boldsymbol{\lambda}(t) \in \Lambda_{c_t}(\mathbf{F}) \Big\},$$
the set of achievable limiting directions of $\boldsymbol{\lambda}$ as $c \rightarrow c'$. We focus here on the two extreme cases $\mathcal{L}_{0^{+}}(\mathbf{F})$ and $\mathcal{L}_{\infty}(\mathbf{F})$.

We now state our main result:

\begin{theorem} \label{thm:c0_dichotomy}
    Let $\F \in \mathbb{R}^{p \times K}_{\geq 0}$ where each column has at least $1$ non-zero entry. There exists a $\boldsymbol{\lambda}^* \in \mathcal{L}_{0^{+}}(\mathbf{F})$ with all non-zero entries if and only if $\mathbf{1} \in \operatorname{cone}(\mathbf{F}).$ When $\mathbf{F}$ meets this condition, the set of achievable directions of $\boldsymbol{\lambda}$ where all entries are non-zero, $\mathcal{L}_{0^{+}}(\mathbf{F}) \cap \Delta^{p-1}_{>0}$, is
    \begin{equation}
    \sigma(\bar{\mathcal{C}}) := \{\sigma(\mathbf{x}): \mathbf{x} \in \bar{\mathcal{C}}\}, \label{eq:nz_cone}
    \end{equation}
    where $\bar{\mathcal{C}} = \operatorname{cone}(\mathbf{g}_1, \dots, \mathbf{g}_K)$ is the cone of the centered factors $\mathbf{g}_k = \bff_{\cdot k} - \Big(\tfrac{1}{p}
\textstyle\sum_j f_{jk}\Big) \mathbf{1}$.
\end{theorem}

It is natural to ask, when $\mathbf{1} \in \operatorname{cone}(\mathbf{F})$, what equation \eqref{eq:nz_cone} looks like, and how its size compares to the achievable sets for choices of $\mathbf{F}$ in standard Poisson NMF. Unfortunately, the high-dimensional, curved sets of equation \eqref{eq:nz_cone} are difficult to characterize simply. However, for the simple case where $K = 2$, equation \eqref{eq:nz_cone} can be described completely as a one-dimensional curve. To see this, note that we can write $$\mathbf{1} = a \bff_{\cdot 1} + b \bff_{\cdot 2}$$ for some $a, b \geq 0$. Averaging the coordinates of both sides gives $1 = a \bar{f}_1 + b \bar{f}_2$, where $\bar{f}_k = \tfrac{1}{p}\sum_j f_{jk}$; subtracting this identity from each coordinate of the equation above then gives $a \mathbf{g}_1 + b \mathbf{g}_2 = \mathbf{0}$, so the two centered factors point in opposite directions. The cone $\bar{\mathcal{C}}$ is therefore a ray from the origin (when exactly one factor is constant, i.e., an intercept), or a full line through the origin (when neither factor is constant, which forces $a, b > 0$). In either case its image $\sigma(\bar{\mathcal{C}})$ is a smooth curve in the simplex passing through the uniform vector $\sigma(\mathbf{0}) = \tfrac{1}{p}\mathbf{1}$. (In the degenerate case where both factors are constant, $\bar{\mathcal{C}} = \{\mathbf{0}\}$ and the curve collapses to the single uniform point.)

Similarly, for standard Poisson NMF when we impose the constraint $\boldsymbol{\lambda} \in \Delta^{p-1}$, the achievable set is a line segment, since for any loadings $\mathbf{v} \geq 0$ with $\mathbf{v} \neq \mathbf{0}$ we can write
$$\boldsymbol{\lambda} \;=\; \frac{\mathbf{F}\mathbf{v}}{\mathbf{1}^\top \mathbf{F}\mathbf{v}} \;=\; w \tilde{\bff}_1 + (1 - w) \tilde{\bff}_2,$$
where $\tilde{\bff}_k = \bff_{\cdot k} / s_k$ is the $k$th column normalized by its sum $s_k = \sum_j f_{jk}$, and $w = v_1 s_1 / (v_1 s_1 + v_2 s_2) \in [0, 1]$. As $\mathbf{v} \geq 0$ varies, $w$ spans all of $[0, 1]$, so the achievable set is exactly the segment joining the two normalized columns $\tilde{\bff}_1$ and $\tilde{\bff}_2$. In the notation above, this segment is exactly $\mathcal{L}_{\infty}(\mathbf{F})$, since the link function converges to the identity as $c \rightarrow \infty$.

Thus, both achievable sets are curves when $K = 2$, and we can measure their size by computing the arclength. Note that the arclength computed from a parameterization of a curve agrees with the intrinsic size of the curve (its one-dimensional Hausdorff measure) whenever the parameterization is injective. Both parameterizations here are injective: the segment $w \mapsto w \tilde{\bff}_1 + (1 - w) \tilde{\bff}_2$ is affine in $w$, and $\sigma$ is injective on $\bar{\mathcal{C}}$, since $\sigma(\mathbf{x}) = \sigma(\mathbf{y})$ requires $\mathbf{x} - \mathbf{y}$ to be a constant vector, and every element of $\bar{\mathcal{C}}$ has mean $0$. Thus, comparing arclengths is a genuine comparison of the sizes of the two achievable sets. This yields the following theorem:

\begin{theorem} \label{thm:c0_arclength}
    For $K = 2$ and $p \geq 5$, there exists an $\mathbf{F} \in \mathbb{R}^{p \times 2}_{\geq 0}$ with $\mathbf{1} \in \operatorname{cone}(\mathbf{F})$ such that
    \begin{equation*}
        \operatorname{arclength}\big(\mathcal{L}_{0^+}(\mathbf{F})\big) > \operatorname{arclength}\big(\mathcal{L}_{\infty}(\widetilde{\mathbf{F}})\big)
    \end{equation*}
    for any $\widetilde{\mathbf{F}} \in \mathbb{R}^{p \times 2}_{\geq 0}.$
\end{theorem}

We prove Theorems \ref{thm:c0_dichotomy} and \ref{thm:c0_arclength}, as well as supporting Lemmas, in the section below.
\subsection{Proofs}

We first prove Lemma \ref{lem:softmax}:

\begin{proof}
Since $c \to 0^+$, we have $\alpha_c = 1$. Inverting the link function $\log(1 + \lambda_j/c) = \eta_j$ yields $\lambda_j = c(\exp(\eta_j) - 1)$. Summing over $j$ and applying the constraint $\sum \lambda_j = 1$, we have
\begin{equation*}
    1 = \sum_{j=1}^p c(\exp(\eta_j) - 1) = c \left( \sum_{j=1}^p \exp(\eta_j) - p \right).
\end{equation*}
Rearranging for $c$ gives:
\begin{equation}
    c = \frac{1}{\sum_{j=1}^p \exp(\eta_j) - p}. \label{eq:pf_c_identity}
\end{equation}
As $c \to 0^+$, the denominator in \eqref{eq:pf_c_identity} must approach $\infty$. Since $p$ is finite, this requires $\sum \exp(\eta_j) \to \infty$, which implies $\eta_{\max} \to \infty$. 

To establish the softmax equivalence, we can substitute \eqref{eq:pf_c_identity} back into the expression for $\lambda_j$ and write
\begin{equation*}
    \lambda_j = \frac{\exp(\eta_j) - 1}{\sum_{j'=1}^p (\exp(\eta_{j'}) - 1)}.
\end{equation*}
Let $\sigma(\boldsymbol{\eta})_j = \frac{\exp(\eta_j)}{\sum_{j' = 1}^{p} \exp(\eta_{j'})}$ denote the standard softmax function. We analyze the ratio
\begin{align*}
    \frac{\lambda_j}{\sigma(\boldsymbol{\eta})_j} &= \frac{\exp(\eta_j) - 1}{\exp(\eta_j)} \cdot \frac{\sum_{j' = 1}^{p} \exp(\eta_{j'})}{\sum_{{j'} = 1}^{p} \exp(\eta_{j'}) - p} \\
    &= (1 - e^{-\eta_j}) \cdot \left( 1 - \frac{p}{\sum_{j' = 1}^{p} \exp(\eta_{j'})} \right)^{-1}.
\end{align*}
Using the Taylor expansion $(1-x)^{-1} = 1 + x + \mathcal{O}(x^2)$ for the second term, we have
\begin{align*}
    \lambda_j &= \sigma(\boldsymbol{\eta})_j (1 - e^{-\eta_j}) \left( 1 + \frac{p}{\sum_{j' = 1}^{p} \exp(\eta_{j'})} + \mathcal{O}\left(\frac{p^2}{(\sum_{j' = 1}^{p} e^{\eta_{j'}})^2}\right) \right) \\
    &= \sigma(\boldsymbol{\eta})_j + \sigma(\boldsymbol{\eta})_j \left( \frac{p}{\sum_{j' = 1}^{p} \exp(\eta_{j'})} - e^{-\eta_j} \right) + \text{h.o.t.}
\end{align*}
The residual term is dominated by $\max( \frac{p \sigma(\boldsymbol{\eta})_j}{\sum_{j'} e^{\eta_{j'}}}, \sigma(\boldsymbol{\eta})_j e^{-\eta_j} )$, where we note that $\sigma(\boldsymbol{\eta})_j e^{-\eta_j} = (\sum_{j' = 1}^{p} e^{\eta_{j'}})^{-1}$. Since $\sum_{j' = 1}^{p} \exp(\eta_{j'}) \geq \exp(\eta_{\max})$, the error is $\mathcal{O}(p e^{-\eta_{\max}})$. As $\eta_{\max} \to \infty$, $\boldsymbol{\lambda} \to \sigma(\boldsymbol{\eta})$.
\end{proof}
Next, we state and prove an intermediate result that is used in the proof of Theorem \ref{thm:c0_dichotomy}.

\begin{lemma} \label{lem:support}
Let $\F \in \mathbb{R}^{p \times K}_{\geq 0}$ where each column has at least $1$ non-zero entry. For a unit vector
$\bw \in \mathbb{R}^{K}_{\geq 0}$, let
$$m(\bw) = \max_j\, \bw \cdot \bff_{j \cdot}, \qquad
S(\bw) = \{j : \bw \cdot \bff_{j \cdot} = m(\bw)\},$$ and define
$$p^{*}(\F) \;=\; \max_{\bw \geq 0,\, \|\bw\| = 1}\, |S(\bw)|.$$
For every $\boldsymbol{\lambda}^* \in \mathcal{L}_{0^{+}}(\mathbf{F})$, there exists some unit vector $\mathbf{w} \geq 0$ such that $S(\mathbf{w})$ contains the support of $\boldsymbol{\lambda}^*$. That is, $\boldsymbol{\lambda}^*_j = 0$ for all $j \notin S(\mathbf{w})$. Consequently $\boldsymbol{\lambda}^*$ has at most $p^{*}(\F)$ non-zero entries.
\end{lemma}

\begin{proof}
Fix some $\mathbf{F}$ and $\boldsymbol{\lambda}^* \in \mathcal{L}_{0^{+}}(\mathbf{F})$. By the definition of $\mathcal{L}_{0^{+}}(\mathbf{F})$, there exist sequences $c_t \rightarrow 0^+$ and $\boldsymbol{l}(t) \geq 0$ such that $\boldsymbol{\lambda}(t) \rightarrow \boldsymbol{\lambda}^*$ elementwise. For this sequence, let $\eta_j(t) = \boldsymbol{l}(t) \cdot \boldsymbol{f}_{j\cdot}$ and $\eta_{\max}(t) = \max_j\eta_j(t)$. By Lemma \ref{lem:softmax},  $\eta_{\max}(t) \rightarrow \infty$, and since $\mathbf{F}$ is fixed, we must have $||\boldsymbol{l}(t)|| \rightarrow \infty$. Let $\boldsymbol{u}(t) = \frac{\boldsymbol{l}(t)}{||\boldsymbol{l}(t)||}$ be the sequence of directions of $\boldsymbol{l}(t)$. Since $||\boldsymbol{u}(t)|| = 1$ for all $t$, by Bolzano-Weierstrass there must exist some subsequence such that $\boldsymbol{u}(t) \rightarrow \mathbf{w}$ elementwise, where $\mathbf{w} \geq 0$ is a unit vector. Note that $\boldsymbol{\lambda}(t) \rightarrow \boldsymbol{\lambda}^*$ along this subsequence.

    Recalling that $S(\mathbf{w})$ is the index set of rows of $\mathbf{F}$ with maximal inner product with $\mathbf{w}$, fix $j \in S(\mathbf{w})$ (which must exist because $\mathbf{F}$ has a finite number of rows) and $i \notin S(\mathbf{w})$ (which we may assume exists, since otherwise the statement is vacuously true). Then, we can write
    $$\frac{\eta_{j}(t) - \eta_i(t)}{\|\boldsymbol{l}(t)\|} \;=\;
\frac{\boldsymbol{l}(t)}{\|\boldsymbol{l}(t)\|} \cdot (\bff_{j \cdot} - \bff_{i
\cdot}) \;\longrightarrow\; \bw \cdot (\bff_{j \cdot} - \bff_{i
\cdot}) \;>\; 0.$$ Then, since $\|\boldsymbol{l}(t)\| \rightarrow \infty$, we must have that $\eta_{j}(t) - \eta_i(t) \rightarrow \infty$. As a result, we can write
$$\sigma(\boldsymbol{\eta}(t))_i \;=\;
\frac{e^{\eta_i(t)}}{\sum_{i'} e^{\eta_{i'}(t)}} \;\leq\;
e^{\eta_i(t) - \eta_{\max}(t)} \;\leq\;
e^{-(\eta_{j}(t) - \eta_i(t))} \;\longrightarrow\; 0,$$
and the additive error in Lemma \ref{lem:softmax} is
$\mathcal{O}(p\, e^{-\eta_{\max}(t)}) \rightarrow 0$, so
$\lambda_i(t) \rightarrow 0$.

Thus, we have shown the existence of a unit vector $\mathbf{w}$ such that the support of $\boldsymbol{\lambda}^*$ must be contained in $S(\mathbf{w})$. Thus, letting $\mathbf{w}^*$ be a maximizer of $|S(\mathbf{w})|$ over all unit vectors $\mathbf{w}$, the size of the support of $\boldsymbol{\lambda}^*$ must be at most $|S(\mathbf{w}^*)| = p^*(\mathbf{F})$.
\end{proof}
We are now able to show Theorem \ref{thm:c0_dichotomy}:
\begin{proof}
Before proving the if and only if statement, we first show the following: (a) Every $\boldsymbol{\lambda}^* \in \mathcal{L}_{0^{+}}(\mathbf{F})$ with $\boldsymbol{\lambda}^* > 0$ elementwise also lies in $\sigma(\bar{\mathcal{C}})$. (b) If $\mathbf{1} \in \operatorname{cone}(\mathbf{F})$, then every element of $\sigma(\bar{\mathcal{C}})$ is an achievable limit.

\textit{Proof of (a)}: For convenience, define the centering function $\Pi(\boldsymbol{\eta}) = \boldsymbol{\eta} -
(\tfrac1p \sum_j \eta_j) \mathbf{1}$. Note that $\sigma(\Pi(\boldsymbol{\eta})) = \sigma(\boldsymbol{\eta})$ since for each $i$,
\begin{equation}
    \sigma(\Pi(\boldsymbol{\eta}))_i = \frac{\exp(\eta_i - 1/p\sum_j \eta_j)}{\sum_{i'} \exp(\eta_{i'}  - 1/p\sum_j \eta_j)} = \frac{\exp(-1/p\sum_j \eta_j)\exp(\eta_i)}{\exp(-1/p\sum_j \eta_j)\sum_{i'} \exp(\eta_{i'})} = \sigma(\boldsymbol{\eta})_i. \label{eq:softmax_shift}
\end{equation}
Now fix $\boldsymbol{\lambda}^* \in \mathcal{L}_{0^{+}}(\mathbf{F})$ with $\boldsymbol{\lambda}^* > 0$ elementwise, achieved along sequences $c_t \rightarrow 0^+$ and $\boldsymbol{l}(t) \geq 0$, and write $\boldsymbol{\eta}(t) = \mathbf{F} \boldsymbol{l}(t) \in \operatorname{cone}(\mathbf{F})$.

We first show that every pairwise difference $\eta_i(t) - \eta_j(t)$ converges to a finite limit. By Lemma \ref{lem:softmax}, $\boldsymbol{\lambda}(t) = \sigma(\boldsymbol{\eta}(t)) + o(1)$, and since $\boldsymbol{\lambda}(t) \rightarrow \boldsymbol{\lambda}^*$ it follows that $\sigma(\boldsymbol{\eta}(t)) \rightarrow \boldsymbol{\lambda}^*$ as well. Because $\boldsymbol{\lambda}^* > 0$, for every pair $i, j$ we have
\begin{equation*}
    e^{\eta_i(t) - \eta_j(t)} \;=\; \frac{\sigma(\boldsymbol{\eta}(t))_i}{\sigma(\boldsymbol{\eta}(t))_j} \;\longrightarrow\; \frac{\lambda^*_i}{\lambda^*_j} \in (0, \infty),
\end{equation*}
and taking logarithms gives $\eta_i(t) - \eta_j(t) \rightarrow \log(\lambda^*_i / \lambda^*_j)$. 

Next, convergence of the pairwise differences implies convergence of the centered vector, because each coordinate of $\Pi(\boldsymbol{\eta}(t))$ is an average of pairwise differences:
\begin{equation*}
    \Pi(\boldsymbol{\eta}(t))_i \;=\; \frac{1}{p} \sum_{k=1}^p \big(\eta_i(t) - \eta_k(t)\big) \;\longrightarrow\; \frac{1}{p} \sum_{k=1}^p \log(\lambda^*_i / \lambda^*_k) \;=\; \Pi(\log \boldsymbol{\lambda}^*)_i,
\end{equation*}
where the logarithm is applied entrywise. Write $\mathbf{x}^* = \Pi(\log \boldsymbol{\lambda}^*)$ for this limit.

We now argue that $\mathbf{x}^* \in \bar{\mathcal{C}}$. Since $\Pi$ is linear, the image of $\operatorname{cone}(\mathbf{F})$ under $\Pi$ is the cone generated by the centered columns:
\begin{equation*}
    \Pi(\operatorname{cone}(\mathbf{F})) \;=\; \operatorname{cone}\big(\Pi(\bff_{\cdot 1}), \dots, \Pi(\bff_{\cdot K})\big) \;=\; \operatorname{cone}(\mathbf{g}_1, \dots, \mathbf{g}_K) \;=\; \bar{\mathcal{C}}.
\end{equation*}
In particular, $\Pi(\boldsymbol{\eta}(t)) \in \bar{\mathcal{C}}$ for every $t$. A finitely generated cone is a closed set, and a closed set contains the limit of any convergent sequence of its elements; therefore $\mathbf{x}^* \in \bar{\mathcal{C}}$. 

Finally, combining the shift invariance \eqref{eq:softmax_shift} with the continuity of $\sigma$,
\begin{equation*}
    \boldsymbol{\lambda}^* \;=\; \lim_{t} \sigma(\boldsymbol{\eta}(t)) \;=\; \lim_{t} \sigma\big(\Pi(\boldsymbol{\eta}(t))\big) \;=\; \sigma(\mathbf{x}^*) \;\in\; \sigma(\bar{\mathcal{C}}),
\end{equation*}
which completes the proof of (a).

\textit{Proof of (b)}: Suppose $\mathbf{1} \in \operatorname{cone}(\mathbf{F})$, i.e., $\mathbf{1} = \mathbf{F} \mathbf{v}$ for some $\mathbf{v} \geq 0$. Fix some $\sigma(\mathbf{x}) \in \sigma(\bar{\mathcal{C}})$; since $\bar{\mathcal{C}} = \Pi(\operatorname{cone}(\mathbf{F}))$, we can write $\mathbf{x} = \Pi(\mathbf{F} \mathbf{u})$ for some $\mathbf{u} \geq 0$. Under the sequence of loadings 
\begin{equation*}
    \boldsymbol{l}(t) \;=\; \mathbf{u} + t \mathbf{v}, 
\end{equation*}
we have,
\begin{equation*}
    \boldsymbol{\eta}(t) \;=\; \mathbf{F} \mathbf{u} + t \mathbf{1},
\end{equation*}
so $\Pi(\boldsymbol{\eta}(t)) = \Pi(\mathbf{F} \mathbf{u}) = \mathbf{x}$ for every $t$ since centering removes the $t \mathbf{1}$ term. Note that $\eta_{\max}(t) = \max_j (\mathbf{F} \mathbf{u})_j + t \rightarrow \infty$, so by the identity \eqref{eq:pf_c_identity} the value of $c$ enforcing the constraint $\sum_j \lambda_j = 1$, namely
\begin{equation*}
    c_t \;=\; \frac{1}{\sum_{j=1}^p \exp(\eta_j(t)) - p},
\end{equation*}
is positive for all sufficiently large $t$ and satisfies $c_t \rightarrow 0^+$. Hence $\boldsymbol{\lambda}(t) \in \Lambda_{c_t}(\mathbf{F})$ with $c_t \rightarrow 0^+$. Finally, Lemma \ref{lem:softmax} together with the shift invariance \eqref{eq:softmax_shift} gives
\begin{equation*}
    \boldsymbol{\lambda}(t) \;=\; \sigma(\boldsymbol{\eta}(t)) + \mathcal{O}\big(p e^{-\eta_{\max}(t)}\big) \;=\; \sigma\big(\Pi(\boldsymbol{\eta}(t))\big) + o(1) \;=\; \sigma(\mathbf{x}) + o(1) \;\longrightarrow\; \sigma(\mathbf{x}).
\end{equation*}
Therefore $\sigma(\mathbf{x}) \in \mathcal{L}_{0^{+}}(\mathbf{F})$, and since $\mathbf{x} \in \bar{\mathcal{C}}$ was arbitrary, every element of $\sigma(\bar{\mathcal{C}})$ is an achievable limit. This completes the proof of (b).

Thus, (a) shows that every strictly positive element of $\mathcal{L}_{0^{+}}(\mathbf{F})$ lies in $\sigma(\bar{\mathcal{C}})$, and (b) shows that, when $\mathbf{1} \in \operatorname{cone}(\mathbf{F})$, every element of $\sigma(\bar{\mathcal{C}})$ is achievable. Since every softmax value is strictly positive, $\sigma(\bar{\mathcal{C}}) \subseteq \Delta^{p-1}_{>0}$, and combining the two inclusions gives $\mathcal{L}_{0^{+}}(\mathbf{F}) \cap \Delta^{p-1}_{>0} = \sigma(\bar{\mathcal{C}})$ when $\mathbf{1} \in \operatorname{cone}(\mathbf{F})$.

Now, we can show the if and only if statement. First, suppose $\mathbf{1} \in \operatorname{cone}(\mathbf{F})$. Trivially $\mathbf{x} = \mathbf{0}$ is in $\bar{\mathcal{C}}$, and $\sigma(\mathbf{0}) = \frac{1}{p}\mathbf{1} > 0$, which establishes that a strictly positive achievable value of $\boldsymbol{\lambda}$ exists. Now, suppose there exists $\boldsymbol{\lambda}^* \in \mathcal{L}_{0^{+}}(\mathbf{F})$ with all non-zero entries. By Lemma \ref{lem:support}, there exists some unit vector $\mathbf{w} \geq 0$ such that $S(\mathbf{w}) = \{1, \dots, p\}$. For this to hold, it must be that $\mathbf{w} \cdot \mathbf{f}_{j \cdot} = m(\mathbf{w})$ for all $j \in \{1, \dots, p\}$. Thus, 
\begin{equation}\label{eq:fw}
    \mathbf{F}\mathbf{w} = m(\mathbf{w})\mathbf{1}.
\end{equation}
Because no column of $\mathbf{F}$ is all $0$, we must have $m(\mathbf{w}) > 0$. We can then divide both sides of equation \eqref{eq:fw} by $m(\mathbf{w})$ and write $\mathbf{F}(\mathbf{w}/m(\mathbf{w})) = \mathbf{1}$, which means that $\mathbf{1} \in \operatorname{cone}(\mathbf{F})$.
\end{proof}

Finally, we are able to show Theorem \ref{thm:c0_arclength}. First, observe that when $K = 2$, $$\operatorname{arclength}\big(\mathcal{L}_{\infty}(\mathbf{F})\big) \leq \sqrt{2} \quad \text{for every } \mathbf{F} \in \mathbb{R}^{p \times 2}_{\geq 0}.$$ To see this, recall that $\mathcal{L}_{\infty}(\mathbf{F})$ is the segment joining the normalized columns $\tilde{\bff}_1, \tilde{\bff}_2 \in \Delta^{p-1}$, so its arclength is $\|\tilde{\bff}_1 - \tilde{\bff}_2\|$. For any $\mathbf{x}, \mathbf{y} \in \Delta^{p-1}$, we have $\|\mathbf{x}\|^2 = \sum_j x_j^2 \leq \sum_j x_j = 1$ (since each $x_j \in [0, 1]$) and $\mathbf{x} \cdot \mathbf{y} \geq 0$, so
$$\|\mathbf{x} - \mathbf{y}\|^2 = \|\mathbf{x}\|^2 + \|\mathbf{y}\|^2 - 2 \, \mathbf{x} \cdot \mathbf{y} \leq 2.$$
 To show that, for all $p \geq 5$, there exists an $\mathbf{F} \in \mathbb{R}^{p \times 2}_{\geq 0}$ such that $\operatorname{arclength}(\mathcal{L}_{0^+}(\mathbf{F})) > \sqrt{2}$, we first show that for $p \geq 2$, we can construct an $\mathbf{F}$ such that $\operatorname{arclength}(\mathcal{L}_{0^+}(\mathbf{F}))$ gets arbitrarily close to 
\begin{equation*}
    \psi(p) = \sum_{j=2}^{p} \frac{1}{\sqrt{j (j-1)}}.
\end{equation*}
Since $\psi(p)$ is monotonically increasing in $p$, and $\psi(5) \approx 1.6276 > \sqrt{2}$, Theorem \ref{thm:c0_arclength} immediately follows. We formalize this in the following Lemma:

\begin{lemma} \label{lem:staircase}
Fix $p \geq 2$. For every $\varepsilon > 0$, there exists an $\mathbf{F} \in \mathbb{R}^{p \times 2}_{\geq 0}$ with $\mathbf{1} \in \operatorname{cone}(\mathbf{F})$ such that
\begin{equation*}
\operatorname{arclength}\big(\mathcal{L}_{0^{+}}(\mathbf{F})\big) \;>\; \psi(p) - \varepsilon.
\end{equation*}
\end{lemma}

\begin{proof}
Take the first column of $\mathbf{F}$ to be the intercept $\mathbf{1}$, so that $\mathbf{1} \in \operatorname{cone}(\mathbf{F})$ automatically, and take the second column to be the vector $\bff$ with entries
\begin{equation*}
f_j \;=\; \sum_{k=j}^{p-1} M^{k}, \qquad j = 1, \dots, p-1, \qquad f_p = 0
\end{equation*}
where $M > 1$ is a parameter. The entries of $\bff$ are decreasing, with gaps $f_j - f_{j+1} = M^j$ that grow by a factor of $M$ each time $j$ increases.

Since the first column of $\mathbf{F}$ is constant, its centered version is $\mathbf{g}_1 = \mathbf{0}$, so $\bar{\mathcal{C}} = \{t \mathbf{g}_2 : t \geq 0\}$. By Theorem \ref{thm:c0_dichotomy} together with the shift invariance of the softmax (equation \eqref{eq:softmax_shift}),
\begin{equation*}
\{\sigma(t \bff) : t \geq 0\} \;=\; \sigma(\bar{\mathcal{C}}) \;\subseteq\; \mathcal{L}_{0^{+}}(\mathbf{F}).
\end{equation*}
Since the curve $\boldsymbol{\lambda}(t) = \sigma(t \bff)$ is injectively parameterized, it therefore suffices to lower bound its arclength.

The result of our construction is that as $t$ grows, the curve proceeds down a ``staircase'' of barycenters. Let $\mathbf{b}_j = \tfrac{1}{j} \sum_{i=1}^{j} \mathbf{e}_i$ denote the barycenter of the face of the simplex spanned by the first $j$ coordinates, so that $\mathbf{b}_p$ is the uniform vector and $\mathbf{b}_1 = \mathbf{e}_1$. The curve starts at $\boldsymbol{\lambda}(0) = \mathbf{b}_p$ and passes near $\mathbf{b}_{p-1}, \mathbf{b}_{p-2}, \dots, \mathbf{b}_1$ in order. To make this precise, for $j = 1, \dots, p-1$, let $t_j = M^{1/2 - j}$, and note that $0 < t_{p-1} < t_{p-2} < \dots < t_1$. At $t_j$, for any $i \leq j$,
\begin{equation*}
0 \;\leq\; t_j (f_i - f_j) \;=\; t_j \sum_{k=i}^{j-1} M^k \;\leq\; t_j \, p M^{j-1} \;=\; p M^{-1/2},
\end{equation*}
while for any $i > j$,
\begin{equation*}
t_j (f_i - f_j) \;\leq\; -t_j M^{j} \;=\; -M^{1/2}.
\end{equation*}
Writing
\begin{equation*}
\sigma(t_j \bff)_i \;=\; \frac{e^{t_j (f_i - f_j)}}{\sum_{l=1}^p e^{t_j (f_l - f_j)}},
\end{equation*}
the bounds above show that, as $M \rightarrow \infty$, the numerator converges to $1$ for $i \leq j$ and to $0$ for $i > j$, so $\sigma(t_j \bff) \rightarrow \mathbf{b}_j$. Since there are only finitely many $j$, for any $\varepsilon' > 0$ we can choose $M$ large enough that $\|\sigma(t_j \bff) - \mathbf{b}_j\| < \varepsilon'$ for all $j$ simultaneously.

It remains to compute the distances between consecutive barycenters. The vector $\mathbf{b}_{j-1} - \mathbf{b}_j$ has its first $j - 1$ coordinates equal to $\tfrac{1}{j-1} - \tfrac{1}{j} = \tfrac{1}{j(j-1)}$ and its $j$th coordinate equal to $-\tfrac{1}{j}$, so
\begin{equation*}
\|\mathbf{b}_{j-1} - \mathbf{b}_j\|^2 \;=\; \frac{j-1}{j^2 (j-1)^2} + \frac{1}{j^2} \;=\; \frac{1}{j (j-1)}.
\end{equation*}
Now, taking $t_p = 0$ (so that $\sigma(t_p \bff) = \mathbf{b}_p$ exactly), the curve visits the points $\boldsymbol{\lambda}(t_p), \boldsymbol{\lambda}(t_{p-1}), \dots, \boldsymbol{\lambda}(t_1)$ at increasing values of $t$, so by the triangle inequality its arclength is at least the length of the polygonal chain through these points:
\begin{equation*}
\operatorname{arclength}\big(\mathcal{L}_{0^{+}}(\mathbf{F})\big) \;\geq\; \sum_{j=2}^{p} \|\boldsymbol{\lambda}(t_{j-1}) - \boldsymbol{\lambda}(t_j)\| \;\geq\; \sum_{j=2}^{p} \Big( \|\mathbf{b}_{j-1} - \mathbf{b}_j\| - 2 \varepsilon' \Big) \;=\; \psi(p) - 2 (p-1) \varepsilon'.
\end{equation*}
Taking $\varepsilon' < \varepsilon / (2(p-1))$ completes the proof.
\end{proof}

Theorem \ref{thm:c0_arclength} now follows by applying Lemma \ref{lem:staircase} with $\varepsilon = \psi(p) - \sqrt{2}$, which is positive for all $p \geq 5$.

\section{Additional figures}
\label{app:add_figs}

\begin{figure}
    \centering
    \includegraphics[width=1.0\linewidth]{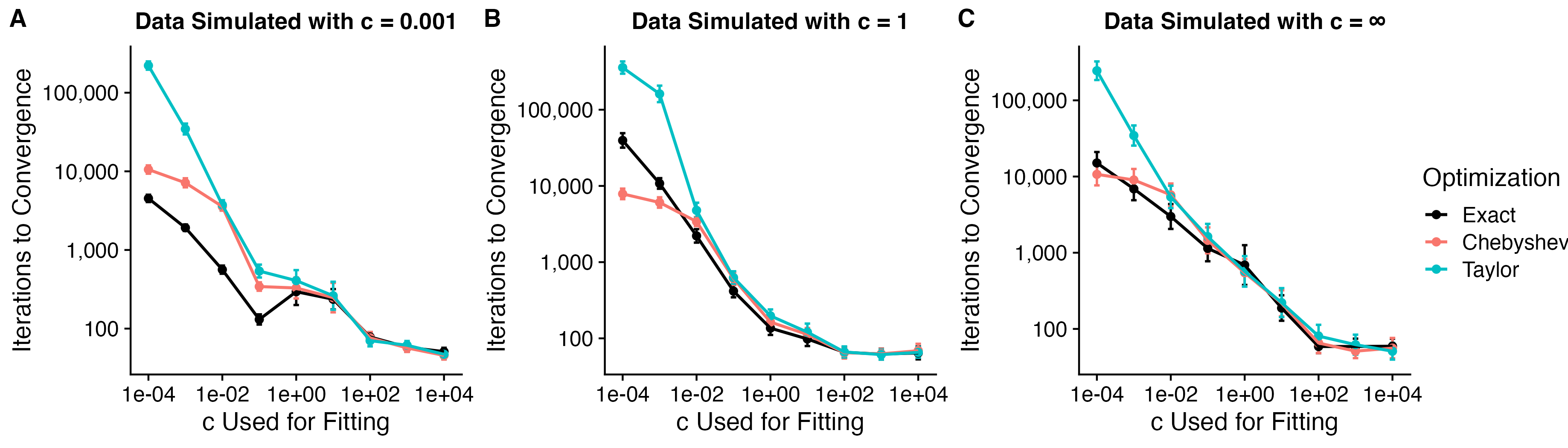}
    \caption{Number of iterations until convergence for the fits of the approximation accuracy experiment of Section \ref{sec:accuracy_approx}. For each generative regime and each finite value of $c$ used for fitting, points show the geometric mean of the number of iterations until convergence over simulation replicates ($30$ per regime; $100$ for data simulated with $c = 10^{-3}$), with error bars showing $95\%$ confidence intervals computed on the log scale. }
    \label{fig:sim2_convergence}
\end{figure}

\begin{figure}
    \centering
    \begin{minipage}[c]{0.63\linewidth}
      \centering
      \includegraphics[width=\linewidth]{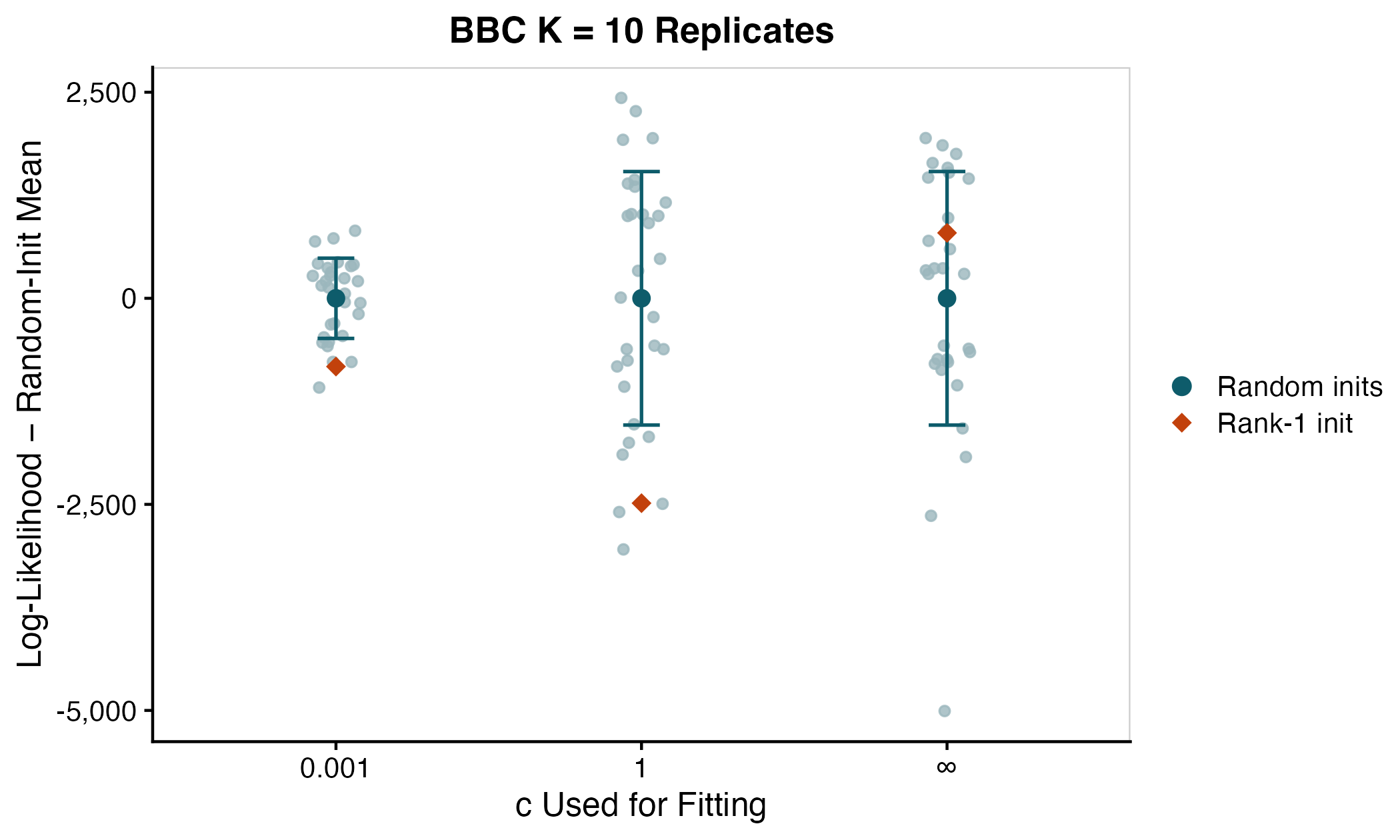}
    \end{minipage}\hfill
    \begin{minipage}[c]{0.34\linewidth}
      \centering
      \begin{tabular}{lrr}
        \toprule
        $c$ & Mean & SD \\
        \midrule
        $0.001$ & $-1{,}065{,}397$ & $487$ \\
        $1$ & $-1{,}124{,}985$ & $1{,}538$ \\
        $\infty$ & $-1{,}125{,}533$ & $1{,}538$ \\
        \bottomrule
      \end{tabular}
    \end{minipage}
    \caption{Stability of the $K = 10$ BBC fits to initialization. For each $c \in \{0.001, 1, \infty\}$, the model was fit to convergence (optimizing the exact log-likelihood) from $30$ random initializations and from the rank-1 initialization used for the results in the main text. The random initializations are matched across values of $c$: for each replicate we drew matrices $\mathbf{L}_{\textrm{raw}}$ and $\mathbf{F}_{\textrm{raw}}$ with i.i.d.\ $\textrm{Uniform}(0, 1)$ entries, depending only on the replicate so that all values of $c$ share the same draws, and initialized each model at $\mathbf{L}_0 = \sqrt{\tau_c}\, \mathbf{L}_{\textrm{raw}}$, $\mathbf{F}_0 = \sqrt{\tau_c}\, \mathbf{F}_{\textrm{raw}}$, where $\tau_c > 0$ was found numerically such that the Poisson rates implied by $\mathbf{B}_0 = \tau_c \mathbf{L}_{\textrm{raw}} \mathbf{F}_{\textrm{raw}}^{\top}$ under that model's link (i.e., $s_i g^{-1}(b_{ij}; c)$, and $\tau_c \mathbf{L}_{\textrm{raw}} \mathbf{F}_{\textrm{raw}}^{\top}$ directly for $c = \infty$) have mean equal to the mean of the observed counts. Light points show the exact log-likelihood of each converged random-initialization fit, centered at the mean over the $30$ random initializations for that value of $c$; error bars show $\pm 1$ standard deviation around this mean. (Table): mean and standard deviation of the exact log-likelihoods of the random-initialization fits for each value of $c$.}
    \label{fig:bbc_init_stability}
\end{figure}

\begin{figure}
    \centering
    \includegraphics[width=0.95\linewidth]{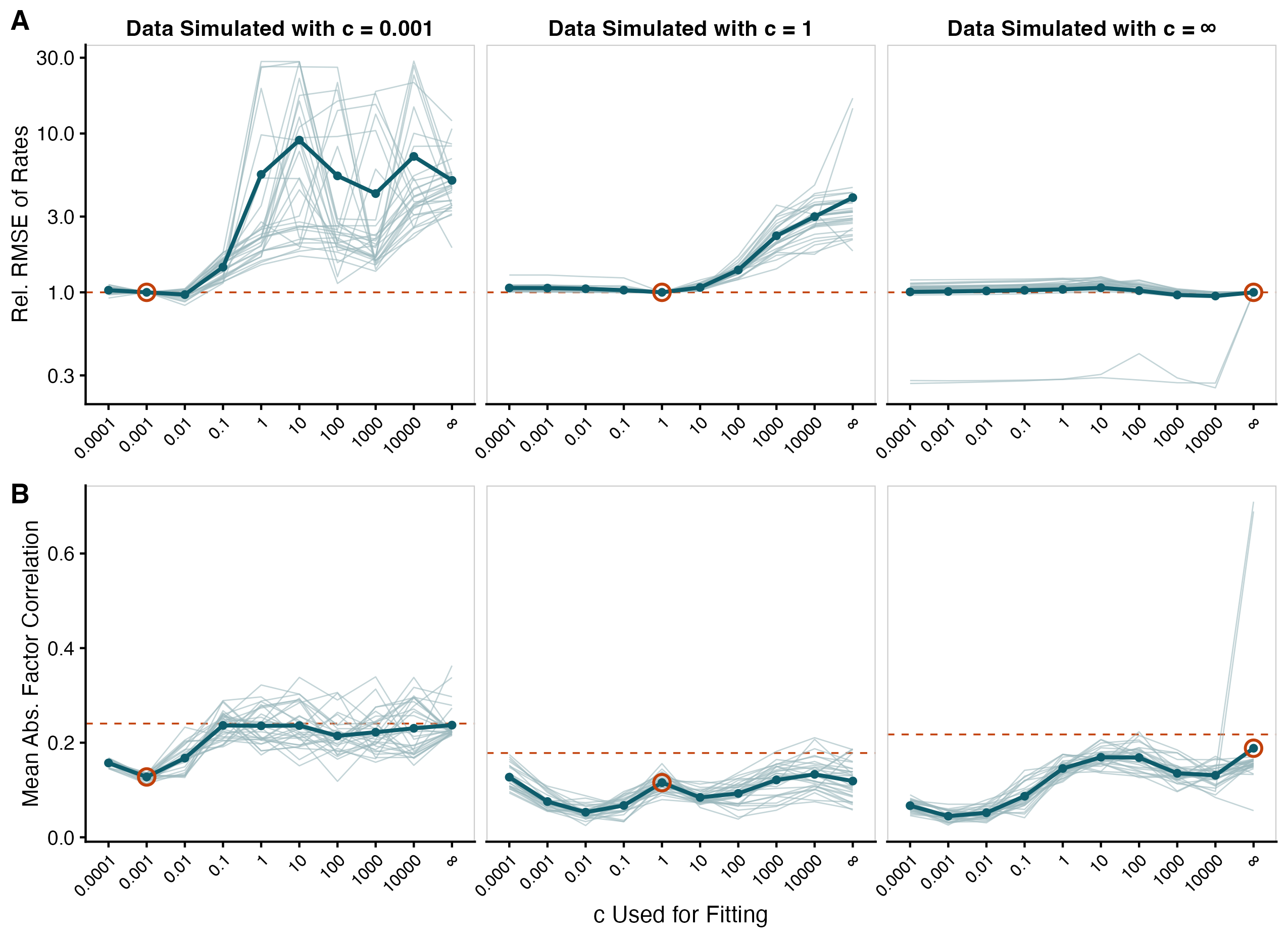}
    \caption{Additional results for Simulation 2. Thin lines show individual simulation replicates, thick lines their mean, and the open circle marks $c_{\textrm{fit}} = c_{\textrm{true}}$. (A) RMSE of the fitted rates against the true rates, relative to the same replicate's fit at $c_{\textrm{fit}} = c_{\textrm{true}}$ (log scale); the dashed line marks $1$. (B) Mean absolute pairwise Spearman correlation between the columns of the estimated $\mathbf{F}$ (as in Figure \ref{fig:sparsity}); the dashed line shows the same quantity for the true factors.}
    \label{fig:rmse_sim}
\end{figure}
\begin{figure}
    \centering
    \includegraphics[width=1.0\linewidth]{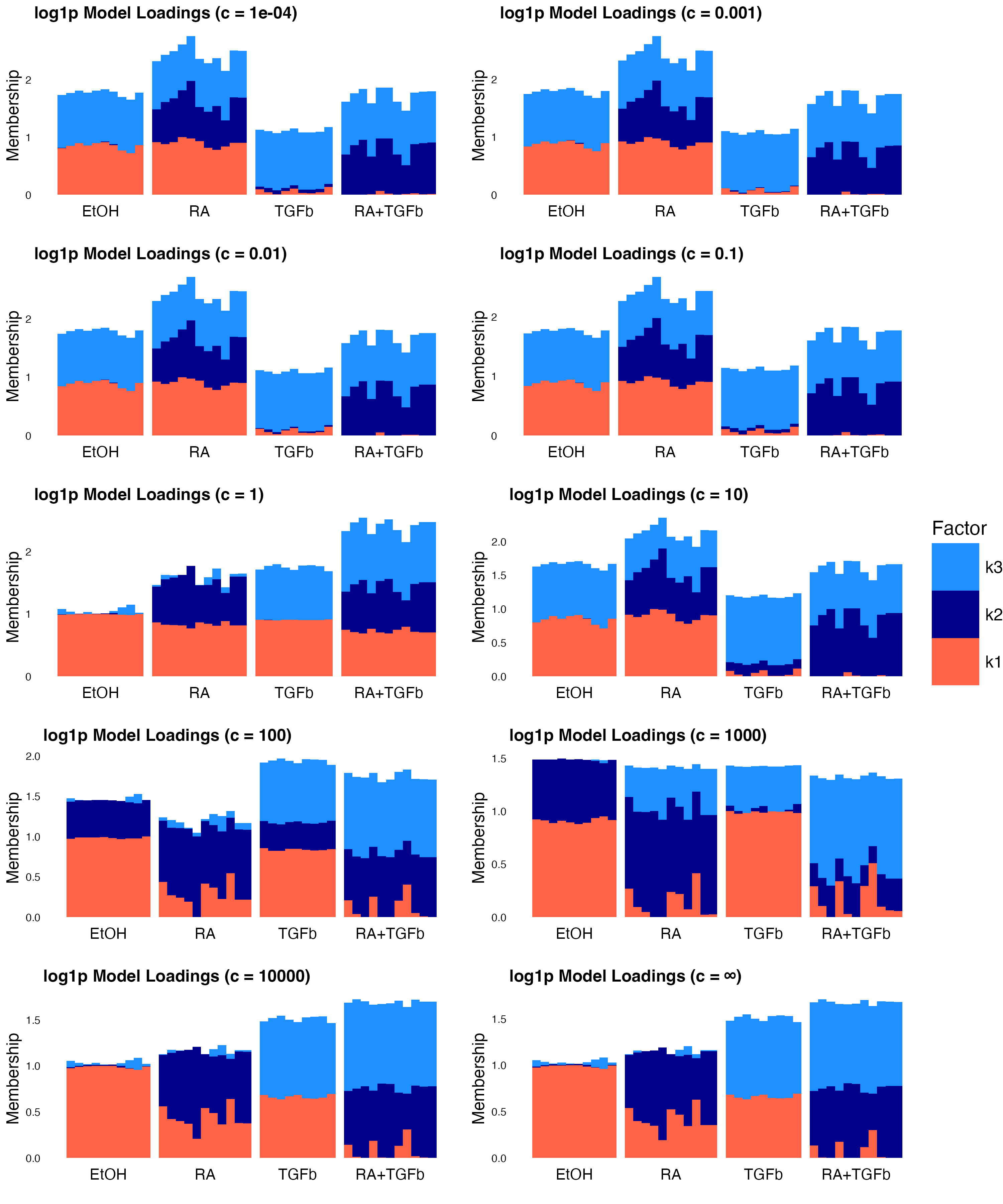}
    \caption{Structure plots for the MCF-7 dataset across various values of $c$.}
    \label{fig:supp_mcf7}
\end{figure}

\begin{figure}
    \centering
    \includegraphics[width=0.79\linewidth]{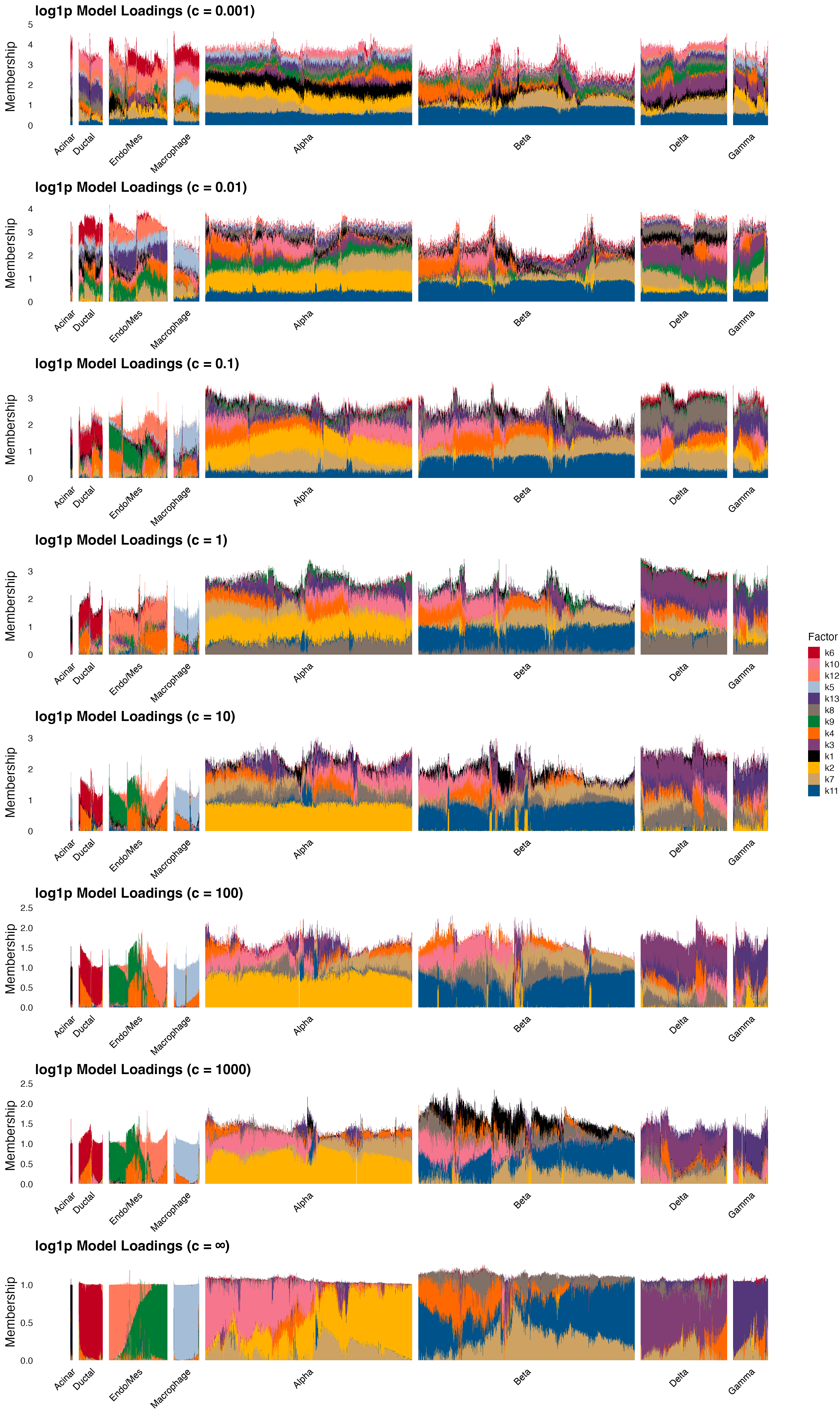}
    \caption{Structure plots for the pancreas dataset across various values of $c$.}
    \label{fig:supp_lsa}
\end{figure}

\begin{figure}
    \centering
    \includegraphics[width=0.9\linewidth]{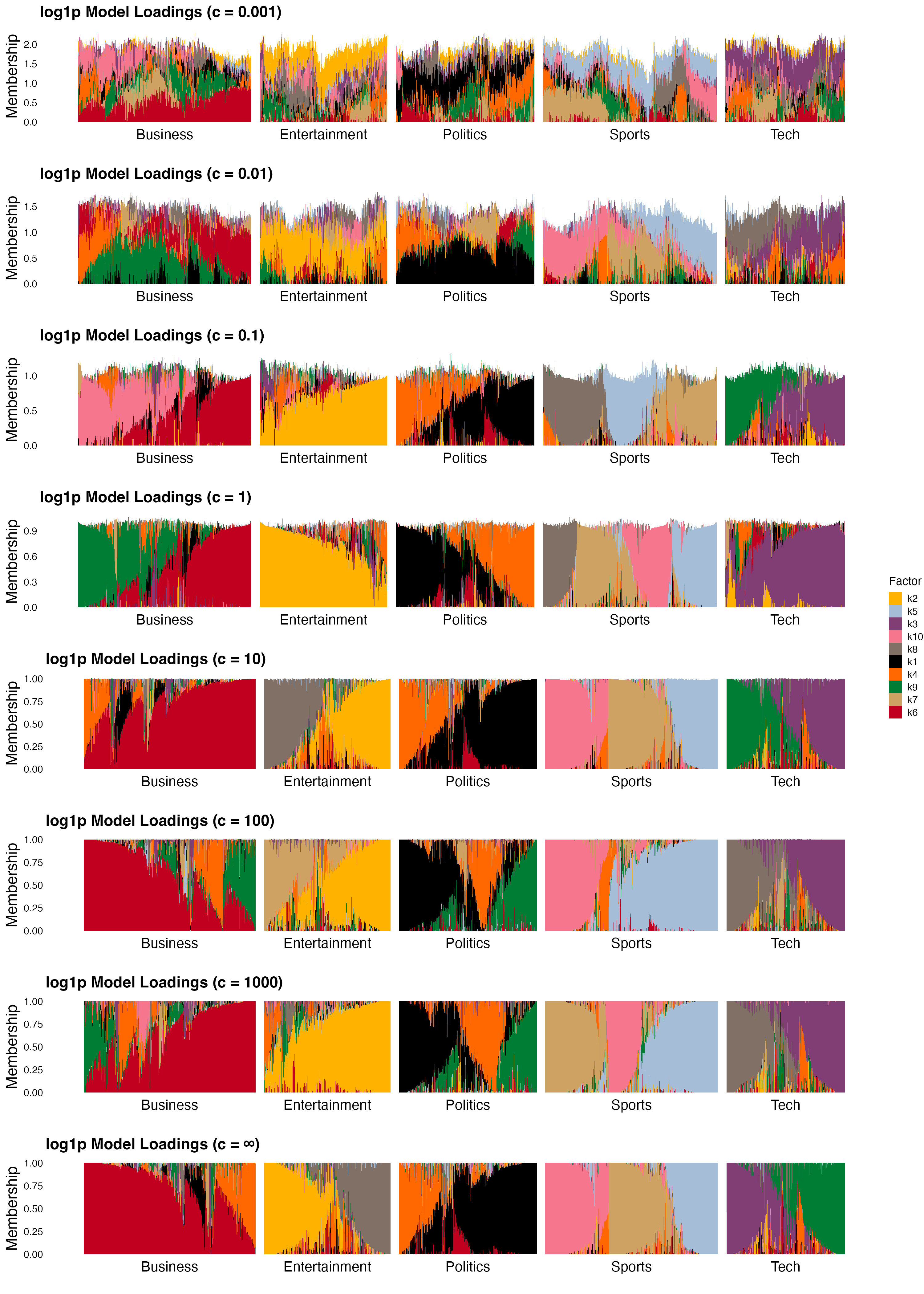}
    \caption{Structure plots for the BBC dataset across various values of $c$.}
    \label{fig:supp_bbc}
\end{figure}

\clearpage
\vskip 0.2in
\bibliography{pois_log1p_nmf}

\end{document}